\documentclass[twoside,11pt]{article}
\usepackage[abbrvbib, preprint]{jmlr2e}

\usepackage{microtype}
\usepackage{graphicx}
\usepackage{booktabs} % for professional tables
\usepackage{subcaption} 

\usepackage{hyperref}

\usepackage[normalem]{ulem}
\usepackage{soul}

\usepackage{graphicx,xcolor,caption}
\usepackage{wrapfig}

\usepackage{mathtools}
\usepackage[textsize=tiny]{todonotes}

\usepackage{dsfont}

\usepackage{shadethm}

\usepackage{parskip}

\usepackage{enumerate,color,xcolor}
\usepackage{graphicx}
\usepackage{subcaption}
\usepackage{url}
\usepackage{multirow}
\usepackage{hhline}

\usepackage{theoremref}
\usepackage{bbm}

\definecolor{darkred}{rgb}{0.7,0,0}
\definecolor{darkblue}{rgb}{0,0,0.7}
\definecolor{darkgreen}{rgb}{0,0.4,0}
\hypersetup{
	linktocpage={true},
	linkcolor = {darkblue},
	citecolor = {darkgreen} % colorlinks = false,    % linkbordercolor = {white},
}

\definecolor{crimson}{rgb}{0.7294,0.0666,0.0470}

\usepackage{dsfont}

\newtheorem{innercustomthm}{\textbf{Condition}}
\newenvironment{customcondition}[1]
{\renewcommand\theinnercustomthm{#1}\innercustomthm}
{\endinnercustomthm}

\usepackage{mleftright}
\usepackage{xparse}

\NewDocumentCommand{\evalat}{sO{\big}mm}{%
  \IfBooleanTF{#1}
   {\mleft. #3 \mright|_{#4}}
   {#3#2|_{#4}}%
}

\newcommand{\R}{\mathbb{R}}
\newcommand{\N}{\mathbb{N}}

\newcommand{\norm}[1]{\left\lVert #1 \right\rVert}
\newcommand{\abs}[1]{\left\vert #1 \right\rvert}

\newcommand{\E}[1]{\mathbb{E}{\left[ #1\right]}}

\newcommand{\brac}[1]{\left( #1 \right)}
\newcommand{\inner}[1]{\left\langle #1 \right\rangle}

\usepackage{caption}

\newcommand{\COMMENT}[1]{\textcolor{blue}{\\ \textbf{Comment:} #1 \\}}

\allowdisplaybreaks
\begin{document}
\title{Algorithmic Stability of Stochastic Gradient Descent with Momentum
under Heavy-Tailed Noise}

\author{%
\name Thanh Dang $^\dagger$  \email td22v@fsu.edu \\
 \addr Department of Mathematics\\ Florida State University, Tallahassee, FL, USA
 \AND
 \name Melih Barsbey $^\dagger$ \email m.barsbey@imperial.ac.uk \\
 \addr Department of Computing\\ Imperial College London, London, UK
 \AND
\name A K M Rokonuzzaman Sonet \email asonet@fsu.edu \\
 \addr Department of Mathematics\\ Florida State University, Tallahassee, FL, USA
 \AND
 \name Mert G\"{u}rb\"{u}zbalaban \email mg1366@rutgers.edu\\
 \addr Department of Management Science
 and Information Systems\\ Rutgers Business School, Piscataway, NJ, USA 
 \AND
  \name Umut \c{S}im\c{s}ekli $^\clubsuit$ \email umut.simsekli@inria.fr \\
 \addr Inria, CNRS, Ecole Normale Sup\'{e}rieure\\ 
PSL Research University, Paris, France
\AND
\name Lingjiong Zhu $^\clubsuit$  \email zhu@math.fsu.edu \\
 \addr Department of Mathematics\\ Florida State University, Tallahassee, FL, USA
 \\
        ~\\
       \name $^\dagger$ \addr Equal contributing first authors.\\
       \name $^\clubsuit$ \addr Corresponding authors.\\
}

\maketitle

\begin{abstract} 
Understanding the generalization properties of optimization algorithms under heavy-tailed noise has gained growing attention. However, the existing theoretical results mainly focus on stochastic gradient descent (SGD) and the analysis of heavy-tailed optimizers beyond SGD is still missing. In this work, we establish generalization bounds for SGD with momentum (SGDm) under heavy-tailed gradient noise. We first consider the continuous-time limit of SGDm, i.e., a Lévy-driven stochastic differential equation (SDE), and establish quantitative Wasserstein algorithmic stability bounds for a class of potentially non-convex loss functions. Our bounds reveal a remarkable observation: For quadratic loss functions, we show that SGDm admits a worse generalization bound in the presence of heavy-tailed noise, indicating that the interaction of momentum and heavy tails can be harmful for generalization. We then extend our analysis to discrete-time and develop a uniform-in-time discretization error bound, which, to our knowledge, is the first result of its kind for SDEs with degenerate noise. This result shows that, with appropriately chosen step-sizes, the discrete dynamics retain the generalization properties of the limiting SDE. We illustrate our theory on both synthetic quadratic problems and neural networks. 
\end{abstract}


\section{Introduction}

The goal of many supervised learning problems is to minimize the population risk that is
\begin{align}  \label{eqn:pop_risk}
    \min_{\theta \in \mathbb{R}^d} \left\{ F(\theta) := \mathbb{E}_{x \sim \mathcal{D}} [f(\theta, x)] \right\},
\end{align}
where \(x \in \mathcal{X}\) represents a random data point drawn from an unknown probability distribution \(\mathcal{D}\) over the data space \(\mathcal{X}\). The space \(\mathcal{X}\) is a subset of a normed vector space equipped with the norm \(\|\cdot\|\), and without loss of generality, we assume that \(0 \in \mathcal{X}\). Furthermore, \(\theta \in \mathbb{R}^d\) is the parameter vector to be learned, and \(f(\theta, x)\) is the loss function. 
 Suitable choices of \(f\) will correspond to a wide range of supervised learning problems which appear in deep learning, logistic regression, and support vector machines \citep{shalev2014understanding}.

Since \(\mathcal{D}\) is often unknown, practitioners instead study the empirical risk minimization problem (ERM)  which is 
\begin{align*}
\min_{\theta \in \mathbb{R}^d} \left\{ \widehat{F}(\theta,X_n) := \frac1{n} \sum\nolimits_{i=1}^n f(\theta, x_i) \right\},
\end{align*}
where \(X_n = \{x_1, \ldots, x_n\} \subset \mathcal{X}^n\) is a training dataset consisting of independent and identically distributed (i.i.d.) observations.

Stochastic gradient descent (SGD) has been the bread-and-butter algorithm to tackle the ERM problem and is based on the 
recursion:
\begin{equation}
\theta_{k+1}=\theta_k-\eta\nabla \tilde{F}_{k+1}(\theta_k,X_{n}), 
\end{equation}
where $\eta>0$ is the step-size and
\begin{equation}\label{stoch:grad}
\nabla \tilde{F}_{k}(\theta,X_{n}) := \frac1{b} \sum\nolimits_{i\in \Omega_k} \nabla f(\theta, x_i)
\end{equation} 
is the stochastic gradient, while $\Omega_k \subset \{1,\dots,n\}$ is a random subset drawn with or without replacement and $b := |\Omega_k| \ll n$ is the batch-size.

A substantial challenge in learning theory is to understand the generalization properties of stochastic optimization algorithms, including SGD. More precisely, one is interested in deriving an upper bound of the generalization error $|\hat{F}(\theta,X_n) - F(\theta)|$. 
The past few years have witnessed the birth of a variety of approaches that aim at answering the previous question for different optimization algorithms, see e.g., \citep{cao2019generalization,lei2020fine,neu2021information,camuto2021fractal,park2022generalization,hodgkinson2022generalization,zhu2024uniform,andreeva2024topological}. 

The recent years have witnessed an increasing attention in the analysis of the generalization error of SGD under \emph{heavy-tailed} gradient noise, which is typically expressed by the following recursion:
\begin{align}
\label{eqn:sgd_ht}
    \theta_{k+1}= \theta_k- \eta\nabla \widehat{F}(\theta_k,X_{n})+\xi_{k+1},
\end{align}
where $(\xi_k)_{k\geq 1}$ is a sequence of heavy-tailed random vectors, potentially with unbounded higher-order moments, i.e., $\mathbb{E}\|\xi_k\|^p = +\infty$ for some $p>1$. The interest in the generalization error analysis of optimizers with heavy-tailed noise mainly stems from two facts:
\begin{enumerate}%[label=(\roman*)]
    \item It has been both theoretically and empirically illustrated that a \emph{heavy-tailed} behavior can naturally emerge in stochastic optimization depending on the choice of hyperparameters ($\eta$ and $b$), the data distribution $\mathcal{D}$, and the geometry of the loss function $f$ \citep{gurbuzbalaban2020heavy,hodgkinson2021multiplicative,schertzer2024stochastic,jiao2024emergence,damek2024analysing}; and moreover the heaviness of the tail turns out to be positively correlated with the generalization performance in certain settings \citep{mahoney2019traditional,simsekli2020hausdorff,martin2021predicting,barsbey2021heavy}. This has motivated the use of the recursion \eqref{eqn:sgd_ht} as a `heavy-tailed proxy' for the true SGD recursion in the presence of heavy tails, which --to some extent-- facilitated the analysis of SGD in terms of its generalization error.
    \item Recently, \citet{wan2024implicit} showed that explicitly injecting heavy-tailed noise to the SGD recursion (i.e., executing \eqref{eqn:sgd_ht} directly, possibly replacing $\nabla \widehat{F}$ with $\nabla \tilde{F}_{k+1}$) for a class of neural networks results in `compressible' network weights, which might provide crucial benefits in resource-bounded applications. Moreover, \citet{lim2022chaotic} showed that heavy-tailed dynamics can emerge in deterministic gradient descent; highlighting the need for a precise understanding of the role of heavy-tails in optimization.
\end{enumerate}
In terms of understanding the links between heavy-tails and generalization, \citet{simsekli2020hausdorff} presented the first generalization bounds where the optimization algorithm was modeled by a general class of heavy-tailed stochastic differential equations (SDE). They showed that the bound is controlled by the heaviness of the tails and some incomputable information theoretic terms. \citet{squareloss2023algorithmic} analyzed the case of SGD on quadratic loss functions, where they obtained fully explicit bounds. They then extended their approach in \citep{generalloss} to a general class of (possibly non-convex) loss functions (the class that we also consider in this study). Very recently, \citet{dupuisht} refined these results and proved tighter bounds. 

While these studies have revealed various interesting phenomena that emerge in the presence of heavy tails, they only cover the case of SGD, hence, the effects of heavy tails on other popular stochastic optimization algorithms yet to be discovered.     

In this study, we aim to take a step for bridging this gap and analyze the generalization properties of stochastic gradient descent with momentum (SGDm) with heavy-tailed noise, which admits the following recursion \citep{SZTG20}:
\begin{align}
\label{eqn:sgd_addht}
    &v_{k+1}=v_k-\eta\gamma v_k-\eta\nabla \widehat{F}(\theta_k,X_{n})+\xi_{k+1},\nonumber\\
    &\theta_{k+1}=\theta_{k}+\eta v_{k+1},
\end{align}
where $\eta>0$ is the step-size (or learning-rate), $\gamma>0$ is the friction or momentum parameter, and $(\xi_k)_{k\geq 1}$ is again a sequence of heavy-tailed random vectors.

Our main goal is to provide an algorithmic stability bound for SGDm with general loss function (which can be non-convex). Our work is a follow-up of \cite{squareloss2023algorithmic,generalloss}, whose authors study algorithmic stability for heavy-tailed SGD without momentum. 
We are interested in providing generalization error bound through the lens of algorithmic stability
for heavy-tailed SGDm, and compare it with the case without momentum.
Our contributions are as follows:
\begin{itemize}
    \item We first consider the continuous-time limit of \eqref{eqn:sgd_addht}, which is an $\alpha$-stable-L\'{e}vy-driven SDE. We derive $1$-Wasserstein algorithmic stability bounds for this SDE (Theorem~\ref{theorem_wassersteinbound_mainsection}), which then leads to a generalization error bound (Corollary~\ref{coro_generalizationbound}). Our analysis relies on a Wasserstein contraction rate of the corresponding SDE that is obtained in \cite{jianwangbao2022coupling} and a framework for probability approximation of Markov processes that is established in \cite{chenxu23amarkovapproach}. 
    
\item While it seems not easy to compare the generalization error bound of SGDm and SGD for general loss functions (see Remark~\ref{remark_comparingbound}), by focusing on the case of quadratic losses, we are able to make a comparison for these two algorithms, and this result is presented in Section~\ref{section_mainsquareloss}. Our result for the quadratic loss is a $p$-Wasserstein algorithmic stability bound
for any $p\in[1,\alpha)$ (Theorem~\ref{theorem_squareloss}), which is itself a novel contribution. 
It turns out that for quadratic losses, the generalization error bound of SGDm is always larger than that of SGD (Corollary~\ref{cor:quadratic}, Proposition~\ref{lemma_comparesingularvalues}). This result reveals the fact that the interaction of momentum and heavy tails can be harmful for generalization. 

\item We provide uniform-in-time $1$-Wasserstein discretization error bound between the $\alpha$-stable-L\'{e}vy-driven SDE and its discretization, i.e., the recursion \eqref{eqn:sgd_addht} (Theorem~\ref{thm:discrete}). To the best of our knowledge, it is the first uniform-in-time $1$-Wasserstein discretization error bound for an $\alpha$-stable-L\'{e}vy-driven SDE with degenerate noise, which is of its own interest. 
As a by-product, we obtain stability (Corollary~\ref{cor:discrete:algostab}) and generalization bounds (Corollary~\ref{cor:discrete:generalizationerrorbound}) for the recursion \eqref{eqn:sgd_addht}, which illustrate that the discrete-time dynamics inherit the generalization properties of the limiting SDE for an appropriately chosen step-size.
\item We support our theory with experiments conducted on synthetic quadratic problems, and fully-connected and convolutional neural networks on MNIST and CIFAR10. 
\end{itemize}

\section{Technical Background and Notations}
\label{section_background}

%%%%%%%%%%%%%%%%%%%%%%%%%%%%%%%%%%%%%%%%%%%
\textbf{Algorithmic stability.} We define algorithmic stability as in the seminal reference \cite{hardt2016train}.
\begin{definition}[\cite{hardt2016train}, Definition 2.1] \label{def:stability}
 For a (surrogate) loss function $\ell:\mathbb{R}^d \times \mathcal{X} \rightarrow \mathbb{R}$, an algorithm $\mathcal{A} : \bigcup_{n=1}^\infty \mathcal{X}^n \to \mathbb{R}^d$ is $\varepsilon$-uniformly stable if
 \begin{align}
     \sup_{X\cong \hat{X}}\sup_{z \in \mathcal{X}}~ \mathbb{E}\left[\ell(\mathcal{A}(X),z) - \ell(\mathcal{A}(\hat{X}),z) \right] \leq \varepsilon ,
 \end{align}
where the first supremum is taken over data $X, \hat{X} \in \mathcal{X}^n$   that differ by one element, denoted by $X\cong \hat{X}$.
\end{definition}
We will employ a surrogate loss function $\ell$ to measure the algorithmic stability, which might be different from the original loss function $f$. The necessity to use a surrogate loss function in the definition of the algorithmic stability has to do with the heavy-tailed noise in our model and a detailed explanation is given in \citep[Section 3.1]{squareloss2023algorithmic} and some example realistic scenarios are presented in \citep[Appendix A]{zhu2024uniform}.

Algorithmic stability is an important concept in learning theory as it is related to the generalization performance of a randomized algorithm, which is the content of the next Theorem. In order to state the result, let us define 
\begin{align*}
    \hat{R}(\theta,X_n) :=& \frac{1}{n}\sum_{i=1}^n \ell(\theta,x_i), \quad 
    R(\theta) := \mathbb{E}_{x\sim \mathcal{D}} [\ell(\theta,x)].
\end{align*}

\begin{theorem}[\cite{hardt2016train}, Theorem 2.2]
Suppose that $\mathcal{A}$ is an $\varepsilon$-uniformly stable algorithm, then the expected generalization error is bounded by
\begin{align}
    \left|\mathbb{E}_{\mathcal{A},X_n}~\left[ \hat{R}(\mathcal{A}(X_n),X_n) \right] - R(\mathcal{A}(X_n))  \right| \leq \varepsilon.
\end{align}
\end{theorem}

%%%%%%%%%%%%%%%%%%%%%%%%%%%

\textbf{Alpha-stable distributions.}
Let $X$ be a real-valued random variable. $X$ follows a symmetric $\alpha$-stable distribution $\mathcal{S}\alpha\mathcal{S}(\sigma)$ if its characteristic function has the form $\mathbb{E}\left[e^{\mathrm{i} uX}\right]=\exp\left(-\sigma^{\alpha}|u|^{\alpha}\right)$, for any $u\in\mathbb{R}$. Here $\sigma>0$ is the scale parameter that measures the spread of $X$ around $0$, while $\alpha\in(0,2]$ is the tail-index
that determines the tail thickness of the distribution (in the sense that as $\alpha$ gets smaller, the tail becomes heavier). $\mathcal{S}\alpha\mathcal{S}$ appears naturally as the limiting distribution in the generalized central limit theorems for
a sum of i.i.d. random variables
with infinite variance \cite{applebaum2009levy}. One challenge when dealing with $\alpha$-stable distribution is that its probability density function does not have a closed-form formula except for some special cases; for example $\mathcal{S}\alpha\mathcal{S}$ reduces to the Cauchy and the Gaussian distributions, respectively, when $\alpha=1$ and $\alpha=2$.
Another important feature of a symmetric $\alpha$-stable distribution is when $0<\alpha<2$,its moments
are finite only up to the order $\alpha$: 
$\mathbb{E}[|X|^{p}]<\infty$
if and only if $p<\alpha$ (so that it has infinite variance).

Let us now extend the definition of $\alpha$-stable distribution to the multi-variate case of random vectors. There are several ways to define multi-variate $\alpha$-stable distribution \cite{ST1994}, but one of the most commonly used versions is the rotationally symmetric $\alpha$-stable distribution.
$X$ follows a $d$-dimensional rotationally symmetric $\alpha$-stable distribution
if it admits the characteristic function $\mathbb{E}\left[e^{i \langle u,X\rangle}\right]=e^{-\sigma^{\alpha}\Vert u\Vert^{\alpha}}$ for
any $u\in\mathbb{R}^{d}$, where $\Vert\cdot\Vert$ denotes the Euclidean norm.

%%%%%%%%%%%%%%%%%%%%%%%%%%%%%%%%%%%%%%%%%
\textbf{Alpha-stable L\'{e}vy processes.}
L\'{e}vy processes are stochastic processes with independent and stationary increments.
We can view their increments as the continuous-time
analogue of random walks.
Important examples of L\'{e}vy processes are the Poisson process, the Brownian motion,
the Cauchy process, and more generally stable
processes \cite{bertoin1996,ST1994,applebaum2009levy}.
L\'{e}vy processes in general can have jumps
and heavy tails. 
%%%%%%%%%%%%%%%%%%%%%%%%%
In this paper, we will consider the rotationally symmetric $\alpha$-stable L\'{e}vy process $(L_{t})_{t\geq 0}$ in $\R^d$ defined as follows.
\begin{itemize}%[noitemsep]
\item $L_0=0$ almost surely;
\item For any $t_{0}<t_{1}<\cdots<t_{N}$, the increments $L_{t_{n}}-L_{t_{n-1}}$
are independent;
\item The difference $L_{t}-L_{s}$ and $L_{t-s}$
are distributed as the symmetric $\alpha$-stable distribution $\mathcal{S}\alpha\mathcal{S}((t-s)^{1/\alpha})$, which has characteristic function $\exp(- (t-s)\|u\|^\alpha)$ for $t>s$;
\item $L_{t}$ has stochastically continuous sample paths, i.e.
for any $\delta>0$ and $s\geq 0$, $\mathbb{P}(\|L_{t}-L_{s}\|>\delta)\rightarrow 0$
as $t\rightarrow s$.
\end{itemize}
In the special case when $\alpha=2$, we have $L_{t}=\sqrt{2}\mathrm{B}_{t}$, where $\mathrm{B}_{t}$ denotes the standard Brownian motion in $\R^d$. 

\textbf{Gradients and Hessians.}
Let $f:\mathbb{R}^{d}\rightarrow\mathbb{R}$ be a twice continuously differentiable function, then $\nabla f$ and $\nabla^{2}f$ are respectively the gradient and the Hessian of $f$. 
%%%%%%%%%%%%%%%%%%%%%%%%%

\textbf{Directional derivatives.}
The first-order  
directional derivative of $f$ is defined as
$\nabla_{v}f(x):=\lim_{\epsilon\rightarrow 0}\frac{f(x+\epsilon v)-f(x)}{\epsilon}$, 
for any direction $ v\in \mathbb{R}^d$. 
%%%%%%%%%%%%%%%%%%%%%%%%%%%%%%%%%%%%%%%%%%%%%%

\textbf{Wasserstein distance.}
For $p\geq 1$, the $p$-Wasserstein distance between two probability measures $\mu$ and $\nu$ on $\mathbb{R}^{d}$
is defined as 
$\mathcal{W}_{p}(\mu,\nu)=\left\{\inf\mathbb{E}\Vert X-Y\Vert^p\right\}^{1/p}$,
where the infimum is taken over all coupling of $X\sim\mu$ and $Y\sim\nu$ \cite{villani2008optimal}.
In particular, the $1$-Wasserstein distance has the following dual representation \cite{villani2008optimal}:
\begin{equation*}
\mathcal{W}_{1}(\mu,\nu)=\sup_{h\in\text{Lip}(1)}\left|\int_{\mathbb{R}^{d}}h(x)\mu(dx)-\int_{\mathbb{R}^{d}}h(x)\nu(dx)\right|,
\end{equation*}
where $\text{Lip}(1)$ consists of the functions $h:\mathbb{R}^{d}\rightarrow\mathbb{R}$
that are $1$-Lipschitz.

%%%%%%%%%%%%%%%%%%%%%%%%%%%%%%%%%%%%%%%%%%%%%%%%%%%%%%%%%%%%%%%%%%%%%%%%%%%%%%%%%
%%%%%%%%%%%%%%%%%%%%%%%%%%%%%%%%%%%%%%%%%%%%%%%%%%%%%%%%%%%%%%%%%%%%%%%%%%%%%%%%%%

\section{Assumptions and the Generalization Bound}\label{sec:main:results}

 In this section, we will develop generalization bounds for the continuous-time limit of the recursion \eqref{eqn:sgd_addht}. Recall $\mathcal{X}$ is the space of data points and $X_n=\{x_1,\ldots,x_n \}\in\mathcal{X}^{n}$ is a dataset. Let $\widehat{X}_n$ be another dataset that differs from $X_n$ by a single data point, that is $\widehat{X}_n=\{\hat{x}_1,\ldots,\hat{x}_i,\dots,\hat{x}_n \}\in\mathcal{X}^{n}$, 
where there is at most one $i\in\{1,\ldots,n\}$ such that $\hat{x}_{i}\neq x_{i}$.

Let $L_t$ be an $\R^d$-valued rotationally invariant $\alpha$-stable process with $1<\alpha<2$ 
and $\gamma,\beta,\zeta$ be some real positive parameters. We will study the the 1-Wasserstein distance between the distribution of the following underdamped heavy-tailed SDE based on the dataset $X_{n}$:\footnote{We derive our theory for a general $\beta>0$; in practical implementations, cf.\ \eqref{eqn:sgd_addht}, $\beta$ will be set to $1$.}
\begin{align}
\label{sde_generalloss}
    d\theta_t&=v_tdt,\nonumber\\
    dv_t&=-\gamma v_tdt-\beta\nabla  \widehat{F} (\theta_t,X_n)dt+\zeta dL_t, 
\end{align}
with $(\theta_0,v_0)=(w,y)$
and the distribution of the following underdamped heavy-tailed SDE based on the dataset $\widehat{X}_{n}$:
\begin{align}
\label{sde_generalloss_differentdataset}
    d\hat{\theta}_t&=\hat{v}_tdt,\nonumber\\
    d\hat{v}_t&=-\gamma \hat{v}_tdt- \beta\nabla  \widehat{F} (\hat{\theta}_t,\widehat{X}_n)dt+\zeta dL_t, 
\end{align}
where $ \nabla \widehat{F}(\theta,X_n)=\frac{1}{n}\sum_{i=1}^n \nabla f(\theta,x_i)$, $ \nabla \widehat{F}(\theta,\widehat{X}_n)=\frac{1}{n}\sum_{i=1}^n \nabla f(\theta,\hat{x}_i)$, and $(\hat{\theta}_0,\hat{v}_0)=(w,y)$.
For simplicity, we assume the initial point $(w,y)$ is deterministic.

%%%%%%%%%%%%%%%%%%%%%%%%%%%%%%%%%%%%%%%%%%%%

By considering a surrogate loss function $\ell$, which we assume to be $L$-Lipschitz, 
our bound on the Wasserstein distance between $\mathrm{Law}\brac{\theta_t,v_t}$ and $\mathrm{Law}\brac{\hat{\theta}_t,\hat{v}_t}$ (Theorem~\ref{theorem_wassersteinbound_mainsection}) immediately provides us a generalization error bound thanks to the dual representation of the Wasserstein distance (cf.\ \citep[Lemma 3]{raginsky2016information}):
\begin{align}
    \label{eqn:wass_stab}  \bigg|\mathbb{E}_{\theta_t,X_n}~\left[ \hat{R}\left(\theta_t,X_n\right) \right] 
    -  R\left(\theta_t\right)  \bigg| 
    \leq L\sup_{X_n\cong \hat{X}_n}   \mathcal{W}_{1}\left(\mathrm{Law}\brac{\theta_t,v_t},\mathrm{Law}\brac{\hat{\theta}_t,\hat{v}_t}\right),
\end{align}
where $\mathrm{Law}\brac{\theta_t,v_t}$ and $\mathrm{Law}\brac{\hat{\theta}_t,\hat{v}_t}$ respectively depend on the datasets $X_n$ and $\hat{X}_n$ via the SDEs
\eqref{sde_generalloss} and \eqref{sde_generalloss_differentdataset}. 
The reason why we require a surrogate loss function is because we need the Lipschitz continuity of the loss to be able to derive the bound in \eqref{eqn:wass_stab}. However, as observed in \cite{generalloss,squareloss2023algorithmic}, our assumptions on the true loss $f$ will be incompatible with the Lipschitz continuity of $f$.\looseness=-1

%%%%%%%%%%%%%%%%%%%%%%%%%%%%%%%%%%%%%%%%%%%%%%%%%%%%%%%%%%%%%%%%%%%%%%%%%%%%%%%%%%
\paragraph{Assumptions.}

We first assume that the loss function is continuously differentiable so that the gradient of the loss function is well-defined.

\begin{customcondition}{H1}\label{cond_gammaandbeta}
$ f (\cdot,x)\in C^1\brac{\R^d}$ for any $x\in\mathcal{X}$. 
\end{customcondition}

The following conditions are taken from \cite{jianwangbao2022coupling}. They will allow us to invoke Corollary 1.4 in \cite{jianwangbao2022coupling} about ergodicity of \eqref{sde_generalloss} and exponential Wasserstein decay of the associated semigroups. 

\begin{customcondition}{H2} 
\label{cond_pseudolipschitz}
There exist universal constants $K_1,K_2$ such that for any $\theta,\hat{\theta}\in\mathbb{R}^{d}$ and $x,\hat{x}\in\mathcal{X}$,
\begin{align*}
\norm{ \nabla f \brac{\theta,x}-\nabla  f (\hat{\theta},\hat{x})}
\leq K_1\norm{\theta-\hat{\theta}}+K_2\norm{x-\hat{x}}\brac{\norm{\theta}+\norm{\hat{\theta}}+1}. 
\end{align*}
\end{customcondition}

Note that Condition~\ref{cond_pseudolipschitz} implies that for any two datasets $X_n$ and $\hat{X}_n$ and any $\theta,\hat{\theta}\in\mathbb{R}^{d}$,
\begin{align*}
\norm{ \nabla \widehat{F}\brac{\theta,X_n}-\nabla  \widehat{F}(\hat{\theta},\hat{X}_n)}
\leq K_1\norm{\theta-\hat{\theta}}
+K_2\rho(X_n,\hat{X}_n)\brac{\norm{\theta}+\norm{\hat{\theta}}+1},
\end{align*}
where 
\begin{align}\label{defn:rho}
    \rho(X_n,\hat{X}_n):= \frac{1}{n}\sum_{i=1}^n \norm{x_i-\hat{x}_i}. 
\end{align}
Condition~\ref{cond_pseudolipschitz} is a pseudo-Lipschitz-like condition on $\nabla f$ (see also \citep{generalloss,zhu2024uniform,pavasovic2023approximate,csimcsekli2024differential}) and is satisfied for various problems such as generalized linear models \citep{bach2014adaptivity}.

\begin{customcondition}{H3} 
\label{cond_allthelambdas}
There exist universal constants $\lambda_1>0$ and $\lambda_2,\lambda_3,\lambda_4,\lambda_5\geq 0$ such that 
\begin{align}
\label{cond_frombaowang}
    \lambda_2\lambda_4<\lambda_1, %\nonumber,\\
    \qquad 2\beta\lambda_4<\frac{\gamma^2}{4}+\sqrt{\beta(\lambda_1-\lambda_2\lambda_4)}\gamma,
\end{align}
such that for every $x\in \mathcal{X}$ and $\theta\in\mathbb{R}^{d}$, 
\begin{align}
\label{cond_lambda123}
    \inner{\theta,\nabla  f \brac{\theta,x}}\geq& \lambda_1\norm{\theta}^2+\lambda_2  f (\theta,x)-\lambda_3, \quad \text{and} \\ %\quad %\text{for every $\theta\in\R^d$},
\label{cond_lambda45}
     f \brac{\theta,x}\geq& -\lambda_4\norm{\theta}^2-\lambda_5.
\end{align}
\end{customcondition}
Condition~\ref{cond_allthelambdas} implies that the dissipativity condition
$\langle\theta,\nabla f(\theta,x)\rangle\geq(\lambda_{1}-\lambda_{2}\lambda_{4})\Vert\theta\Vert^{2}-\lambda_{2}\lambda_{5}-\lambda_{3}$ holds. On the other hand, Condition~\ref{cond_pseudolipschitz} implies $f$ has at most quadratic growth \citep{raginsky2017non}, and together with a dissipativity condition, it implies Condition~\ref{cond_allthelambdas}.
Therefore, Condition~\ref{cond_allthelambdas} is essentially a dissipativity condition that is satisfied 
for various non-convex optimization problems, such as one-hidden-layer neural networks \citep{akiyama2022excess}, non-convex formulations of classification problems (e.g. in logistic regression with a sigmoid/non-convex link function), robust regression problems (e.g. \cite{gao2022global}), regularized regression problems where the loss is a strongly convex quadratic plus a smooth penalty that grows slower than a quadratic; see \cite{erdogdu2022} for many other examples. Dissipativity conditions also arise in the sampling and Bayesian learning and global convergence in non-convex optimization literature \citep{raginsky2017non,gao2022global}.

%%%%%%%%%%%%%%%%%%%%%%%%%%%%%%%%%%%%%%%%%%%%%%%%%%%%%%%%%%%%%%%%%%%%%%%%%%%%%
%%%%%%%%%%%%%%%%%%%%%%%%%%%%%%%%%%%%%%%%%%%%%%%%%%%%%%%%%%%%%%%%%%%%%%%%%%%%%

\paragraph{Generalization bound.} Under our assumptions, we are now ready to present our generalization bound. In the main body of the paper, for notational simplicity, we will present all the results for the stationary distributions of the parameters (i.e., the law when $t$ or $k$ goes to infinity). However, all of our bounds hold for any time $t$ or iteration $k$, possibly with different constants, as shown in the Appendix.

\begin{theorem}
    \label{theorem_wassersteinbound_mainsection}
Assume Conditions~\ref{cond_gammaandbeta},~\ref{cond_pseudolipschitz}, and \ref{cond_allthelambdas}. Let $\mu,\hat{\mu}$ be the invariant measures of the process $\left\{\brac{\theta_t,v_t}:t\geq 0\right\}$ and the process $\left\{\brac{\hat{\theta}_t,\hat{v}_t}:t\geq 0\right\}$ respectively. Then it holds that 
\begin{align}
    \mathcal{W}_1\brac{\mu,\hat{\mu}}
    \leq \rho(X_n,\widehat{X}_n)\cdot\widetilde{C},
\end{align}
where $\rho(X_n,\widehat{X}_n)$ is defined in \eqref{defn:rho} and explicit form of the constant $\widetilde{C}$ is provided in the proof in Appendix~\ref{proof:main:results}. 
\end{theorem}

Due to space constraints, the proofs of Theorem~\ref{theorem_wassersteinbound_mainsection} and 
all the subsequent results will be provided in the Appendix.

Notice the upper bound of $\mathcal{W}_1\brac{\mu,\hat{\mu}}$ is $\rho(X_n,\widehat{X}_n)$ up to an explicitly computable constant; if this term is small (which is the case when the two datasets $X_n$ and $\widehat{X}_n$ are close to each other), then our upper bound will also be small. 

Now by combining Theorem~\ref{theorem_wassersteinbound_mainsection} and \eqref{eqn:wass_stab}, we are able to provide a generalization error bound under a Lipschitz surrogate loss function. 

\begin{corollary}
\label{coro_generalizationbound}
    Assume Conditions~\ref{cond_gammaandbeta},~\ref{cond_pseudolipschitz}, and~\ref{cond_allthelambdas}. Assume that $\ell$ is $L$-Lipschitz and $\sup_{x,y\in \mathcal{X}}\norm{x-y}\leq D$ for some $D<\infty$. Then it holds that,
\begin{align*}
 &\left|\mathbb{E}_{\theta_\infty,X_n}~\left[ \hat{R}(\theta_\infty,X_n) \right] -  R(\theta_\infty)  \right| \\
    &\leq \frac{1}{n}  \Big(d_1 D+d_2 D^{5/4}+d_3 D^{3/2}
    +d_4 D^{7/4}+d_5D^2+d_6 D^{5/2}\Big),
\end{align*}
where the real coefficients $d_i,1\leq i\leq 6$ are independent of $D$ and are given in \eqref{def_d1throughd6} in Appendix~\ref{proof:main:results}, and $(\theta_{\infty},v_{\infty})$ follows the stationary distribution of $(\theta_{t},v_{t})$. 
\end{corollary}

\begin{remark}
\label{remark_comparingbound}
One would expect that we can directly compare the above generalization error bound for heavy-tailed SGD with momentum to the generalization error bound for heavy-tailed SGD without momentum in \citep[Corollary 6]{generalloss}, however this is a hard task when considering general loss functions. One reason is that we rely on theoretical results in \cite{jianwangbao2022coupling} to obtain our generalization error bound and some of the constants in the aforementioned reference (namely $c_0$ and $C_0$ in their paper) are not explicit. 
\end{remark}

%%%%%%%%%%%%%%%%%%%%%%%%%%%%%%%%%%%%%%%%%%%%%%%%%%%%%%%%%%%%%%%%%%%%%%%%%%%%%
%%%%%%%%%%%%%%%%%%%%%%%%%%%%%%%%%%%%%%%%%%%%%%%%%%%%%%%%%%%%%%%%%%%%%%%%%%%%%

\section{Comparison with SGD}\label{section_mainsquareloss}

In this section, by considering only quadratic loss functions, we are able to derive estimates with explicit constants on the generalization error bound of SGD and SGDm, thus allowing us to make the comparison between the two algorithms. 

To be able to make a fair comparison, we use the identical setting introduced in \citep{squareloss2023algorithmic}.  Let $f(\theta,x)=\brac{\theta^{\top}x}^2$ and denote $Y_t=(\theta_t,y_t), \hat{Y}_t=(\hat{\theta}_t,\hat{y}_t)$. Recall that $X,\hat{X}\in \R^{n\times d}$ where
$X=X_n= (x_1,\cdots,x_n)^{\top}$ and $\hat{X}=\widehat{X}_{n}=(\hat{x}_1,\ldots,\hat{x}_i,\dots,\hat{x}_n )^{\top}$
 are two datasets differing by exactly one data point. Then the continuous-time proxies of heavy-tailed SGDm \eqref{sde_generalloss}-\eqref{sde_generalloss_differentdataset} become:
\begin{align}
\label{sde_squaredloss}
dY_t=-AY_tdt+\Sigma dL_t;\quad d\hat{Y}_t=-\hat{A}\hat{Y}_tdt+\Sigma dL_t,
\end{align}
where 
\begin{equation}
A=\begin{bmatrix}
   0 & -I\\
     \frac{1}{n} X^{\top}X &\gamma
\end{bmatrix},
\qquad\hat{A}=\begin{bmatrix}
   0 & -I\\
     \frac{1}{n} \hat{X}^{\top}\hat{X} &\gamma
\end{bmatrix},
\end{equation}
and $\Sigma=\begin{bmatrix}
    0 & 0\\
    0  & \zeta I
\end{bmatrix}$
is a $2d\times 2d$ matrix. 

On the other hand, the SDEs considered in \citep{squareloss2023algorithmic} for SGD without momentum are as follows:
\begin{align}
\label{sde_squaredloss_overdamped}
&dZ_t=-\brac{\frac{1}{n}X^{\top}X}Z_tdt+ \zeta dL_t;\nonumber
\\
&d\hat{Z}_t=-\brac{\frac{1}{n} \hat{X}^{\top}\hat{X}}\hat{Z}_tdt+\zeta dL_t.
\end{align}

To facilitate the presentation, we will denote $\theta_{\min}$ the smaller of the smallest singular values of $\frac{1}{n} X^{\top}X$ and $\frac{1}{n} \hat{X}^{\top}\hat{X}$. Similarly, $\sigma_{\min}$ is the smaller of the smallest singular values of $A$ and $\hat{A}$. 
By definition, $x_ix_i^{\top}-\tilde{x}_i\tilde{x}_i^{\top}$ is a $d\times d$ matrix of at most rank $2$ and can be written as
\begin{align}
\label{def_thetasigmaandothers}
    x_ix_i^{\top}-\tilde{x}_i\tilde{x}_i^{\top}=\sigma_1v_1v_1^{\top}+\sigma_2 v_2v_2^{\top},
\end{align}
where $\sigma_1,\sigma_2$ are non-zero constants and $v_1,v_2$ are orthonormal vectors in $\R^d$.

\begin{theorem}
\label{theorem_squareloss}
Assume that $X^{\top}X,\hat{X}^{\top}\hat{X}$ are positive definite. Then the processes $Y_t,\hat{Y}_t, Z_t$ and $\hat{Z}_t$ have unique stationary distributions.
In particular, let $\mu,\hat{\mu},\nu$ and $\hat{\nu}$ be respectively the stationary distributions of $Y_t, \hat{Y}_t, Z_t$ and $\hat{Z}_t$, then we have the following uniform-in-time estimate in $p$-Wasserstein distance for any $p\in [1,\alpha)$:
\begin{align}
    &\mathcal{W}_p\brac{\mu,\hat{\mu}}
     \leq \frac{\zeta\abs{\sigma_1+\sigma_2}\norm{Y_0}}{n}\cdot\Bigg(\frac{4V_d^{1/2} }{\sigma^{3/2}_{\min}(2-\alpha)^{1/2}}\nonumber\\
     &\quad+ C(p) \brac{\frac{V_d}{\alpha-p} }^{1/p}\bigg(\frac{1}{\sigma_{\min}}(1-e^{-\sigma_{\min}})+e^{-\sigma_{\min}}\brac{\frac{1}{\sigma_{\min}}+\frac{2}{\sigma_{\min}^2}+\frac{2}{\sigma_{\min}^3} }\bigg)^{1/p}\Bigg) \label{wassersteinbound_squareloss_with_momentum},\\
     &\mathcal{W}_p\brac{\nu,\hat{\nu}}
     \leq \frac{\zeta\abs{\sigma_1+\sigma_2}\norm{Y_0}}{n}\cdot\Bigg(\frac{4V_d^{1/2} }{\theta^{3/2}_{\min}(2-\alpha)^{1/2}}\nonumber\\
     &\quad+ C(p) \brac{\frac{V_d}{\alpha-p} }^{1/p}\bigg(\frac{1}{\theta_{\min}}(1-e^{-\theta_{\min}})
     +e^{-\theta_{\min}}\brac{\frac{1}{\theta_{\min}}+\frac{2}{\theta_{\min}^2}+\frac{2}{\theta_{\min}^3} }\bigg)^{1/p}\Bigg) \label{wassersteinbound_squareloss_without_momentum}, 
\end{align}
where $C(p)$ is a constant that depends only on $p$, and $V_d=\frac{\pi^{d/2}}{\Gamma(\frac{d}{2}+1)}$ is the volume of a $d$-dimensional unit ball. 
\end{theorem}
The above result
in combination with \eqref{eqn:wass_stab} yields the following generalization error bound for a Lipschitz continuous loss function. 

\begin{corollary}\label{cor:quadratic}
              Assume that $\ell$ is $L$-Lipschitz, then we have
     \begin{align}
      \label{estimate_squareloss_withmomentum}
        & \left|\mathbb{E}_{\theta_\infty,X_n}~\left[ \hat{R}(\theta_\infty,X_n) \right] -  R(\theta_\infty)  \right|\nonumber\\
         &\leq L\frac{\zeta\abs{\sigma_1+\sigma_2}\norm{Y_0}}{n}\cdot\Bigg(\frac{4V_d^{1/2} }{\sigma^{3/2}_{\min}(2-\alpha)^{1/2}}+ C \frac{V_d}{\alpha-1} \bigg(\frac{1}{\sigma_{\min}}(1-e^{-\sigma_{\min}})\nonumber\\
     &\qquad\qquad\qquad\qquad\qquad\qquad+e^{-\sigma_{\min}}\brac{\frac{1}{\sigma_{\min}}+\frac{2}{\sigma_{\min}^2}+\frac{2}{\sigma_{\min}^3} }\bigg)\Bigg) ,
     \end{align}
and
 \begin{align}
 \label{estimate_squareloss_withoutmomentum}
         &\left|\mathbb{E}_{Z_\infty,X_n}~\left[ \hat{R}(Z_\infty,X_n) \right] -  R(Z_\infty)  \right|\nonumber\\
         &\leq L\frac{\zeta\abs{\sigma_1+\sigma_2}\norm{Y_0}}{n}\cdot\Bigg(\frac{4V_d^{1/2} }{\theta^{3/2}_{\min}(2-\alpha)^{1/2}}+ C \frac{V_d}{\alpha-1} \bigg(\frac{1}{\theta_{\min}}(1-e^{-\theta_{\min}})\nonumber\\
     &\qquad\qquad\qquad\qquad\qquad\qquad\quad+e^{-\theta_{\min}}\brac{\frac{1}{\theta_{\min}}+\frac{2}{\theta_{\min}^2}+\frac{2}{\theta_{\min}^3} }\bigg)\Bigg) ,
\end{align}
where $C$ is a constant independent of the dimension $d$ and other parameters, $V_d=\frac{\pi^{d/2}}{\Gamma(\frac{d}{2}+1)}$ is the volume of a $d$-dimensional unit ball, the constants $\sigma_1,\sigma_2$ and $\sigma_{\min},\theta_{\min}$ are defined in \eqref{def_thetasigmaandothers} and the random vectors $Y_\infty = [\theta_\infty, y_\infty]$ and $Z_\infty$ follow the stationary distributions of the processes $Y_t$ and $Z_t$ respectively. 
\end{corollary}

Regarding the estimates in Corollary~\ref{cor:quadratic}, notice that 
\begin{align*}
    x\mapsto(1-e^{-x})/x,\qquad x\mapsto e^{-x}\brac{1/x+2/x^2+2/x^3} 
\end{align*}
are monotone decreasing functions on $(0,\infty)$. It follows that, in order to compare the generalization error bound for heavy-tailed SGD and heavy-tailed SGD with momentum, the key quantities
to compare are $\sigma_{\min}$ and $\theta_{\min}$. 
In the next result, we will show 
that we always have $\sigma_{\min}\leq \theta_{\min}$.

\begin{proposition}
\label{lemma_comparesingularvalues}
    It holds that $\sigma_{\min}\leq \theta_{\min}$. 
 \end{proposition}

Since $\sigma_{\min}\leq \theta_{\min}$ per Proposition~\ref{lemma_comparesingularvalues}, the generalization error bound for SGDm at \eqref{estimate_squareloss_withmomentum} is larger than the generalization error bound for SGD at \eqref{estimate_squareloss_withoutmomentum} according to Corollary~\ref{cor:quadratic}.

\begin{remark}
There are some key differences between our estimate at \eqref{estimate_squareloss_withoutmomentum} and \citep[Theorem~4]{squareloss2023algorithmic}, the latter of which is also about algorithmic stability of heavy-tailed SGD without momentum for least square regression. First, whereas \cite{squareloss2023algorithmic} uses a power function (of some power between 1 and 2) as the loss function in their definition of algorithmic stability (see their Section 3.1), our estimate \eqref{estimate_squareloss_withoutmomentum} is derived under the assumption that the loss function is Lipschitz-continuous. Second, in term of methodology, \cite{squareloss2023algorithmic} takes advantages of Fourier transform to estimate the stability, while we use a simple coupling argument. 
\end{remark}

\begin{remark}
Here we discuss how the choice of friction (momentum) parameter $\gamma>0$ affects the generalization bound for SGDm in \eqref{estimate_squareloss_withmomentum}. Per  Appendix~\ref{section_proofcomparesingularvalues}, we have 
$\sigma_{\min}= \min_{1\leq i\leq d} \{g_i(\gamma) \}$,
where 
\begin{align*}
g_i(\gamma):=\frac{\gamma^2+\kappa_i^2+1- \sqrt{\brac{\gamma^2+\kappa_i^2+1}^2-4\kappa_i^2} }{2}. 
\end{align*}
In addition, since the map $x\mapsto x-\sqrt{x^2-a^{2}}$ is strictly decreasing for $x\geq a>0$, $g_{i}(\gamma)$ is strictly decreasing in $\gamma>0$. These facts and the estimate \eqref{estimate_squareloss_withmomentum} suggest that a choice of smaller $\gamma$ will lead to a smaller generalization bound for SGDm. Note however that no matter how we choose $\gamma>0$, generalization error of SGDm cannot be tighter than that of SGD, as Proposition~\ref{lemma_comparesingularvalues} has shown. 
\end{remark}

%%%%%%%%%%%%%%%%%%%%%%%%%%%%%%%%%%%%%%%%%%%%%%%%%%%

\section{Discrete-Time Analysis}\label{sec:discrete}

In Section~\ref{sec:main:results} and Section~\ref{section_mainsquareloss}, 
our analysis was based on the continuous-time dynamics
\eqref{sde_generalloss}--\eqref{sde_generalloss_differentdataset}.
Next, we introduce and study the following
discretization of \eqref{sde_generalloss}--\eqref{sde_generalloss_differentdataset}:
\begin{align}
    \label{discrete_equation:0}
    V_{k+1}&=V_k-\eta\gamma V_k-\eta\nabla\widehat{F} (\Theta_k,X_n)+\zeta\xi_{k+1},\nonumber\\
    \Theta_{k+1}&=\Theta_k+\eta V_{k+1}, 
\end{align}
and
\begin{align}
    \label{discrete_equation:1}
    \hat{V}_{k+1}&=\hat{V}_k-\eta\gamma\hat{V}_k-\eta\nabla\widehat{F} (\hat{\Theta}_k,\widehat{X}_n)+\zeta\xi_{k+1},\nonumber\\
    \hat{\Theta}_{k+1}&=\hat{\Theta}_k+\eta\hat{V}_{k+1}, 
\end{align}
with $\xi_{k+1}:=L_{k+1}-L_k$ and $(\Theta_{0},V_{0})=(\hat{\Theta}_{0},\hat{V}_{0})=(w,y)$. 
We will obtain a uniform-in-time 1-Wasserstein error
bound on the discretization error between \eqref{sde_generalloss}-\eqref{sde_generalloss_differentdataset}
and \eqref{discrete_equation:0}-\eqref{discrete_equation:1}. 
To the best of our knowledge, the uniform-in-time discretization error bound in 1-Wasserstein distance for L\'{e}vy-driven SDE
has only been studied in \cite{chenxu2023euler} for rotationally invariant L\'{e}vy noise and in \cite{dangzhu2024euler} for L\'{e}vy noise with i.i.d.\ components that allows $\widehat{F}$ to be non-convex.
Our discretization scheme \eqref{discrete_equation:0}-\eqref{discrete_equation:1} is fundamentally different than 
the ones considered in \cite{chenxu2023euler,dangzhu2024euler}.
First, it is based on \eqref{sde_generalloss}-\eqref{sde_generalloss_differentdataset} with \emph{degenerate noise}.
Second, it is a modification of the Euler-Maruyama scheme. 
Therefore, by obtaining the time-uniform 
1-Wasserstein discretization error
guarantee, we make a contribution to the theory of L\'{e}vy-driven SDE with degenerate noise which is of its own interest.

Before we proceed, we first obtain the following ergodicity result for \eqref{discrete_equation:0}-\eqref{discrete_equation:1}. 

\begin{theorem}\label{thm:ergodicity:discretedynamics}
Assume Conditions~\ref{cond_gammaandbeta},~\ref{cond_pseudolipschitz}, and~\ref{cond_allthelambdas} hold, and also that $\sup_{x,y\in \mathcal{X}}\norm{x-y}\leq D$ for some $D<\infty$.
The Markov chain $\{(\Theta_n,V_n):n\in\N\}$ in \eqref{discrete_equation:0} admits a unique invariant measure $\mu_\eta$ 
and the Markov chain $\{(\hat{\Theta}_n,\hat{V}_n):n\in\N\}$ in \eqref{discrete_equation:1} admits a unique invariant measure $\hat{\mu}_\eta$
provided that 
$\eta<\bar{\eta}$,
where $\bar{\eta}$ is an explicit constant given in \eqref{bar:eta:defn} in Appendix~\ref{appendix:discrete}.
\end{theorem}

Next, let us recall that $\mu$ is the unique invariant measure
for the process $\left\{\brac{\theta_t,v_t}:t\geq 0\right\}$ 
in \eqref{sde_generalloss} and 
$\hat{\mu}$ is the unique invariant measure for 
the process $\left\{\brac{\hat{\theta}_t,\hat{v}_t}:t\geq 0\right\}$
in \eqref{sde_generalloss_differentdataset}. 
Then, we have the following 
uniform-in-time 
1-Wasserstein discretization error
guarantee.

\begin{theorem}\label{thm:discrete}
Under the assumptions in Theorem~\ref{thm:ergodicity:discretedynamics}, 
\begin{align}
&\mathcal{W}_{1}(\mu,\mu_{\eta})\leq C\eta^{1/\alpha},
\\
&\mathcal{W}_{1}(\hat{\mu},\hat{\mu}_{\eta})\leq\widehat{C}\eta^{1/\alpha},
\end{align}
where $C,\widehat{C}>0$ are some constants (independent of $\eta$)
that are provided in the proof in Appendix~\ref{appendix:discrete}.
\end{theorem}
By the triangle inequality for 1-Wasserstein distance and applying
Theorem~\ref{thm:discrete}, we obtain the 1-Wasserstein algorithmic stability for the discrete-time dynamics \eqref{discrete_equation:0}-\eqref{discrete_equation:1}:
\begin{equation}
\label{equation_algostabdiscrete_pre}
\mathcal{W}_{1}(\mu_{\eta},\hat{\mu}_{\eta})\leq\mathcal{W}_{1}(\mu,\hat{\mu})+C\eta^{1/\alpha}+\widehat{C}\eta^{1/\alpha},
\end{equation}
where $\mathcal{W}_{1}(\mu,\hat{\mu})$
is the 1-Wasserstein algorithmic stability for the continuous-time 
dynamics and by applying Theorem~\ref{theorem_wassersteinbound_mainsection} to \eqref{equation_algostabdiscrete_pre}, 
we arrive at the following result.

\begin{corollary}
\label{cor:discrete:algostab}
Under the Assumptions in Theorem~\ref{theorem_wassersteinbound_mainsection}
and Theorem~\ref{thm:discrete}, we have 
\begin{equation}
\mathcal{W}_{1}(\mu_{\eta},\hat{\mu}_{\eta})\leq\widetilde{C}\rho(X_n,\widehat{X}_n)+C\eta^{1/\alpha}+\widehat{C}\eta^{1/\alpha},
\end{equation}
where $\widetilde{C}$ is given in Theorem~\ref{theorem_wassersteinbound_mainsection}, and $C,\widehat{C}$ are given in Theorem~\ref{thm:discrete}..
\end{corollary}

As a corollary, one can also derive the generalization error bounds
for the discrete-time dynamics using Corollary~\ref{cor:discrete:algostab}
and the generalization error bounds 
for the continuous-time dynamics in Corollary~\ref{coro_generalizationbound}.

\begin{corollary}\label{cor:discrete:generalizationerrorbound}
Assume Conditions~\ref{cond_gammaandbeta},~\ref{cond_pseudolipschitz}, and~\ref{cond_allthelambdas}. Assume that $\ell$ is $L$-Lipschitz and $\sup_{x,y\in \mathcal{X}}\norm{x-y}\leq D$ for some $D<\infty$. Moreover, the step size $\eta$ satisfies
$\eta<\bar{\eta}$,
where $\bar{\eta}$ is an explicit constant given in \eqref{bar:eta:defn} in Appendix~\ref{appendix:discrete}. Then it holds that,
    \begin{align*}   &\left|\mathbb{E}_{\Theta_\infty,X_n}~\left[ \hat{R}(\Theta_\infty,X_n) \right] -  R(\Theta_\infty)  \right| \nonumber  \\
     &\leq \frac{1}{n}  \Big(d_1 D+d_2 D^{5/4}+d_3 D^{3/2}
     +d_4 D^{7/4}+d_5D^2+d_6 D^{5/2}\Big)+2L\eta^{1/\alpha}\brac{d_7+d_8\sqrt{D}+d_9D}, 
\end{align*}
where the constants $d_i$, $1\leq i\leq 9$, are independent of $D$ and their explicit formulas are provided in \eqref{def_d1throughd6} in Appendix~\ref{proof:main:results} and in \eqref{def_d7throughd9} in Appendix~\ref{appendix:discrete}. 
\end{corollary}
This result shows that, if $\eta$ is sufficiently small, the discrete-time process retains the generalization properties of the continuous-time SDE.

\section{Experiments}\label{sec:experiments}

%%%%%%%%%%%%%%%%%%%%%%%%%%%%%%
\paragraph{Synthetic data.}
We first consider the setting in Section~\ref{section_mainsquareloss} and test our theory on a linear model using synthetic data. We assume that  $x_i\overset{i.i.d.}{\sim}\mathcal{N}(0,\sigma_A)$, where $\mathcal{N}$ is a Gaussian distribution and $\sigma_A$ determines the distribution's standard deviation. In the synthetic data experiments we systematically vary $\sigma_A$ and the tail exponent of the noise, $\alpha$. Throughout experiments we fix the learning-rate $\eta = 0.05$ and set the momentum parameter $\gamma\in\{2.5, 5.0\}$ when utilizing SGDm, and train the models for $2000$ epochs. We set the sample size to $n = 1000$, and we conduct experiments across two dimensionalities with $d \in \{100, 250\}$. 
\noindent
\begin{minipage}[b][][b]{0.48\textwidth}
\centering
\includegraphics[width=\linewidth]{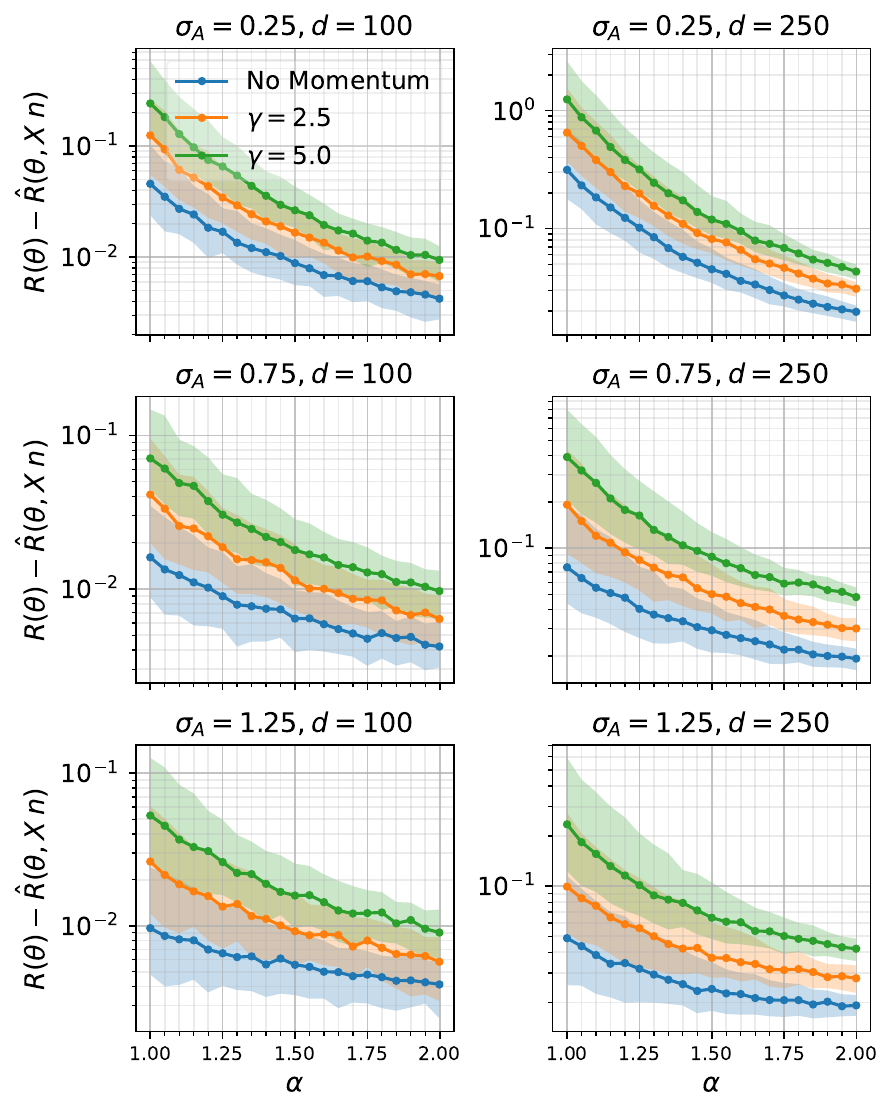}
    % \vspace{-2em}
    \captionof{figure}{Experiments comparing SGD with and without momentum on synthetic data with quadratic loss $f$. }
    \label{fig:synth-data}
\end{minipage} \hfill
% \begin{figure}[htb]
\begin{minipage}[b][][b]{0.45\textwidth}
    \centering
    % \vspace{-3\baselineskip}
	% \begin{subfigure}[t]{0.23\textwidth}
		\includegraphics[width=
		0.48\linewidth]{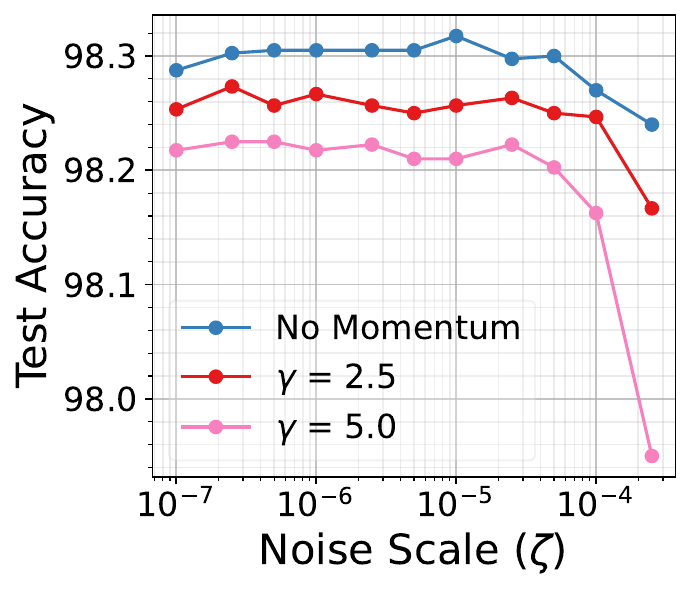}
        % \vspace{-1.75em}
        % \caption{}
        % \label{fig:mnist_fcn}
	% \end{subfigure}
	% \begin{subfigure}[t]{0.23\textwidth}
		\includegraphics[width=		0.48\linewidth]{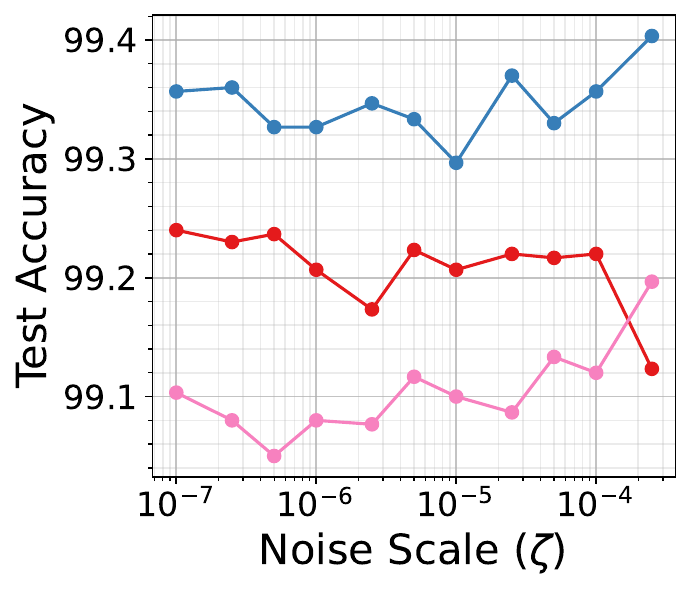}
        % \vspace{-1.75em}
        % \caption{}
        % \label{fig:mnist_vgg11}
	% \end{subfigure}
	% \begin{subfigure}[t]{0.23\textwidth}
		\includegraphics[width=		0.48\linewidth]{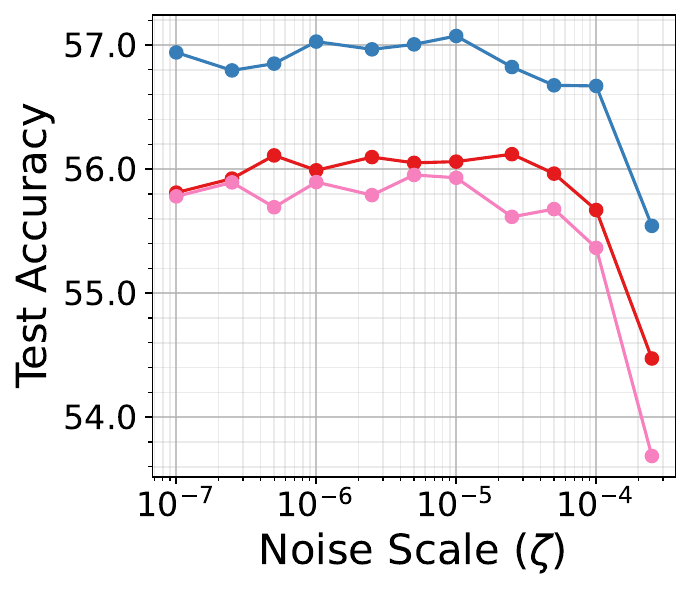}
        % \vspace{-1.75em}
        % \caption{}
        % \label{fig:cifar10_fcn}
	% \end{subfigure}
	% \begin{subfigure}[t]{0.23\textwidth}
		\includegraphics[width=		0.48\linewidth]{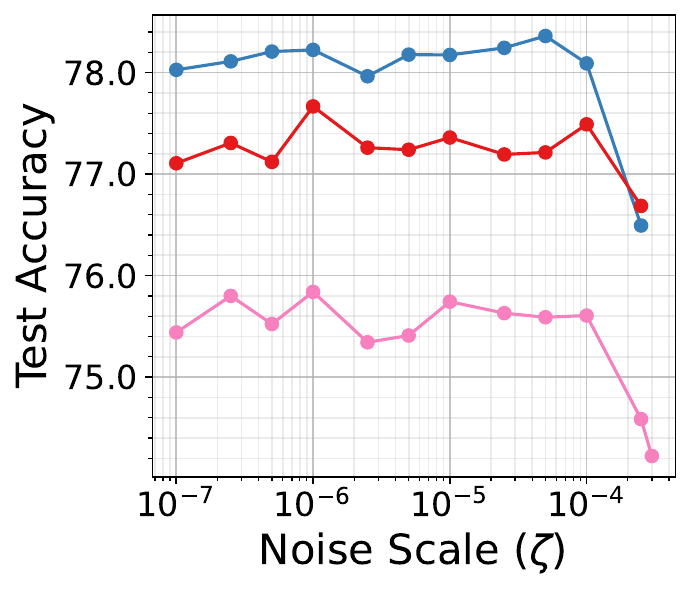}
        % \vspace{-1.75em}
        % \caption{}
        % \label{fig:cifar10_vgg11}
	% \end{subfigure}
    % \vspace{-1em}
    \captionof{figure}{Comparing SGD with and without momentum, using the following model-dataset combinations: (top left) MNIST - FCN, (top right) MNIST - CNN, (bottom left) CIFAR-10 - FCN, and (bottom right) CIFAR-10 - CNN.}
    \label{fig:mnist-cifar10}
\end{minipage}

The surrogate loss is chosen as $\ell(\theta,x) = |\theta^\top x|$. Experiments in each configuration was repeated for $250$ different random seeds. For each replication, the size of the test set was set to $10000$, sampled independently from the training set. The generalization gap was computed to be the difference between test and training  losses. To mitigate numerical issues, $\zeta$ parameter was selected so that the overall added noise scale was identical across runs with and without momentum, which implies an additional scaling of $1/\eta$ for SGDm.

The results are presented in Figure~\ref{fig:synth-data}, where we plot the median of the generalization gap computed for all replications in each setting, with the shaded area representing the interquartile range. The results are clearly in support of our hypothesis. Across all selections of variance, dimension, and tail index, SGD surpasses SGDm in having a smaller generalization gap. Furthermore, we observe that the generalization error decreases as we decrease $\gamma$, which is also in line with our theory. Having confirmed our theoretical predictions in synthetic data, we now investigate if our conclusions apply beyond our theoretical setting.

%\vspace{-0.2in}
%%%%%%%%%%%%%%%%%%%%%%%%%%%%%%
\paragraph{MNIST and CIFAR-10.}
We demonstrate our results on frequently used image classification datasets MNIST \cite{lecunnMNIST1998} and CIFAR-10 \cite{krizhevskyImageNetClassification2017}. We test our hypothesis under training with two different architectures: a fully connected network (FCN) and a convolutional neural network (CNN). The FCN includes one hidden layer of width $5000$ with ReLU activation, while the CNN is a slightly simplified version of the VGG11 architecture \cite{simonyanVeryDeep2015}, both trained with cross-entropy loss. The training was similar to above, where we use an SGD with or without momentum, with a constant learning rate of 0.05. The models were trained until $100\%$ accuracy, in rare cases where a model does not reach $100\%$ training accuracy due to added noise, we include the model in our results if it has a final training accuracy of $>97.5\%$. All the results are the average of $3$ random seeds. See Appendix~\ref{appendix:experimental_details}
for further details regarding our setup.

The results are presented in Figure~\ref{fig:mnist-cifar10}. Given equal (and in rare cases near-equal) training accuracy, the differences between test accuracy are equivalent to generalization gap. Here we again see a clear advantage for SGD in comparison to SGDm, where the performance of SGDm gracefully degrades for increasing $\gamma$; hence, providing another clear support towards our theoretical predictions.

%%%%%%%%%%%%%%%%%%%%%%%%%%%%%%%%%
\section{Conclusion}
In this work, we established generalization bounds for SGD with momentum (SGDm) under heavy-tailed noise through the lens of uniform stability. Analyzing the continuous-time limit of SGDm as a Lévy-driven SDE, we first derived stability bounds for 
a class of non-convex loss functions. Remarkably, our results showed that for quadratic losses SGDm admits a generalization bound that is always worse than that of SGD without momentum, highlighting that the interaction of heavy tails and momentum can be harmful for generalization. Extending our analysis to discrete-time, we then developed a novel discretization error bound, showing that with appropriate step-sizes, the discrete dynamics retain the SDE’s generalization properties. Finally, we validated our findings on quadratic problems and neural networks. \looseness=-1

\textbf{Limitations and future work.} Our results illustrate that momentum worsens generalization under heavy-tailed noise. However, at this stage we are not able to explain \emph{why} this happens and we leave the finer understanding of the interaction between heavy tails and momentum for future work. On the other hand, \citet{liu2023breaking}, showed that with gradient clipping, momentum can achieve faster rates on the training error under heavy-tailed noise. The link between convergence speed and the generalization error is also yet to be understood.  \looseness=-1

\section*{Acknowledgements}
A K M Rokonuzzaman Sonet and Lingjiong Zhu are partially supported by the NSF grant DMS-2053454.
Mert G\"urb\"uzbalaban’s research is supported in part by the Office of Naval Research Award
Number N00014-24-1-2628. Umut \c{S}im\c{s}ekli's research is partially supported by the European Research Council Starting Grant
DYNASTY – 101039676 and the management
of Agence Nationale de la Recherche as part of the “France 2030” program, reference
ANR-23-IACL-0008 (PR[AI]RIE-PSAI).
Lingjiong Zhu is also partially supported by the NSF grant DMS-2208303.

\bibliography{refs}

%%%%%%%%%%%%%%%%%%%%%%%%%%%%%%%%%%%%%%%%%%%%%%%%%%%%%%%%%%%%%%%%%%%%%%%%%%%%%%%
%%%%%%%%%%%%%%%%%%%%%%%%%%%%%%%%%%%%%%%%%%%%%%%%%%%%%%%%%%%%%%%%%%%%%%%%%%%%%%%
% APPENDIX
%%%%%%%%%%%%%%%%%%%%%%%%%%%%%%%%%%%%%%%%%%%%%%%%%%%%%%%%%%%%%%%%%%%%%%%%%%%%%%%
%%%%%%%%%%%%%%%%%%%%%%%%%%%%%%%%%%%%%%%%%%%%%%%%%%%%%%%%%%%%%%%%%%%%%%%%%%%%%%%
\newpage
\appendix
\onecolumn

\begin{center}
\Large \bf Appendix
\end{center}

The Appendix is organized as follows.
\begin{itemize}
\item
In Appendix~\ref{proof:main:results}, we provide
the proofs of the results in Section~\ref{sec:main:results}
for the general loss function.
\item 
In Appendix~\ref{appendix_squareloss}, we provide
the proofs of the results in Section~\ref{section_mainsquareloss} for the quadratic loss function.
\item 
In Appendix~\ref{appendix:discrete}, we provide
the proofs of the results in Section~\ref{sec:discrete}
for the discrete-time analysis.
\item 
In Appendix~\ref{appendix:experimental_details}, we provide
additional details regarding our experiments presented in Section~\ref{sec:experiments}.
\end{itemize}

%%%%%%%%%%%%%%%%%%%%%%%%%%%%%%%%%%%%%%%%%%%%%%%%%%%%%%%%%%%%%%%%%%%%%%%%%%%%%%%
\section{Proofs of the Results in Section~\ref{sec:main:results}}\label{proof:main:results}

\subsection{Notations}
\label{section_notationsinappendix}

Let us first introduce some notations
that will be used in the proofs of results in 
Section~\ref{sec:main:results}.

\begin{itemize}
\item
Let $\{P_t:t\geq 0\}$ and $\{\widehat{P}_t:t\geq 0\}$ denote respectively the semigroups associated with the process $\left\{\brac{\theta_t,v_t}:t\geq 0\right\}$ given in \eqref{sde_generalloss} and the process $\left\{\brac{\hat{\theta}_t,\hat{v}_t}:t\geq 0\right\}$ given in \eqref{sde_generalloss_differentdataset}. 
\item 
For the process $\left\{\brac{\theta_t,v_t}:t\geq 0\right\}$ given in \eqref{sde_generalloss} and the process $\left\{\brac{\hat{\theta}_t,\hat{v}_t}:t\geq 0\right\}$ given in \eqref{sde_generalloss_differentdataset}, we write
\begin{equation*}
\brac{\theta_t,v_t}=\brac{\theta_t^{w,y},v_t^{w,y}},
\qquad
\brac{\hat{\theta}_t,\hat{v}_t}=\brac{\hat{\theta}_t^{w,y},\hat{v}_t^{w,y}},
\end{equation*}
for any $t\geq 0$ to emphasize the dependence on the initialization $\brac{\theta_0,v_0}=\brac{\hat{\theta}_0,\hat{v}_0}=(w,y)$.
\item
The operator norm $\norm{\cdot}_{\operatorname{op}}$ of a linear map $T:\mathbb{R}^{d}\to\mathbb{R}^{d}$ is defined as 
\begin{align*}
\norm{T}_{\operatorname{op}}:=\sup_{v\in\mathbb{R}^{d}:\norm{v}=1}\norm{Tv}. 
\end{align*}
\item 
Per \citep[Section 2.1]{bookcartan1983differential}, any differentiable function $f:\R^d\to\R$ and a choice of $x\in\R^d$ induces a linear map $\nabla f(x):\R^d\to \R$. 
This allows us to define 
the operator norm
\begin{align*}
   \norm{ \nabla f(x)}_{\operatorname{op}}:=\sup_{y\in \R^d: \norm{y}= 1} \norm{\inner{ \nabla f(x), y}},
\end{align*}
and the supremum norm 
\begin{align*}
    \norm{ \nabla f(x)}_{\operatorname{op},\infty}=\sup_{x\in\R^d}\sup_{y\in \R^d: \norm{y}= 1} \norm{\inner{ \nabla f(x), y}}. 
\end{align*}
If $f$ is twice-differentiable, similar definitions of norms can be introduced to $\nabla^2 f(x)$ as a linear map: $\R^{d}\times \R^d\to\R$. A formal introduction to higher derivatives can be found in \citep[Section 5.1]{bookcartan1983differential}. 

\item
For any two real numbers $x,y$, we denote $x\vee y:=\max\{x,y\}$.
\item
Denote $\operatorname{Lip}(1)$ the space of $1$-Lipschitz functions from $\mathbb{R}^{2d}$ to $\mathbb{R}$.
\end{itemize}

%%%%%%%%%%%%%%%%%%%%%%%%%%%%%%%%%%%%%%%%%%%%
\subsection{Proof of Theorem~\ref{theorem_wassersteinbound_mainsection}}
\label{section_proofofmaintheorem}

\begin{theorem}
\label{theorem_wassersteinbound_continuousdynamics_appendix}
[Restatement of Theorem~\ref{theorem_wassersteinbound_mainsection}]\label{theorem_wassersteinbound_appendix}
    Assume Conditions~\ref{cond_gammaandbeta},~\ref{cond_pseudolipschitz}, and~\ref{cond_allthelambdas}. The following two statements hold.
    \begin{enumerate}%[label={\roman*})]
        \item For every positive integer $N$ and $\eta\in (0,1)$, 
\begin{align*}
&\mathcal{W}_1\brac{\mathrm{Law}\brac{\theta^{w,y}_{N\eta},v^{w,y}_{N\eta}},\mathrm{Law}\brac{\hat{\theta}^{w,y}_{N\eta},\hat{v}^{w,y}_{N\eta}}}\\
&\leq (1+\gamma)\Bigg[C_2\brac{1+\gamma+\norm{\nabla f(0,0)}+K_1+\frac{K_2}{n}\sum_{i=1}^n \norm{x_i} }\\
&\qquad\qquad\qquad\qquad\qquad\cdot \brac{1+\widehat{P}(w,y)+\mathcal{M}(w,y)}+\frac{K_2}{n}\sum_{i=1}^n \norm{x_i}+\zeta\E{\norm{L_1}}\Bigg]\eta^{1+1/{\alpha}}\\
           &\qquad+(1+\gamma)\Bigg[C_2\brac{1+\gamma+\norm{\nabla f(0,0)}+K_1+\frac{K_2}{n}\sum_{i=1}^n \norm{\hat{x}_i} }\\
           &\qquad\qquad\qquad\qquad\qquad\cdot \brac{1+\widehat{P}(w,y)+\widehat{\mathcal{M}}(w,y)}+\frac{K_2}{n}\sum_{i=1}^n \norm{\hat{x}_i}+\zeta\E{\norm{L_1}}\Bigg]\eta^{1+1/{\alpha}}\\
    &\qquad\quad+K_2 \rho(X_n,\widehat{X}_n)\Bigg[2C_2 +2C_2\widehat{P}(w,y)
+C_2\mathcal{M}(w,y)+C_2\widehat{\mathcal{M}}(w,y)+1\Bigg]\eta\\
&\quad+ \frac{C_*}{\lambda_*}\Bigg\{(1+\gamma)\Bigg[C_2\left(1+\gamma+\norm{\nabla f(0,0)}+K_1+\frac{K_2}{n}\sum_{j=1}^n \norm{x_j} \right)\\
&\qquad\qquad\qquad\qquad\qquad\cdot \brac{1+\widehat{\mathcal{P}}(w,y) +\mathcal{M}(w,y)}+\frac{K_2}{n}\sum_{j=1}^n \norm{x_j}+\zeta\E{\norm{L_1}}\Bigg]\eta^{1/{\alpha}}
    \\
    &\qquad\quad+ (1+\gamma)\Bigg[C_2\left(1+\gamma+\norm{\nabla f(0,0)}+K_1+\frac{K_2}{n}\sum_{j=1}^n \norm{\hat{x}_j} \right)\\
    &\qquad\qquad\qquad\qquad\qquad\cdot \brac{1+\widehat{\mathcal{P}}(w,y) +\widehat{\mathcal{M}}(w,y) }+\frac{K_2}{n}\sum_{j=1}^n \norm{\hat{x}_j}+\zeta\E{\norm{L_1}}\Bigg]\eta^{1/{\alpha}}
           \\  &\qquad\qquad\quad+K_2\rho(X_n,\widehat{X}_n)\Bigg[2C_2+2C_2\widehat{\mathcal{P}}(w,y)+C_2\mathcal{M}(w,y)+C_2\widehat{\mathcal{M}}(w,y)+1\Bigg]\Bigg\}. 
\end{align*}
The constants $\gamma,K_1,K_2$ and the function $\rho$ are specified in Condition~\ref{cond_pseudolipschitz} and Condition~\ref{cond_gammaandbeta}. $C_2$ is defined in \eqref{def_C2}. $C_*$ and $\lambda_*$ are specified in Lemma~\ref{lemma_wassersteindecaybyBaoWang}. The functions $\mathcal{M}(\cdot,\cdot),\widehat{\mathcal{M}}(\cdot,\cdot)$ are respectively defined as
\begin{align}
\label{def_M}
   &\mathcal{M}(w,y)
   \nonumber
   \\
   &:= \sqrt{\max_{1\leq i\leq n}\abs{f(0,x_i)}}
   \nonumber
   \\
   &\hspace{0.5em}+\sqrt{\frac{K_2\max_{1\leq i\leq n}\norm{x_i}+1}{2}}C_2\brac{1+\norm{w}+\norm{y}+\max_{1\leq i\leq n}\sqrt{\abs{f\brac{w,\hat{x}_i}}}}\nonumber\\
   &\hspace{1em}+\sqrt{K_2\max_{1\leq i\leq n}\norm{x_i} +\norm{\nabla f(0,0) }}\sqrt{C_2}\brac{1+\sqrt{\norm{w}}+\sqrt{\norm{y}}+\max_{1\leq i\leq n}\sqrt[4]{\abs{f\brac{w,\hat{x}_i}}}},
\end{align}
and
\begin{align}
\label{def_hatM}
     &\widehat{\mathcal{M}}(w,y)\nonumber
     \\
     &:= \sqrt{\max_{1\leq i\leq n}\abs{f(0,\hat{x}_i)}}\nonumber
     \\
     &\hspace{0.5em}+\sqrt{\frac{K_2\max_{1\leq i\leq n}\norm{\hat{x}_i}+1}{2}}C_2\brac{1+\norm{w}+\norm{y}+\max_{1\leq i\leq n}\sqrt{\abs{f\brac{w,\hat{x}_i}}}}\nonumber\\
     &\hspace{1em}+\sqrt{K_2\max_{1\leq i\leq n}\norm{\hat{x}_i} +\norm{\nabla f(0,0) }}\sqrt{C_2}\brac{1+\sqrt{\norm{w}}+\sqrt{\norm{y}}+\max_{1\leq i\leq n}\sqrt[4]{\abs{f\brac{w,\hat{x}_i}}}},
\end{align}
and $\widehat{\mathcal{P}}(\cdot,\cdot)$ is defined as
\begin{align}
\label{def_hatP}
    \widehat{P}(w,y):=C_2\brac{1+\norm{w}+\norm{y}+\max_{1\leq i\leq n}\sqrt{\abs{f\brac{w,\hat{x}_i}}}}. 
\end{align}

\item If $\mu,\hat{\mu}$ denote respectively the invariant measures of the process $\{\brac{\theta_t,v_t}:t\geq 0\}$ and the process $\left\{\brac{\hat{\theta}_t,\hat{v}_t}:t\geq 0\right\}$, then
\begin{align*}
    \mathcal{W}_1\brac{\mu,\hat{\mu}}
    &\leq \rho(X_n,\widehat{X}_n)\cdot \widetilde{C}, 
\end{align*}
where $\widetilde{C}$ is defined as 
\begin{align*}
&\widetilde{C}:=\frac{C_*K_2}{\lambda_*}\Bigg\{2C_2+2C_2^2\brac{1+\sqrt{\hat{m}_{1}}}+C_2\Bigg(\sqrt{m_{1}}+\sqrt{\frac{K_2m_{2}+1}{2}}C_2\brac{1+\sqrt{\hat{m}_{1}}}\nonumber\\
   &\quad\qquad\qquad\qquad+\sqrt{K_2m_{2} +\norm{\nabla f(0,0) }}\sqrt{C_2}\brac{1+\sqrt[4]{\hat{m}_{1}}}\Bigg)
   \\
   &\qquad\qquad+C_2 \Bigg(\sqrt{\hat{m}_{1}}+\sqrt{\frac{K_2\hat{m}_{2}+1}{2}}C_2\brac{1+\sqrt{\hat{m}_{1}}}\nonumber\\
     &\qquad\qquad\qquad\qquad\qquad+\sqrt{K_2\hat{m}_{2} +\norm{\nabla f(0,0) }}\sqrt{C_2}\brac{1+\sqrt[4]{\hat{m}_{1}}} \Bigg)+1\Bigg\}, 
\end{align*}
where
\begin{align}
&m_{1}:=\max_{1\leq i\leq n}|f(0,x_{i})|, 
\quad\hat{m}_{1}:=\max_{1\leq i\leq n}|f(0,\hat{x}_{i})|,
\\
&m_{2}:=\max_{1\leq i\leq n}\|x_{i}\|,
\quad\hat{m}_{2}:=\max_{1\leq i\leq n}\|\hat{x}_{i}\|,
\end{align}
and $\rho(X_n,\widehat{X}_{n})$ is defined in \eqref{defn:rho} and $K_2$ is defined in Condition~\ref{cond_pseudolipschitz}, and $C_2$ is defined as 
\begin{align}
\label{def_C2mainsection}
   C_2&=\min \left\{\sqrt{ \frac{r^2-r_0^2}{8}}, \sqrt{\frac{r^2-r_0^2}{8r^2}} \right\}^{-1}
\cdot\max \left\{\sqrt{\beta},\sqrt{\beta\lambda_4+r^2},1+\sqrt{\beta\lambda_5+1}+\frac{C_0}{c_0} \right\}, 
\end{align}
where
$r=\brac{\frac{\gamma^2}{2}+\frac{\gamma}{2} \sqrt{\beta(\lambda_1-\lambda_2\lambda_4)} -\beta\lambda_4}^{1/2}$ and $r_0=\frac{\gamma}{2}$. Moreover, the constants $\lambda_*$ and $C_*$ are given by
\begin{align}
\label{def_C*lambda*mainsection}
     \lambda_*=\min \left\{ \frac{c_0\epsilon}{1+2\epsilon},\frac{3c_1\brac{1-\frac{1}{\alpha}}\gamma}{8(1+c_1)} \right\};\quad
     C_*=  c_0^2\sqrt{2d}. 
 \end{align}
In addition, the constants $\lambda_4,\lambda_5$ are from Condition~\ref{cond_allthelambdas}. Meanwhile, the constants $c_0$ and $C_0$ (which are from Lemma~\ref{lemma_lyapunov} in the Appendix) depend only on the parameters $\alpha,\gamma,\zeta,\beta$, dimension $d$ plus $\lambda_i, 1\leq i\leq 5$ in Condition~\ref{cond_allthelambdas}, and does not depend on the dataset $X_n$.
    \end{enumerate}
\end{theorem}

\begin{proof}
The proof is inspired by the proof strategies in \cite{generalloss}[Proof of Theorem 3.3], \cite{chenxu2023euler}[Proof of Theorem 1.2] (see also \cite{chenxu23amarkovapproach}). 

To prove Part i), we start with a decomposition of the semigroups that is in the spirit of the classical Lindeberg's principle (also known as Lindeberg exchange method):
\begin{align*}
    &P_{N\eta}h(w,y)-\widehat{P}_{N\eta}h(w,y)=\sum_{i=1}^{N}\widehat{P}_{(i-1)\eta}\brac{P_\eta-\widehat{P}_\eta }P_{(N-i)\eta}h(w,y).
\end{align*}
While the above decomposition appears simple, it provides a powerful way to obtain some significant results about probabilistic approximation of Markov process. The idea was first proposed in \cite{chenxu23amarkovapproach} and has been successfully applied to various probabilistic approximation problems in \cite{chenxu2023euler,chenxufollowuppaper, generalloss,jin2024approximation,deng2024total}. 

Based on the above equation, we can write
\begin{align*}
        &\sup_{h\in \operatorname{Lip}(1)} \abs{P_{N\eta}h(w,y)-\widehat{P}_{N\eta}h(w,y) }\\
        &\leq \sup_{h\in \operatorname{Lip}(1)} \abs{\widehat{P}_{(N-1)\eta}\brac{P_\eta-\widehat{P}_\eta }h(w,y) }+\sup_{h\in \operatorname{Lip}(1)} \sum_{i=1}^{N-1}\abs{\widehat{P}_{(i-1)\eta}\brac{P_\eta-\widehat{P}_\eta }P_{(N-i)\eta}h(w,y) }\\
        &=:\mathcal{A}_1+\mathcal{A}_2. 
    \end{align*}

Regarding the term $\mathcal{A}_1$, Lemma~\ref{lemma_onestepestimate} implies that
\begin{align*}
  &\mathcal{A}_1\leq \norm{\nabla h}_{\operatorname{op},\infty}  \Bigg\{ (1+\gamma)\Bigg[C_2\brac{1+\gamma+\norm{\nabla f(0,0)}+K_1+\frac{K_2}{n}\sum_{i=1}^n \norm{x_i} }\\
           &\qquad\qquad\cdot \brac{1+\E{\norm{\hat{\theta}^{w,y}_{(N-1)\eta}}}+\E{\norm{\hat{v}^{w,y}_{(N-1)\eta}}}+\max_{1\leq i\leq n}\E{\sqrt{f\brac{\hat{\theta}^{w,y}_{(N-1)\eta},x_i }}}}\\
           &\hspace{23em}+\frac{K_2}{n}\sum_{i=1}^n \norm{x_i}+\zeta\E{\norm{L_1}}\Bigg]\eta^{1+1/{\alpha}}\\
           &\qquad\quad+(1+\gamma)\Bigg[C_2\brac{1+\gamma+\norm{\nabla f(0,0)}+K_1+\frac{K_2}{n}\sum_{i=1}^n \norm{\hat{x}_i} }\\
           &\qquad\qquad\cdot \brac{1+\E{\norm{\hat{\theta}^{w,y}_{(N-1)\eta}}}+\E{\norm{\hat{v}^{w,y}_{(N-1)\eta}}}+\max_{1\leq i\leq n}\E{\sqrt{f\brac{\hat{\theta}^{w,y}_{(N-1)\eta},\hat{x}_i }}}}\\
           &\hspace{23em}+\frac{K_2}{n}\sum_{i=1}^n \norm{\hat{x}_i}+\zeta\E{\norm{L_1}}\Bigg]\eta^{1+1/{\alpha}}\\
    &\qquad\qquad+K_2 \rho(X_n,\widehat{X}_n)\Bigg[C_2\Bigg(2+2\E{\norm{\hat{\theta}^{w,y}_{(N-1)\eta}}}+2\E{\norm{\hat{v}^{w,y}_{(N-1)\eta}}}
    \\
    &\qquad\qquad\qquad+\max_{1\leq i\leq n}\E{\sqrt{\abs{f\brac{\hat{\theta}_{(N-1)\eta},x_i}}}}+\max_{1\leq i\leq n}\E{\sqrt{\abs{f\brac{\hat{\theta}_{(N-1)\eta},\hat{x}_i}}}}\Bigg)+1\Bigg]\eta\Bigg\}.
\end{align*}
Combining this with the fact that $h\in \operatorname{Lip}(1)$ and the moment estimates in Lemma~\ref{lemma_uniformmomentbound_continuousdynamics}, Lemma~\ref{lemma_momentbound_squarerootf} to get
\begin{align*}
\mathcal{A}_1&\leq
(1+\gamma)\Bigg[C_2\brac{1+\gamma+\norm{\nabla f(0,0)}+K_1+\frac{K_2}{n}\sum_{i=1}^n \norm{x_i} }\\
           &\qquad\qquad\cdot \brac{1+\widehat{P}(w,y)+\mathcal{M}(w,y)}+\frac{K_2}{n}\sum_{i=1}^n \norm{x_i}+\zeta\E{\norm{L_1}}\Bigg]\eta^{1+1/{\alpha}}\\
           &\qquad+(1+\gamma)\Bigg[C_2\brac{1+\gamma+\norm{\nabla f(0,0)}+K_1+\frac{K_2}{n}\sum_{i=1}^n \norm{\hat{x}_i} }\\
           &\qquad\qquad\cdot \brac{1+\widehat{P}(w,y)+\widehat{\mathcal{M}}(w,y)}+\frac{K_2}{n}\sum_{i=1}^n \norm{\hat{x}_i}+\zeta\E{\norm{L_1}}\Bigg]\eta^{1+1/{\alpha}}\\
    &\qquad\quad+K_2 \rho(X_n,\widehat{X}_n)\Bigg[2C_2 +2C_2\widehat{P}(w,y)
+C_2\mathcal{M}(w,y)+C_2\widehat{\mathcal{M}}(w,y)+1\Bigg]\eta,
\end{align*}
where $\mathcal{M}(w,y),\widehat{\mathcal{M}}(w,y)$ and $\widehat{P}(w,y)$ are respectively defined in \eqref{def_M}, \eqref{def_hatM} and \eqref{def_hatP}.  

Regarding the term $\mathcal{A}_2$, Lemma~\ref{lemma_onestepestimate} implies that for any $h\in \operatorname{Lip}(1)$, we have 
\begin{align*}
  &\abs{ \brac{P_\eta-\widehat{P}_\eta }P_{(N-i)\eta}h(w,y)}
  \\
&\leq \norm{\nabla P_{(N-i)\eta} h}_{\operatorname{op},\infty}\Bigg\{(1+\gamma)\Bigg[C_2\Big(1+\gamma+\norm{\nabla f(0,0)}+K_1+\frac{K_2}{n}\sum_{j=1}^n \norm{x_j} \Big)\\
           &\hspace{0.5em}\cdot \brac{1+\norm{w}+\norm{y}+\max_{1\leq j\leq n}\sqrt{\abs{f\brac{w,x_j}}}}+\frac{K_2}{n}\sum_{i=1}^n \norm{x_i}+\zeta\E{\norm{L_1}}\Bigg]\eta^{1+1/{\alpha}}
    \\
    &\qquad+ (1+\gamma)\Bigg[C_2\Big(1+\gamma+\norm{\nabla f(0,0)}+K_1+\frac{K_2}{n}\sum_{j=1}^n \norm{\hat{x}_j} \Big)\\
           &\hspace{0.5em}\cdot \brac{1+\norm{w}+\norm{y}+\max_{1\leq j\leq n}\sqrt{\abs{f\brac{w,\hat{x}_j}}}}+\frac{K_2}{n}\sum_{j=1}^n \norm{\hat{x}_j}+\zeta\E{\norm{L_1}}\Bigg]\eta^{1+1/{\alpha}}
           \\
           &\qquad\qquad+K_2 \rho(X_n,\widehat{X}_n)\Bigg[C_{2}\bigg(2+2\norm{w}+2\norm{y}+\max_{1\leq j\leq n}\sqrt{\abs{f\brac{w,x_j}}}\\
           &\hspace{19em}+\max_{1\leq j\leq n}\sqrt{\abs{f\brac{w,\hat{x}_j}}}\bigg)+1\Bigg]\eta\Bigg\}. 
\end{align*}
Moreover, we know from Lemma~\ref{lemma_semigroupgradiateestimate} that 
\begin{align*}
    \norm{\nabla P_{(N-i)\eta} h}_{\operatorname{op},\infty}\leq \norm{\nabla h}_{\operatorname{op},\infty}   C_* \exp\brac{-\lambda_*(N-i)\eta},
\end{align*}
so that 
\begin{align*}
  \mathcal{A}_2
  &\leq \sum_{i=1}^{N-1} C_* \exp\brac{-\lambda_*(N-i)\eta}\Bigg\{(1+\gamma)\Bigg[C_2\Big(1+\gamma+\norm{\nabla f(0,0)}+K_1+\frac{K_2}{n}\sum_{j=1}^n \norm{x_j} \Big)\\
           &\qquad\cdot \brac{1+\E{\norm{\hat{\theta}^{w,y}_{(N-i)\eta}}}+\E{\norm{\hat{v}^{w,y}_{(N-i)\eta}}} +\max_{1\leq i\leq n}\E{\sqrt{f\brac{\hat{\theta}^{w,y}_{(N-i)\eta},x_i }}}}\\
           &\hspace{20em}+\frac{K_2}{n}\sum_{j=1}^n \norm{x_j}+\zeta\E{\norm{L_1}}\Bigg]\eta^{1+1/{\alpha}}
    \\
    &\qquad\qquad\qquad+ (1+\gamma)\Bigg[C_2\Big(1+\gamma+\norm{\nabla f(0,0)}+K_1+\frac{K_2}{n}\sum_{j=1}^n \norm{\hat{x}_j} \Big)\\
           &\qquad\cdot \brac{1+\E{\norm{\hat{\theta}^{w,y}_{(N-i)\eta}}}+\E{\norm{\hat{v}^{w,y}_{(N-i)\eta}}} +\max_{1\leq i\leq n}\E{\sqrt{f\brac{\hat{\theta}^{w,y}_{(N-i)\eta},\hat{x}_i }}}}\\
           &\hspace{20em}+\frac{K_2}{n}\sum_{j=1}^n \norm{\hat{x}_j}+\zeta\E{\norm{L_1}}\Bigg]\eta^{1+1/{\alpha}}
           \\
           &\qquad\qquad+K_2 \rho(X_n,\widehat{X}_n)\Bigg[C_{2}\Bigg(2+2\E{\norm{\hat{\theta}^{w,y}_{(N-i)\eta}}}+2\E{\norm{\hat{v}^{w,y}_{(N-i)\eta}}}\\
           &\hspace{5em}+\max_{1\leq j\leq n}\E{\sqrt{f\brac{\hat{\theta}^{w,y}_{(N-i)\eta},x_j }}}+\max_{1\leq j\leq n}\E{\sqrt{f\brac{\hat{\theta}^{w,y}_{(N-i)\eta},\hat{x}_j }}}\Bigg)+1\Bigg]\eta\Bigg\}. 
\end{align*}

It follows from the uniform moment estimates in Lemma~\ref{lemma_uniformmomentbound_continuousdynamics} and Lemma~\ref{lemma_momentbound_squarerootf} that
\begin{align*}
  &\mathcal{A}_2\leq \sum_{i=1}^{N-1} C_* \exp\brac{-\lambda_*(N-i)\eta}\Bigg\{(1+\gamma)\Bigg[C_2\Big(1+\gamma+\norm{\nabla f(0,0)}+K_1+\frac{K_2}{n}\sum_{j=1}^n \norm{x_j} \Big)\\
           &\qquad\qquad\qquad\qquad\cdot \brac{1+\widehat{\mathcal{P}}(w,y) +\mathcal{M}(w,y)}+\frac{K_2}{n}\sum_{j=1}^n \norm{x_j}+\zeta\E{\norm{L_1}}\Bigg]\eta^{1+1/{\alpha}}
    \\
    &\qquad+ (1+\gamma)\Bigg[C_2\left(1+\gamma+\norm{\nabla f(0,0)}+K_1+\frac{K_2}{n}\sum_{j=1}^n \norm{\hat{x}_j} \right)\cdot \brac{1+\widehat{\mathcal{P}}(w,y) +\widehat{\mathcal{M}}(w,y) }\\
           &\qquad\qquad\qquad\qquad\qquad+\frac{K_2}{n}\sum_{j=1}^n \norm{\hat{x}_j}+\zeta\E{\norm{L_1}}\Bigg]\eta^{1+1/{\alpha}}
           \\
&\qquad\qquad\qquad+K_2\rho(X_n,\widehat{X}_n)\Bigg[2C_2+2C_2\widehat{\mathcal{P}}(w,y)+C_2\mathcal{M}(w,y)+C_2\widehat{\mathcal{M}}(w,y)+1\Bigg]\eta\Bigg\}.
\end{align*}
 
Since 
\begin{align*}
    \sum_{i=1}^{N-1} \exp\brac{-\lambda_*(N-i)\eta}\leq \exp\brac{-\lambda_*N}\int_1^N \exp\brac{\lambda_* \eta x}dx\leq \frac{1}{\lambda_* \eta}, 
\end{align*}
we can deduce that
\begin{align*}
  \mathcal{A}_2&\leq \frac{C_*}{\lambda_*}\Bigg\{(1+\gamma)\Bigg[C_2\Big(1+\gamma+\norm{\nabla f(0,0)}+K_1+\frac{K_2}{n}\sum_{j=1}^n \norm{x_j} \Big)\\
           &\qquad\qquad\cdot \brac{1+\widehat{\mathcal{P}}(w,y) +\mathcal{M}(w,y)}+\frac{K_2}{n}\sum_{j=1}^n \norm{x_j}+\zeta\E{\norm{L_1}}\Bigg]\eta^{1/{\alpha}}
    \\
    &\qquad\qquad+ (1+\gamma)\Bigg[C_2\Big(1+\gamma+\norm{\nabla f(0,0)}+K_1+\frac{K_2}{n}\sum_{j=1}^n \norm{\hat{x}_j} \Big)\\
           &\qquad\qquad\qquad\qquad\cdot \brac{1+\widehat{\mathcal{P}}(w,y) +\widehat{\mathcal{M}}(w,y) }+\frac{K_2}{n}\sum_{j=1}^n \norm{\hat{x}_j}+\zeta\E{\norm{L_1}}\Bigg]\eta^{1/{\alpha}}
           \\
     &\qquad\qquad\qquad+K_2\rho(X_n,\widehat{X}_n)\Bigg[2C_2+2C_2\widehat{\mathcal{P}}(w,y)+C_2\mathcal{M}(w,y)+C_2\widehat{\mathcal{M}}(w,y)+1\Bigg]\Bigg\}. 
\end{align*}

Combining the bounds on $\mathcal{A}_1$ and $\mathcal{A}_2$ yields the desired result in Part i). 

Part ii) is a simple consequence of Part i). Indeed, observe that
\begin{align*}
\mathcal{W}_1(\mu,\hat{\mu})
&\leq \mathcal{W}_1\brac{\mu,\mathrm{Law}\brac{\hat{\theta}^{w,y}_{N\eta},\hat{v}^{w,y}_{N\eta}}}
\\
&\qquad+\mathcal{W}_1\brac{\mathrm{Law}\brac{\theta^{w,y}_{N\eta},v^{w,y}_{N\eta}},\mathrm{Law}\brac{\hat{\theta}^{w,y}_{N\eta},\hat{v}^{w,y}_{N\eta}}}+\mathcal{W}_1\brac{\mathrm{Law}\brac{\theta^{w,y}_{N\eta},v^{w,y}_{N\eta}},\hat{\mu}}. 
\end{align*}
    We apply $\lim_{N\to\infty}$ on both hand sides in the above equation to get
    \begin{align*}
     \mathcal{W}_1(\mu,\hat{\mu})\leq    \lim_{N\to\infty}\mathcal{W}_1\brac{\mathrm{Law}\brac{\theta^{w,y}_{N\eta},v^{w,y}_{N\eta}},\mathrm{Law}\brac{\hat{\theta}^{w,y}_{N\eta},\hat{v}^{w,y}_{N\eta}}}. 
    \end{align*}
Since $\mathcal{W}_1(\mu,\hat{\mu})$ is independent of $\eta$ and initial condition $(w,y)$, we can set $\eta=0$ and $(w,y)=(0,0)$ to obtain
\begin{align*}
    &\mathcal{W}_1(\mu,\hat{\mu})\\
    &\leq    \lim_{N\to\infty}\mathcal{W}_1\brac{\mathrm{Law}\brac{\theta^{w,y}_{N\eta},v^{w,y}_{N\eta}},\mathrm{Law}\brac{\hat{\theta}^{w,y}_{N\eta},\hat{v}^{w,y}_{N\eta}}}\\
    &\leq  \rho(X_n,\widehat{X}_n)\frac{C_*K_2}{\lambda_*}\Bigg\{2C_2+2C_2^2\brac{1+\max_{1\leq i\leq n}\sqrt{\abs{f\brac{0,\hat{x}_i}}}}\\
    &\quad+C_2\Bigg(\sqrt{\max_{1\leq i\leq n}\abs{f(0,x_i)}}+\sqrt{\frac{K_2\max_{1\leq i\leq n}\norm{x_i}+1}{2}}C_2\brac{1+\max_{1\leq i\leq n}\sqrt{\abs{f\brac{0,\hat{x}_i}}}}\nonumber\\
   &\quad\quad+\sqrt{K_2\max_{1\leq i\leq n}\norm{x_i} +\norm{\nabla f(0,0) }}\sqrt{C_2}\brac{1+\max_{1\leq i\leq n}\sqrt[4]{\abs{f\brac{0,\hat{x}_i}}}}\Bigg)\\
    &\quad\quad\quad+C_2 \Bigg(\sqrt{\max_{1\leq i\leq n}\abs{f(0,\hat{x}_i)}}+\sqrt{\frac{K_2\max_{1\leq i\leq n}\norm{\hat{x}_i}+1}{2}}C_2\brac{1+\max_{1\leq i\leq n}\sqrt{\abs{f\brac{0,\hat{x}_i}}}}\nonumber\\
     &\quad\quad\quad\quad+\sqrt{K_2\max_{1\leq i\leq n}\norm{\hat{x}_i} +\norm{\nabla f(0,0) }}\sqrt{C_2}\brac{1+\max_{1\leq i\leq n}\sqrt[4]{\abs{f\brac{0,\hat{x}_i}}}} \Bigg)+1\Bigg\}. 
\end{align*}
This completes the proof.
\end{proof}

%%%%%%%%%%%%%%%%%%%%%%%%%%%%%%%%%%%
\subsection{Technical Lemmas}\label{section_technicallemmas}

We start with two important results from \cite{jianwangbao2022coupling}, i.e. their Lemma 4.1 and Corollary 1.4. Notice that compared to the equation considered in \citep[Corollary 1.4]{jianwangbao2022coupling}, our Equation~\eqref{sde_generalloss} has an additional parameter $\zeta$ in front of the $\alpha$-stable L\'{e}vy process. This is such a minor addition that unsurprisingly, the finding in \cite{jianwangbao2022coupling} still holds in our setting. Indeed, denote 
\begin{equation}
\Theta(dy):=\frac{\alpha 2^{\alpha-1}\Gamma(\frac{d+\alpha}{2})}{\pi^{d/2}\Gamma\brac{1-\frac{\alpha}{2}}}\frac{1}{\norm{y}^{\alpha+d
}}dy 
\end{equation}
as the L\'{e}vy measure of the $\alpha$-stable L\'{e}vy process $L_t$ then the  proof of Corollary 1.4 of \cite{jianwangbao2022coupling}(at the end of page 119) involves showing $\Theta(dy)$ satisfies their Condition $\operatorname{B_2}$. Writing $\phi_F$ for the characteristic function of a random variable $F$ then per \citep[Section 1.2.4]{applebaum2009levy}, $\Theta(dy)$ is the unique measure which satisfies
\begin{align*}
    \phi_{L_1}(u)=\E{\exp\brac{-iu\zeta L_1}}=\int_{\R^d}\brac{\exp\brac{i\inner{u,y}}-1-i\inner{u,y}\mathds{1}_{\{\norm{y}\leq 1 \} } }\Theta(dy). 
\end{align*}
This leads to
\begin{align*}
     \phi_{\zeta L_1}(u)= \phi_{ L_1}(\zeta u)&=\int_{\R^d}\brac{\exp\brac{i\inner{\zeta u,y}}-1-i\inner{\zeta u,y}\mathds{1}_{\{\norm{y}\leq 1 \} } }\Theta(dy)\\
     &=\int_{\R^d}\brac{\exp\brac{i\inner{ u,y}}-1-i\inner{ u,y}\mathds{1}_{\{\norm{y}\leq \zeta \} } }\Theta\brac{\frac{dy}{\zeta}}\\
     &=\int_{\R^d}\brac{\exp\brac{i\inner{ u,y}}-1-i\inner{ u,y}\mathds{1}_{\{\norm{y}\leq 1 \} } }\Theta\brac{\frac{dy}{\zeta}}, 
\end{align*}
where the last line is due to rotational symmetry: $ \int_{\R^d}\inner{u,y}\mathds{1}_{\{a\leq \norm{y}\leq b \}} du=0$
for any $a,b\geq 0$. Our calculation of $\phi_{\zeta L_1}(u)$ suggests that the L\'{e}vy measure of $\zeta L_t$ in our Equation \eqref{sde_generalloss} is
\begin{align*}
    \Theta \brac{\frac{dy}{\zeta}}=\frac{\alpha 2^{\alpha-1}\Gamma(\frac{d+\alpha}{2})}{\pi^{d/2}\Gamma\brac{1-\frac{\alpha}{2}}}\zeta^{\alpha+d}\frac{1}{\norm{y}^{\alpha+d}}. 
\end{align*}
Since it has been shown in the proof of Corollary 1.4 in \cite{jianwangbao2022coupling} that $\Theta(dy)$ satisfies their Condition $\operatorname{B_2}$, one can see right away that $\Theta \brac{\frac{dy}{\zeta}}$ also satisfies Condition $\operatorname{B_2}$ (with a different scaling constant that depends on $\zeta$). Thus, analogous to \citep[Corollary 1.4]{jianwangbao2022coupling} which is about their Equation (1.1) with L\'{e}vy noise $L_t$, we will have Lemma~\ref{lemma_wassersteindecaybyBaoWang} below about Equation \eqref{sde_generalloss} with L\'{e}vy noise $\zeta L_t$.

\begin{lemma}(\citep[Lemma 4.1 and Lemma 4.4]{jianwangbao2022coupling})
\label{lemma_lyapunov}
Recall the function $\widehat{F} (\theta,X_n)$ introduced in Section~\ref{sec:main:results}. Let $r_0=\frac{\gamma}{2}$ and $r$  be any constant in the interval
\begin{align*}
    \brac{\frac{\gamma}{2}, \brac{\frac{\gamma^2}{4}+\gamma \sqrt{\beta(\lambda_1-\lambda_2\lambda_4)} -2\beta\lambda_4}^{1/2}}. 
\end{align*}
Set 
\begin{align}\label{defn:V:0}
    V_0(\theta,X_n):=\beta\brac{ \widehat{F} (\theta,X_n)+\lambda_4\norm{\theta}^2+\lambda_5},
\end{align}
and 
\begin{align}
\label{definition_N}
    N(\theta,z,X_n):= 1+V_0(x,X_n)+\frac{r^2}{2}\norm{\theta}^2+\frac{1}{2}\norm{z}^2+r_0\inner{\theta,z}.
\end{align}
Moreover, let
\begin{align}
\label{def_Wlyapunovfunction}
    \mathbb{W}(\theta,z,X_n):=1+N(\theta,z,X_n)^{1/2 }. 
\end{align}
Note that the constant $r$ is well-defined due to \eqref{cond_frombaowang}.

Next, denote $\mathcal{L}$ the infinitesimal generator of \eqref{sde_generalloss} which acts on real-valued functions $h(\theta,v)$ that are twice continuously differentiable in the first and second variables as 
\begin{align*}
    \mathcal{L}h(\theta,v)&=\inner{ v, \nabla_\theta h(\theta,v)}+\inner{-\gamma v-\beta\nabla \widehat{F}(\theta,X_n),\nabla_v h(\theta,v) }\\
    &+\frac{\alpha 2^{\alpha-1}\Gamma(\frac{d+\alpha}{2})}{\pi^{d/2}\Gamma\brac{1-\frac{\alpha}{2}}}\cdot\int_{\R^d} \left(h(\theta,v+z)-h(\theta,v)-\inner{\nabla_vh(\theta,v),z} \mathds{1}_{\{\norm{z}\leq 1\} }\right)\frac{\zeta^{d+\alpha}}{\norm{z}^{d+\alpha}}dz, 
\end{align*}
and similarly, let $\widehat{\mathcal{L}}$ be the infinitesimal generator of \eqref{sde_generalloss_differentdataset}
such that
\begin{align*}
    \widehat{\mathcal{L}}h(\theta,v)&=\inner{ v, \nabla_\theta h(\theta,v)}+\inner{-\gamma v-\beta\nabla \widehat{F}(\theta,\widehat{X}_n),\nabla_{v} h(\theta,v) }\\
    &+\frac{\alpha 2^{\alpha-1}\Gamma(\frac{d+\alpha}{2})}{\pi^{d/2}\Gamma\brac{1-\frac{\alpha}{2}}}\cdot\int_{\R^d} \left(h(\theta,v+z)-h(\theta,v)-\inner{\nabla_{v}h(\theta,v),z} \mathds{1}_{\{\norm{z}\leq 1\} }\right)\frac{\zeta^{d+\alpha}}{\norm{z}^{d+\alpha}}dz. 
\end{align*}

Then under Conditions~\ref{cond_gammaandbeta},~\ref{cond_pseudolipschitz}, and~\ref{cond_allthelambdas}, $\mathbb{W}(\cdot,\cdot,\cdot)$ defined at \eqref{def_Wlyapunovfunction} is a Lyapunov function associated to the processes $(\theta_t,v_t)_{t\geq 0}$ in \eqref{sde_generalloss} and $(\hat{\theta}_t,\hat{v}_t)_{t\geq 0}$ in \eqref{sde_generalloss_differentdataset} and satisfies
    \begin{align*}
        &\mathcal{L}\mathbb{W}(\theta,v,X_n)\leq -c_0 \mathbb{W}(\theta,v,X_n)+C_0,
        \\
        &\widehat{\mathcal{L}}\mathbb{W}(\hat{\theta},\hat{v},\widehat{X}_n)\leq -c_0 \mathbb{W}(\hat{\theta},\hat{v},\widehat{X}_n)+C_0,
    \end{align*}
for some positive constants $c_0$ and $C_0$ that depend on $\alpha,\gamma,\beta,\zeta$, the dimension $d$ plus the parameters $\lambda_i, 1\leq i\leq 5$ in Condition~\ref{cond_allthelambdas}, but do not depend on the datasets $X_n,\widehat{X}_{n}$.
\end{lemma}

\begin{proof}
    Refer to Lemma 4.1 and Lemma 4.4 in \cite{jianwangbao2022coupling}. The specific choice of constants $r_0$ and $r$ are given in the proof of their Lemma 4.4. Moreover, notice that per Equation (4.4) and the discussion right below Theorem 1.3 in \citep{jianwangbao2022coupling}, let $p\in (0,\alpha)$ then
    \begin{align*}
       \mathbb{W}_p(\theta,z,X_n):= 1+\left(N(\theta,z,X_n)\right)^{p/2}
    \end{align*}
    is a Lyapunov function associated to the processes $(\theta_t,v_t)_{t\geq 0}$ in \eqref{sde_generalloss}. Since we are working with $\alpha$-stable L\'{e}vy processes with $\alpha\in (1,2)$, we choose $p=1$  (and thus $ \mathbb{W}(\theta,z,X_n)= \mathbb{W}_1(\theta,z,X_n)$ is our Lyapunov function). 
    
    Further note that there are two typos in the upper bound of $r$ in the proof of their Lemma 4.4 ($\alpha_0$ in there should be $\alpha$ and $r_0^2/2$ in there should be $r_0^2$), per private communication with one of the authors of \cite{jianwangbao2022coupling}. 
\end{proof}

\begin{lemma}(\citep[Corollary 1.4]{jianwangbao2022coupling})
\label{lemma_wassersteindecaybyBaoWang}
 Under Condition~\ref{cond_pseudolipschitz} and Condition~\ref{cond_allthelambdas}, the process $(\theta_t,v_t)_{t\geq 0}$ in \eqref{sde_generalloss} admits a unique invariant measure and correspondingly, the process $(\hat{\theta}_t,\hat{v}_t)_{t\geq 0}$ in \eqref{sde_generalloss_differentdataset} also admits a unique invariant measure. Moreover, for any $t\geq 0$ and initial conditions $(w,y), (w',y')$,  it holds that
 \begin{align*}
&\mathcal{W}_1\brac{\operatorname{Law}\brac{\theta^{w,y}_t,v^{w,y}_t},\operatorname{Law}\brac{\theta^{w',y'}_t,v^{w',y'}_t} }\leq C_*e^{-\lambda_* t}\norm{(w,y)-\brac{w',y'}},
\\
&\mathcal{W}_1\brac{\operatorname{Law}\brac{\hat{\theta}^{w,y}_t,\hat{v}^{w,y}_t},\operatorname{Law}\brac{\hat{\theta}^{w',y'}_t,\hat{v}^{w',y'}_t} }\leq C_*e^{-\lambda_* t}\norm{(w,y)-\brac{w',y'}},
 \end{align*}
 where 
 \begin{align*}
     \lambda_*=\min \left\{ \frac{c_0\epsilon}{1+2\epsilon},\frac{3c_1\brac{1-\frac{1}{\alpha}}\gamma}{8(1+c_1)} \right\};\qquad
     C_*=  c_0^2\sqrt{2d},
 \end{align*}
where $\epsilon$ and $c_1$ are positive constants defined in \citep[Section 3.2]{jianwangbao2022coupling}. The positive constant $c_0$ is from Lemma~\ref{lemma_lyapunov}. 
\end{lemma}

\begin{proof}
This is a consequence of the discussion at the beginning of this section \ref{section_technicallemmas}, \citep[Corollary 1.4]{jianwangbao2022coupling} and the proof of their Theorem 1.1. Specifically, we use inequality (4.1) and the equivalence relation stated at the end of the proof of Theorem 1.1. The explicit form of the constant $\lambda_*$ is also provided in the proof of Theorem 1.1. 
\end{proof}

Based on Lemma~\ref{lemma_wassersteindecaybyBaoWang}, we immediately get the following semigroup gradient estimate.
\begin{lemma}
\label{lemma_semigroupgradiateestimate}
Assume $h:\mathbb{R}^{2d}\rightarrow\mathbb{R}$ is a Lipschitz-continuous function.  Under Conditions~\ref{cond_gammaandbeta},~\ref{cond_pseudolipschitz} and~\ref{cond_allthelambdas}, it holds that
\begin{align}
\label{estimate_gradsemigroup_prop}
  \sup_{w,y\in\R^d} \norm{\nabla_u P_th(w,y)}&\leq \norm{\nabla h}_{\operatorname{op},\infty}\norm{u} C_*e^{-\lambda_* t},
  \\
  \sup_{w,y\in\R^d} \norm{\nabla_u \widehat{P}_th(w,y)}&\leq \norm{\nabla h}_{\operatorname{op},\infty}\norm{u} C_*e^{-\lambda_* t},\nonumber
\end{align}
where the constants $C_*$ and $\lambda_*$ are defined in Lemma~\ref{lemma_wassersteindecaybyBaoWang}. 
\end{lemma}

\begin{proof}
We can compute that
\begin{align*}
    \left|P_th(w,y)-P_th\brac{w',y'}\right|&=\left|\E{h\brac{\theta^{w,y}_t,v^{w,y}_t}-h\brac{\theta^{w',y'}_t,v^{w',y'}_t} }\right|\\
    &\leq \norm{\nabla f}_{\operatorname{op},\infty} \mathcal{W}_1\brac{\delta_{(w,y)}P_t,\delta_{\brac{w',y'}}P_t}\\
    &\leq \norm{\nabla f}_{\operatorname{op},\infty}  C_* e^{-\lambda_* t}\norm{(w,y)-\brac{w',y'}},  
\end{align*}
where the last line is due to Lemma~\ref{lemma_wassersteindecaybyBaoWang}. 
This implies that
\begin{equation*}
\sup_{w,y\in\R^d} \norm{\nabla_u P_th(w,y)}\leq \norm{\nabla h}_{\operatorname{op},\infty}\norm{u} C_*e^{-\lambda_* t}.    
\end{equation*}
Similarly, one can show that
\begin{equation*}
\sup_{w,y\in\R^d} \norm{\nabla_u \widehat{P}_th(w,y)}\leq \norm{\nabla h}_{\operatorname{op},\infty}\norm{u} C_*e^{-\lambda_* t}.
\end{equation*}
This completes the proof.
\end{proof}

The next two lemmas provide moment estimates concerning $\theta^{w,y},v^{w,y}$ and $\hat{\theta}^{w,y},\hat{v}^{w,y}$. 

\begin{lemma}
\label{lemma_uniformmomentbound_continuousdynamics}
Under Conditions~\ref{cond_gammaandbeta},~\ref{cond_pseudolipschitz}, and~\ref{cond_allthelambdas}, we have the uniform estimate (over $t\in [0,\infty)$)
 \begin{align*}
   \E{\norm{\theta^{w,y}_t}}+\E{\norm{v^{w,y}_t}}&\leq   \min \left\{\sqrt{ \frac{r^2-r_0^2}{8}}, \sqrt{\frac{r^2-r_0^2}{8r^2}} \right\}^{-1}\E{\mathbb{W}\left(\theta^{w,y}_s,v^{w,y}_s,X_n\right)}\\
   &
     \leq C_2\brac{1+\norm{w}+\norm{y}+\max_{1\leq i\leq n}\sqrt{\abs{f\brac{w,x_i}}}},
     \\
   \E{\norm{\hat{\theta}^{w,y}_t}}+\E{\norm{\hat{v}^{w,y}_t}}&\leq   \min \left\{\sqrt{ \frac{r^2-r_0^2}{8}}, \sqrt{\frac{r^2-r_0^2}{8r^2}} \right\}^{-1}\E{\mathbb{W}\left(\hat{\theta}^{w,y}_s,\hat{v}^{w,y}_s,\widehat{X}_n\right)}\\
   &
     \leq C_2\brac{1+\norm{w}+\norm{y}+\max_{1\leq i\leq n}\sqrt{\abs{f\brac{w,\hat{x}_i}}}},
 \end{align*}
where the function $\mathbb{W}(\cdot,\cdot,\cdot)$ is defined in Lemma~\ref{lemma_lyapunov} and
\begin{align}
\label{def_C2}
    C_2=\min \left\{\sqrt{ \frac{r^2-r_0^2}{8}}, \sqrt{\frac{r^2-r_0^2}{8r^2}} \right\}^{-1}\cdot\max \left\{\sqrt{\beta},\sqrt{\beta\lambda_4+r^2},1+\sqrt{\beta\lambda_5+1}+\frac{C_0}{c_0} \right\}. 
\end{align}
The constants $r$ and $r_0$ are provided in Lemma~\ref{lemma_lyapunov}. 
\end{lemma}

\begin{proof}
    By Dynkin's formula (see e.g. \cite{oksendal}) and Lemma~\ref{lemma_lyapunov},
    \begin{align*}
    \E{\mathbb{W}\brac{\theta^{w,y}_t,v^{w,y}_t,X_n}}&=\mathbb{W}\brac{w,y,X_n}+\int_0^t \E{\mathcal{L}\mathbb{W}(\theta^{w,y}_s,v^{w,y}_s,X_n)}ds\\
        &\leq \mathbb{W}\brac{w,y,X_n}+\int_0^t \left(-c_0\E{\mathbb{W}(\theta^{w,y}_s,v^{w,y}_s,X_n)}+C_0\right)ds,
    \end{align*}
where $\mathbb{W}(\cdot,\cdot,\cdot)$ is the Lyapunov function that is defined in Lemma~\ref{lemma_lyapunov}.

This implies
\begin{align*}
    e^{c_0 t}\E{\mathbb{W}(\theta^{w,y}_t,v^{w,y}_t,X_n)}- \E{\mathbb{W}(w,y,X_n)}\leq \frac{C_0}{c_0}\brac{e^{c_0 t}-1}\leq \frac{C_0}{c_0}e^{c_0 t},
\end{align*}
and therefore
\begin{align}
\label{middlestep_momentbound}
     \E{\mathbb{W}(\theta^{w,y}_s,v^{w,y}_s,X_n)}&\leq e^{-c_0t}\E{\mathbb{W}(w,y,X_n)}+\frac{C_0}{c_0}.
\end{align}

Next, \citep[(4.5)]{jianwangbao2022coupling} states that regarding the Lyapunov function in Lemma~\ref{lemma_lyapunov}, we have
\begin{align}
\label{lyapunovfunction_upperandlowerestimate}
  &  1+\brac{1+\beta\left(\widehat{F} (\theta,X_n)+\lambda_4 \norm{\theta}^2+\lambda_5\right)+\frac{r^2-r_0^2}{4}\brac{\norm{\theta}^2+\frac{1}{r^2}\norm{v}^2} }^{1/2}\nonumber \\
    &\leq \mathbb{W}(\theta,v,X_n)\nonumber\\
    &\leq 1+\brac{1+\beta\left(\widehat{F} (\theta,X_n)+\lambda_4 \norm{\theta}^2+\lambda_5\right)+r^2\norm{\theta}^2+\norm{v}^2  }^{1/2}. 
\end{align}
Combining with the fact that $\widehat{F} (\theta,X_n)+\lambda_4 \norm{\theta}^2+\lambda_5\geq 0$ per \eqref{cond_lambda45}, it follows that 
\begin{align}
\label{lyapunovfunction_upperandlowerestimate_secondversion}
   &1+ \min  \left\{\sqrt{ \frac{r^2-r_0^2}{4}}, \sqrt{\frac{r^2-r_0^2}{4r^2}} \right\}\brac{\norm{\theta}^2+\norm{v}^2 }^{1/2} \nonumber\\
   &\leq \mathbb{W}(\theta,v,X_n)\nonumber\\
   &\leq 1+\brac{1+\beta\left(\widehat{F} (\theta,X_n)+\lambda_4 \norm{\theta}^2+\lambda_5\right)+r^2\norm{\theta}^2+\norm{v}^2  }^{1/2}, 
\end{align}
which further leads to
\begin{align*}
    &  \min \left\{\sqrt{ \frac{r^2-r_0^2}{8}}, \sqrt{\frac{r^2-r_0^2}{8r^2}} \right\} \brac{\norm{\theta}+\norm{v}}\\
    &\leq \mathbb{W}(\theta,v,X_n)\\
   &\leq  1+\brac{\sqrt{1+\beta\lambda_5}+\sqrt{\beta}\sqrt{\norm{\widehat{F} (\theta,X_n)}}+\sqrt{\beta\lambda_4+r^2}\norm{\theta}+\norm{v}  }.
\end{align*}
The above estimate of $\mathbb{W}(\theta,v,X_n)$ and \eqref{middlestep_momentbound} imply 
\begin{align*}
    &\E{\norm{\theta^{w,y}_t}}+\E{\norm{v^{w,y}_t}}\\&\leq \min \left\{\sqrt{ \frac{r^2-r_0^2}{8}}, \sqrt{\frac{r^2-r_0^2}{8r^2}} \right\}^{-1}\E{\mathbb{W}(\theta^{w,y}_t,v^{w,y}_t,X_n)}\\
    &\leq \min \left\{\sqrt{ \frac{r^2-r_0^2}{8}}, \sqrt{\frac{r^2-r_0^2}{8r^2}} \right\}^{-1}\cdot\max \left\{\sqrt{\beta},\sqrt{\beta\lambda_4+r^2},1+\sqrt{\beta\lambda_5+1}+\frac{C_0}{c_0} \right\}\\
    &\hspace{17em}\cdot\brac{1+\norm{\theta}+\norm{v}+\max_{1\leq i\leq n}\sqrt{\abs{f \brac{\theta,x_i}}}}. 
\end{align*}
Finally, $\E{\norm{\hat{\theta}^{w,y}_t}}+\E{\norm{\hat{v}^{w,y}_t}}$ can be bounded in the same way. 
\end{proof}
%%%%%%%%%%%%%%%%%%%%%%%%%%%%%%%%%%%%%%%%%%%%%%%%%%%%%%%%%%%%%%%%%%%%%%%

Lemma~\ref{lemma_uniformmomentbound_continuousdynamics} allows us to deduce the following estimate that is needed in the proof of Theorem~\ref{theorem_wassersteinbound_appendix}.  

\begin{lemma}
\label{lemma_momentbound_squarerootf}
  It holds for any $x\in\mathcal{X}$ that
\begin{align*}
  &\E{\sqrt{f\brac{\hat{\theta}^{w,y}_{t},x }}}\\
   &\leq \sqrt{\abs{f(0,x)}}+\sqrt{K_2\norm{x} +\norm{\nabla f(0,0) }}\sqrt{C_2}\brac{1+\sqrt{\norm{w}}+\sqrt{\norm{y}}+\max_{1\leq i\leq n}\sqrt[4]{\abs{f\brac{w,\hat{x}_i}}}}\\
   &\qquad\qquad\qquad\qquad+\sqrt{\frac{K_2\norm{x}+1}{2}}C_2\brac{1+\norm{w}+\norm{y}+\max_{1\leq i\leq n}\sqrt{\abs{f\brac{w,\hat{x}_i}}}},
\end{align*}   
where the constant $C_2$ is defined in \eqref{def_C2}. 
\end{lemma}

\begin{proof}
     Condition~\ref{cond_pseudolipschitz} implies 
\begin{align*}
    \norm{\nabla f(\theta,x)}\leq \norm{\nabla f(0,0)} +K_1\norm{\theta}+K_2\norm{x}\brac{\norm{\theta}+1},
\end{align*}
which further implies
\begin{align*}
    \abs{f(\theta,x)}\leq \abs{f(0,x)}+\norm{\nabla f(0,0)}\norm{\theta}+\frac{K_1}{2}\norm{\theta}^2+K_2\norm{x}\brac{\frac{\norm{\theta}^2}{2}+\norm{\theta}}. 
\end{align*}
Hence,
\begin{align*}
   \E{\sqrt{f\left(\hat{\theta}^{w,y}_{t},x \right)}}&\leq \sqrt{\abs{f(0,x)}}+\sqrt{K_2\norm{x}+\norm{\nabla f(0,0)} }\E{\sqrt{\norm{\hat{\theta}^{w,y}_t}}}
   \\
   &\qquad\qquad\qquad+\sqrt{\frac{K_2\norm{x}+1}{2}}\E{\norm{\hat{\theta}^{w,y}_t}}. 
\end{align*}
Combining this with the bound on $\E{\norm{\hat{\theta}^{w,y}_t}}$ in Lemma~\ref{lemma_uniformmomentbound_continuousdynamics}, together with the inequality (which follows from Jensen's inequality):
\begin{equation*}
\E{\sqrt{\norm{\hat{\theta}^{w,y}_t}}}\leq \sqrt{\E{\norm{\hat{\theta}^{w,y}_t}}},
\end{equation*}
we arrive at the desired bound on $\E{\sqrt{f\left(\hat{\theta}^{w,y}_{t},x \right)}}$. This completes the proof.
\end{proof}

%%%%%%%%%%%%%%%%%%%%%%%%%%%%%%%%%%%%%%%%%%%%%%%%%%%%%%%%%%%%%%%%%%%%%%%
In addition, Lemma~\ref{lemma_uniformmomentbound_continuousdynamics} also gives us the following estimates. 
\begin{lemma}
\label{lemma_momentestimatewithrespecttoinitialcondition}
Under Conditions~\ref{cond_gammaandbeta},~\ref{cond_pseudolipschitz}, and~\ref{cond_allthelambdas}, we have
    \begin{align*}
           &\E{\norm{\theta^{w,y}_t-w}}+ \E{\norm{v^{w,y}_t-y}}
           \\
           &\leq \brac{t\vee t^{1/\alpha}}\Bigg(C_2\brac{1+\gamma+\norm{\nabla f(0,0)}+K_1+\frac{K_2}{n}\sum_{i=1}^n \norm{x_i} }\\
           &\qquad\cdot \brac{1+\norm{w}+\norm{y}+\max_{1\leq i\leq n}\sqrt{\abs{f\brac{w,x_i}}}}+\frac{K_2}{n}\sum_{i=1}^n \norm{x_i}+\E{\norm{L_1}}\Bigg), 
    \end{align*}
and
    \begin{align*}
           &\E{\norm{\hat{\theta}^{w,y}_t-w}}+ \E{\norm{\hat{v}^{w,y}_t-y}}
           \\
           &\leq \brac{t\vee t^{1/\alpha}}\Bigg(C_2\brac{1+\gamma+\norm{\nabla f(0,0)}+K_1+\frac{K_2}{n}\sum_{i=1}^n \norm{\hat{x}_i} }\\
           &\qquad\cdot \brac{1+\norm{w}+\norm{y}+\max_{1\leq i\leq n}\sqrt{\abs{f\brac{w,\hat{x}_i}}}}+\frac{K_2}{n}\sum_{i=1}^n \norm{\hat{x}_i}+\E{\norm{L_1}}\Bigg).
    \end{align*}
\end{lemma}

\begin{proof}
It follows from
    \begin{align*}
        \theta^{w,y}_t-w=\int_0^t v^{w,y}_s ds,
    \end{align*}
and Lemma~\ref{lemma_uniformmomentbound_continuousdynamics} that
\begin{align*}
    \E{\norm{\theta^{w,y}_t-w}}\leq  t\cdot C_2\brac{1+\norm{w}+\norm{y}+\max_{1\leq i\leq n}\sqrt{\abs{f\brac{w,x_i}}}}+t^{1/\alpha }\cdot\E{\norm{L_1}}. 
\end{align*}

Next, we have
\begin{align*}
    v^{w,y}_t-y=\int_0^t \left(-\gamma v^{w,y}_s -\nabla  \widehat{F} (\theta^{w,y}_s,X_n)\right)ds+L_t. 
\end{align*}

Condition~\ref{cond_pseudolipschitz} implies 
\begin{equation}
\norm{\nabla\widehat{F}(\theta,X_n)}\leq\norm{\nabla f(0,0)}+K_{1}\norm{\theta}+\frac{K_2}{n}\sum_{i=1}^n \norm{x_i}(\norm{\theta}+1). 
\end{equation}
We combine this with Lemma~\ref{lemma_uniformmomentbound_continuousdynamics} to obtain
\begin{align*}
     &\E{\norm{v^{w,y}_t-y}}\\
     &\leq \int_0^t \brac{\gamma\E{\norm{v^{w,y}_s}}+\norm{\nabla f(0,0)}+\left(K_1+\frac{K_2}{n}\sum_{i=1}^n \norm{x_i}\right)\E{\norm{\theta^{w,y}_s}}+\frac{K_2}{n}\sum_{i=1}^n \norm{x_i}}ds\\
     &\qquad\qquad\qquad\qquad\qquad\qquad+t^{1/\alpha} \E{\norm{L_1}}\\
     &\leq t\brac{\gamma+\brac{\norm{\nabla f(0,0)}+K_1+\frac{K_2}{n}\sum_{i=1}^n \norm{x_i} }}
     \\
     &\qquad\qquad\quad\cdot C_2\brac{1+\norm{w}+\norm{y}+\max_{1\leq i\leq n}\sqrt{\abs{f\brac{w,x_i}}}}+t\frac{K_2}{n}\sum_{i=1}^n \norm{x_i}+t^{1/\alpha} \E{\norm{L_1}}. 
\end{align*}
Similarly, one can show that
\begin{align*}
&\E{\norm{\hat{\theta}^{w,y}_t-w}}+ \E{\norm{\hat{v}^{w,y}_t-y}}
\\
&\leq \brac{t\vee t^{1/\alpha}}\Bigg(C_2\brac{1+\gamma+\norm{\nabla f(0,0)}+K_1+\frac{K_2}{n}\sum_{i=1}^n \norm{\hat{x}_i} }\\
&\qquad\cdot \brac{1+\norm{w}+\norm{y}+\max_{1\leq i\leq n}\sqrt{\abs{f\brac{w,\hat{x}_i}}}}+\frac{K_2}{n}\sum_{i=1}^n \norm{\hat{x}_i}+\E{\norm{L_1}}\Bigg).
\end{align*}
This completes the proof.
\end{proof}

Based on the previous results, we are able to perform the following one-step comparison of the semigroups associated with $\brac{\theta_t,v_t}$ at \eqref{sde_generalloss} and $\brac{\hat{\theta}_t,\hat{v}_t}$ at \eqref{sde_generalloss_differentdataset}. 

\begin{lemma}
\label{lemma_onestepestimate}
    Assume Conditions~\ref{cond_gammaandbeta},~\ref{cond_pseudolipschitz}, and~\ref{cond_allthelambdas}. Then for $0<\eta<1$ and any Lipschitz function $h:\mathbb{R}^{2d}\rightarrow\mathbb{R}$, it holds that 
\begin{align*}
  &\abs{  P_\eta h(w,y)-\widehat{P}_\eta(w,y)}
  \\
&\leq \norm{\nabla h}_{\operatorname{op},\infty}  \Bigg((1+\gamma)\Bigg(C_2\brac{1+\gamma+\norm{\nabla f(0,0)}+K_1+\frac{K_2}{n}\sum_{i=1}^n \norm{x_i} }\\
           &\qquad\qquad\qquad\cdot \brac{1+\norm{w}+\norm{y}+\max_{1\leq i\leq n}\sqrt{\abs{f\brac{w,x_i}}}}+\frac{K_2}{n}\sum_{i=1}^n \norm{x_i}+\E{\norm{L_1}}\Bigg)\eta^{1+1/{\alpha}}
    \\
    &\qquad+ (1+\gamma)\Bigg(C_2\brac{1+\gamma+\norm{\nabla f(0,0)}+K_1+\frac{K_2}{n}\sum_{i=1}^n \norm{\hat{x}_i} }\\
           &\qquad\qquad\qquad\cdot \brac{1+\norm{w}+\norm{y}+\max_{1\leq i\leq n}\sqrt{\abs{f\brac{w,\hat{x}_i}}}}+\frac{K_2}{n}\sum_{i=1}^n \norm{\hat{x}_i}+\E{\norm{L_1}}\Bigg)\eta^{1+1/{\alpha}}
           \\
           &\quad+K_2 \rho(X_n,\widehat{X}_n)\\
           &\hspace{4em}\cdot\left(C_{2}\brac{2+2\norm{w}+2\norm{y}+\max_{1\leq i\leq n}\sqrt{\abs{f\brac{w,x_i}}}+\max_{1\leq i\leq n}\sqrt{\abs{f\brac{w,\hat{x}_i}}}}+1\right)\eta\Bigg). 
\end{align*}
\end{lemma}

\begin{proof}
We can compute that
    \begin{align*}
        &\abs{P_\eta h(w,y)-\widehat{P}_\eta(w,y)}\\
        &=\abs{\E{h\left(\theta^{w,y}_\eta,v^{w,y}_\eta\right)-h\left(\hat{\theta}^{w,y}_\eta,\hat{v}^{w,y}_\eta\right)}}\\
        &\leq \norm{\nabla h}_{\operatorname{op},\infty}\brac{\E{\norm{\theta^{w,y}_\eta-\hat{\theta}^{w,y}_\eta}}+\E{\norm{v^{w,y}_\eta-\hat{v}^{w,y}_\eta}} }\\
        &\leq \norm{\nabla h}_{\operatorname{op},\infty} \bigg((1+\gamma)\int_0^\eta \left(\E{\norm{v^{w,y}_s-\hat{v}^{w,y}_s}} + K_1\E{\norm{\theta^{w,y}_s -\hat{\theta}^{w,y}_s}}\right)ds\\
        &\qquad\qquad\qquad\qquad+ \int_0^\eta\rho(X_n,\widehat{X}_n)K_2\brac{\E{\norm{\theta^{w,y}_s} } +\E{\norm{\hat{\theta}^{w,y}_s}}+1 } ds\bigg)\\
        &=:\norm{\nabla h}_{\operatorname{op},\infty} \brac{\mathcal{A}_1+\mathcal{A}_2}. 
    \end{align*}
Lemma~\ref{lemma_momentestimatewithrespecttoinitialcondition} implies 
\begin{align*}
    \mathcal{A}_1&\leq (1+\gamma)\bigg(\int_0^\eta \left(\E{\norm{v^{w,y}_s-y}} + K_1\E{\norm{\theta^{w,y}_s -w}}\right)ds\\
    &\qquad\qquad\qquad\qquad\qquad\qquad\qquad+\int_0^\eta \left(\E{\norm{\hat{v}^{w,y}_s-y}}+ K_1\E{\norm{\hat{\theta}^{w,y}_s -w}}\right)ds\bigg)\\
    &\leq (1+\gamma)\Bigg(C_2\brac{1+\gamma+\norm{\nabla f(0,0)}+K_1+\frac{K_2}{n}\sum_{i=1}^n \norm{x_i} }\\
           &\qquad\qquad\cdot \brac{1+\norm{w}+\norm{y}+\max_{1\leq i\leq n}\sqrt{\abs{f\brac{w,x_i}}}}+\frac{K_2}{n}\sum_{i=1}^n \norm{x_i}+\E{\norm{L_1}}\Bigg)\eta^{1+1/{\alpha}}
    \\
    &\qquad+ (1+\gamma)\Bigg(C_2\brac{1+\gamma+\norm{\nabla f(0,0)}+K_1+\frac{K_2}{n}\sum_{i=1}^n \norm{\hat{x}_i} }\\
           &\qquad\qquad\cdot \brac{1+\norm{w}+\norm{y}+\max_{1\leq i\leq n}\sqrt{\abs{f\brac{w,\hat{x}_i}}}}+\frac{K_2}{n}\sum_{i=1}^n \norm{\hat{x}_i}+\E{\norm{L_1}}\Bigg)\eta^{1+1/{\alpha}}.     
\end{align*}
Meanwhile, Lemma~\ref{lemma_uniformmomentbound_continuousdynamics} states that
\begin{align*}
    \mathcal{A}_2& \leq K_2 \rho(X_n,\widehat{X}_n)\\
    &\qquad\cdot\left(C_{2}\brac{2+2\norm{w}+2\norm{y}+\max_{1\leq i\leq n}\sqrt{\abs{f\brac{w,x_i}}}+\max_{1\leq i\leq n}\sqrt{\abs{f\brac{w,\hat{x}_i}}}}+1\right)\eta. 
\end{align*}
Combining the bounds on $\mathcal{A}_1$ and $\mathcal{A}_2$ yields the desired result. This completes the proof.
\end{proof}

%%%%%%%%%%%%%%%%%%%%%%%%%%%%%%%%%%%%%%%%%%%%%%%%%%%%%%%%%%%%%%%%

\subsection{Proof of Corollary~\ref{coro_generalizationbound}}
\label{section_proofcorollarygeneralization_continuous}
Under the assumption that $\sup_{x,y\in\mathcal{X}}\|x-y\|\leq D$
for some $D<\infty$ and $X_{n}$ and $\widehat{X}_{n}$ differ by at most one data point, 
we get
\begin{align}
    \rho(X_n,\widehat{X}_n)= \frac{1}{n}\sum_{i=1}^n \norm{x_i-\hat{x}_i}
    \leq\frac{D}{n}.
\end{align}
Since for any $w,y\in\mathbb{R}^{d}$, $\brac{\theta^{w,y}_t,v^{w,y}_t}$ converges
to the unique invariant measure as $t\rightarrow\infty$, we can write
$\brac{\theta^{w,y}_\infty,v^{w,y}_\infty}=\brac{\theta_\infty,v_\infty}$, 
omitting the superscript on $w,y$.
Then from Theorem~\ref{theorem_wassersteinbound_mainsection} and \eqref{eqn:wass_stab}, we have
\begin{align}
\label{generalizationbound_unsimplified}
    &\left|\mathbb{E}_{\theta_\infty,X_n}~\left[ \hat{R}(\theta_\infty,X_n) \right] -  R(\theta_\infty)  \right| \nonumber  \\ 
&\leq \frac{D}{n} \frac{C_*K_2 L}{\lambda_*}\Bigg\{2C_2+2C_2^2\brac{1+\max_{1\leq i\leq n}\sqrt{\abs{f\brac{0,\hat{x}_i}}}}\nonumber\\
    &\qquad+C_2\Bigg(\sqrt{\max_{1\leq i\leq n}\abs{f(0,x_i)}}+\sqrt{\frac{K_2\max_{1\leq i\leq n}\norm{x_i}+1}{2}}C_2\brac{1+\max_{1\leq i\leq n}\sqrt{\abs{f\brac{0,\hat{x}_i}}}}\nonumber\\
   &\qquad\qquad+\sqrt{K_2\max_{1\leq i\leq n}\norm{x_i} +\norm{\nabla f(0,0) }}\sqrt{C_2}\brac{1+\max_{1\leq i\leq n}\sqrt[4]{\abs{f\brac{0,\hat{x}_i}}}}\Bigg)\nonumber\\
    &\qquad\qquad\quad+C_2 \Bigg(\sqrt{\max_{1\leq i\leq n}\abs{f(0,\hat{x}_i)}}+\sqrt{\frac{K_2\max_{1\leq i\leq n}\norm{\hat{x}_i}+1}{2}}C_2\brac{1+\max_{1\leq i\leq n}\sqrt{\abs{f\brac{0,\hat{x}_i}}}}\nonumber\\
     &\qquad\qquad\qquad\qquad+\sqrt{K_2\max_{1\leq i\leq n}\norm{\hat{x}_i} +\norm{\nabla f(0,0) }}\sqrt{C_2}\brac{1+\max_{1\leq i\leq n}\sqrt[4]{\abs{f\brac{0,\hat{x}_i}}}} \Bigg)+1\Bigg\}\nonumber\\
     &=\frac{D}{n} \frac{C_*K_2 L}{\lambda_*}\Bigg\{\mathcal{A}_1+\mathcal{A}_2+\mathcal{A}_3\Bigg\}, 
\end{align}
where 
\begin{align}
\label{def_termA}
    \mathcal{A}_1
    &:=C_2\Bigg(\sqrt{\max_{1\leq i\leq n}\abs{f(0,x_i)}}+\sqrt{\frac{K_2\max_{1\leq i\leq n}\norm{x_i}+1}{2}}C_2\brac{1+\max_{1\leq i\leq n}\sqrt{\abs{f\brac{0,\hat{x}_i}}}}\nonumber\\
   &\qquad+\sqrt{K_2\max_{1\leq i\leq n}\norm{x_i} +\norm{\nabla f(0,0) }}\sqrt{C_2}\brac{1+\max_{1\leq i\leq n}\sqrt[4]{\abs{f\brac{0,\hat{x}_i}}}}\Bigg);\nonumber\\
   \mathcal{A}_2&:=C_2 \Bigg(\sqrt{\max_{1\leq i\leq n}\abs{f(0,\hat{x}_i)}}+\sqrt{\frac{K_2\max_{1\leq i\leq n}\norm{\hat{x}_i}+1}{2}}C_2\brac{1+\max_{1\leq i\leq n}\sqrt{\abs{f\brac{0,\hat{x}_i}}}}\nonumber\\
     &\qquad+\sqrt{K_2\max_{1\leq i\leq n}\norm{\hat{x}_i} +\norm{\nabla f(0,0) }}\sqrt{C_2}\brac{1+\max_{1\leq i\leq n}\sqrt[4]{\abs{f\brac{0,\hat{x}_i}}}} \Bigg);\nonumber\\
    \mathcal{A}_3&:=1+ 2C_2+2C_2^2\brac{1+\max_{1\leq i\leq n}\sqrt{\abs{f\brac{0,\hat{x}_i}}}}.
\end{align}

In the next step, we aim to further bound $\mathcal{A}_i$ for $i\in \{1,2,3\}$. Condition~\ref{cond_pseudolipschitz} implies that 
\begin{align*}
    \norm{\nabla f(0,x)}\leq \norm{\nabla f(0,0)}+K_2\norm{x},
\end{align*}
which further implies 
\begin{align*}
    \abs{f(0,x)}\leq \abs{f(0,0)}+\norm{\nabla f(0,0)}\norm{x}+\frac{K_2}{2}\norm{x}^2. 
\end{align*}
Since $\sup_{x,y\in\mathcal{X}}\|x-y\|\leq D$ and $0\in\mathcal{X}$, we have $\sup_{x\in\mathcal{X}}\|x\|\leq D$, and it follows that 
\begin{align}
\label{estimatesquarerootf(0,xi)}
 \max_{1\leq i\leq n}\sqrt{\abs{f\brac{0,{x}_i}}}\vee  \max_{1\leq i\leq n}\sqrt{\abs{f\brac{0,\hat{x}_i}}}&\leq \sqrt{\abs{f(0,0)}}+\sqrt{\norm{\nabla f(0,0)}}\sqrt{D}+\sqrt{\frac{K_2}{2}}D;\nonumber\\
   \max_{1\leq i\leq n}\sqrt[4]{\abs{f\brac{0,\hat{x}_i}}}&\leq \sqrt[4]{\abs{f(0,0)}}+\sqrt[4]{\norm{\nabla f(0,0)}}\sqrt[4]{D}+\sqrt[4]{\frac{K_2}{2}}\sqrt{D}. 
\end{align}
Furthermore, $\sup_{x\in\mathcal{X}}\|x\|\leq D$ also implies
\begin{align}
\label{estimate_maxxi}
    \max_{1\leq i\leq n} \norm{x_i}\vee\max_{1\leq i\leq n} \norm{\hat{x}_i}\leq D. 
\end{align}
Now let us apply the estimates \eqref{estimatesquarerootf(0,xi)} and \eqref{estimate_maxxi} to $\mathcal{A}_1$ defined in \eqref{def_termA}, starting with 
\begin{align*}
   & C_2^2\sqrt{\frac{K_2\max_{1\leq i\leq n}\norm{x_i}+1}{2}}\brac{1+\max_{1\leq i\leq n}\sqrt{\abs{f\brac{0,\hat{x}_i}}}}\\
   &\leq C_2^2\brac{\sqrt{\frac{K_2\max_{1\leq i\leq n}\norm{x_i}}{2}}+\sqrt{\frac{1}{2}}}\brac{1+\max_{1\leq i\leq n}\sqrt{\abs{f\brac{0,\hat{x}_i}}}}\\ 
   &\leq \frac{C_2^2K_2}{2}D^{3/2}+ C^2_2\brac{\sqrt{\frac{K_2\norm{\nabla f(0,0)}}{2}} +\frac{\sqrt{K_2}}{2} }D \\
   &\qquad+C_2^2\brac{(1+\sqrt{\abs{f(0,0)}})\sqrt{\frac{K_2}{2}}+\frac{1}{\sqrt{2}}\sqrt{\norm{\nabla f(0,0)}}}D^{1/2} +\frac{C_2^2}{\sqrt{2}}\left(1+\sqrt{\abs{f(0,0)}}\right),
\end{align*}
and 
\begin{align*}
   & \sqrt{K_2\max_{1\leq i\leq n}\norm{x_i} +\norm{\nabla f(0,0) }}\sqrt{C_2}\brac{1+\max_{1\leq i\leq n}\sqrt[4]{\abs{f\brac{0,\hat{x}_i}}}}\\
   &\leq \brac{\sqrt{K_2\max_{1\leq i\leq n}\norm{x_i}} +\sqrt{\norm{\nabla f(0,0) }}}\sqrt{C_2}\brac{1+\max_{1\leq i\leq n}\sqrt[4]{\abs{f\brac{0,\hat{x}_i}}}}\\
   &\leq \frac{C_2^{3/2}K^{3/4}_2}{\sqrt{2}}D+C_2^{3/2}\sqrt[4]{\norm{\nabla f(0,0)}}\sqrt{K_2}D^{3/4}\\
   &+C_2^{3/2}\brac{(1+\sqrt[4]{\abs{f(0,0)}})\sqrt{K_2}+\sqrt{\norm{\nabla f(0,0)}} \sqrt[4]{\frac{K_2}{2}} } D^{1/2}\\
   &\qquad+C_2^{3/2}\norm{\nabla f(0,0)}^{3/4}D^{1/4}+C_2^{3/2}\sqrt{\norm{\nabla f(0,0)}}\brac{1+\sqrt[4]{\abs{f(0,0)}} },
\end{align*}
which leads to 
\begin{align*}
    \mathcal{A}_1
    &\leq \frac{C_2^2K_2}{2}D^{3/2}+\left(C^2_2\brac{\sqrt{\frac{K_2\norm{\nabla f(0,0)}}{2}} +\frac{\sqrt{K_2}}{2} }+  \frac{C_2^{3/2}K^{3/4}_2}{\sqrt{2}}\right)D\\
    &\quad+C_2^{3/2}\sqrt[4]{\norm{\nabla f(0,0)}}\sqrt{K_2}D^{3/4}+\Bigg(C_2^2\brac{(1+\sqrt{\abs{f(0,0)}})\sqrt{\frac{K_2}{2}}+\frac{1}{\sqrt{2}}\sqrt{\norm{\nabla f(0,0)}}} \\
    &\quad\quad+C_2^{3/2}\brac{(1+\sqrt[4]{\abs{f(0,0)}})\sqrt{K_2}+\sqrt{\norm{\nabla f(0,0)}} \sqrt[4]{\frac{K_2}{2}} } \Bigg)D^{1/2}\\
    &\quad\quad\quad+C_2^{3/2}\norm{\nabla f(0,0)}^{3/4}D^{1/4}\\
    &\quad\quad\quad\quad+\brac{\frac{C_2^2}{\sqrt{2}}(1+\sqrt{\abs{f(0,0)}})+C_2^{3/2}\sqrt{\norm{\nabla f(0,0)}}\brac{1+\sqrt[4]{\abs{f(0,0)}} }  }. 
\end{align*}
It is clear that \eqref{estimatesquarerootf(0,xi)} and \eqref{estimate_maxxi} also lead to the same bound on the term $\mathcal{A}_2$ which is defined in \eqref{def_termA}. Regarding the term $\mathcal{A}_3$, 
\begin{align*}
    \mathcal{A}_3&=1+ 2C_2+2C_2^2\brac{1+\max_{1\leq i\leq n}\sqrt{\abs{f\brac{0,\hat{x}_i}}}}\\
    &\leq C_2^2\sqrt{2K_2}D+2C_2^2\sqrt{\norm{\nabla f(0,0)}}D^{1/2}+ \brac{2C_2^2\brac{1+\sqrt{\abs{f(0,0)}}}+2C_2+1}. 
\end{align*}
Thus, we can deduce from \eqref{generalizationbound_unsimplified} the desired generalization error bound which is 
\begin{align*}
 &\left|\mathbb{E}_{\theta_\infty,X_n}~\left[ \hat{R}(\theta_\infty,X_n) \right] -  R(\theta_\infty)  \right| \\
    &\leq \frac{1}{n}  \Big(d_1 D+d_2 D^{5/4}+d_3 D^{3/2}+d_4 D^{7/4}+d_5D^2+d_6 D^{5/2}\Big),
\end{align*}
where the constants $d_i,1\leq i\leq 6$ are given by:
\begin{align}
\label{def_d1throughd6}
&d_1:=\frac{C_*K_2 L}{\lambda_*}\Bigg(\sqrt{2}{C_2^2}(1+\sqrt{\abs{f(0,0)}})+2C_2^{3/2}\sqrt{\norm{\nabla f(0,0)}}\brac{1+\sqrt[4]{\abs{f(0,0)}} } \nonumber\\
    &\hspace{20em}+2C_2^2\brac{1+\sqrt{\abs{f(0,0)}}}+2C_2+1 \Bigg)\nonumber,
  \\
      &d_2:=\frac{C_*K_2 L}{\lambda_*}\brac{2C_2^{3/2}\norm{\nabla f(0,0)}^{3/4}},\nonumber
      \\
   &d_3:=\frac{C_*K_2 L}{\lambda_*}\Bigg(\sqrt{2}C_2^2\cdot(1+\sqrt{\abs{f(0,0)}})\sqrt{K_2} \nonumber
   \\&\qquad+2C_2^{3/2}\brac{(1+\sqrt[4]{\abs{f(0,0)}})\sqrt{K_2}+\sqrt{\norm{\nabla f(0,0)}} \sqrt[4]{\frac{K_2}{2}} } +(\sqrt{2}+2)C_2^2\sqrt{\norm{\nabla f(0,0)}}\Bigg),\nonumber\\
    &d_4:=\frac{C_*K_2 L}{\lambda_*}\brac{2C_2^{3/2}\sqrt[4]{\norm{\nabla f(0,0)}}\sqrt{K_2}},\nonumber\\
    &d_5:=\frac{C_*K_2 L}{\lambda_*}\Bigg(2C^2_2\brac{\sqrt{\frac{K_2\norm{\nabla f(0,0)}}{2}} +\frac{\sqrt{K_2}}{2} }+  \sqrt{2}C_2^{3/2}K^{3/4}_2+ C_2^2\sqrt{2K_2}\Bigg),\nonumber\\
    &d_6:=\frac{C_*K_2^{2} L}{\lambda_*}C_2^2,
\end{align}
where the constant $K_2$ is defined in Condition~\ref{cond_pseudolipschitz}; while $C_2$ and $\lambda_*,C_*$ are respectively defined in \eqref{def_C2mainsection} and \eqref{def_C*lambda*mainsection}.

%%%%%%%%%%%%%%%%%%%%%%%%%%%%%%%%%%%%%%%%%%%%%%%%%%%%%%%%%%%%%%%%
%%%%%%%%%%%%%%%%%%%%%%%%%%%%%%%%%%%%%%%%%%%%%%%%%%%%%%%%%%%%%%%%

%%%%%%%%%%%%%%%%%%%%%%%%%%%%%%%%%%%%%%%%%%%%%%%%%%%%%%%%%%%%%%%%
%%%%%%%%%%%%%%%%%%%%%%%%%%%%%%%%%%%%%%%%%%%%%%%%%%%%%%%%%%%%%%%%

\section{Proofs of the Results in Section~\ref{section_mainsquareloss}}\label{appendix_squareloss}

\subsection{Proof of Theorem~\ref{theorem_squareloss}}

First, we establish erdogicity. Per \citep[Lemma 3]{squareloss2023algorithmic}, the positive-definiteness of $X^{\top}X$ and $\hat{X}^{\top}\hat{X}$ implies that $Z_t$ and $\hat{Z}_t$ have unique stationary distributions. Next, regarding ergodicity of $Y_t$ (and similarly $\widehat{Y}_t$), we rely on Lemma~\ref{lemma_wassersteindecaybyBaoWang} which requires that Conditions~\ref{cond_pseudolipschitz} and \ref{cond_allthelambdas} are satisfied for the quadratic loss function $\widehat{F}(\theta,X_n)=\frac{1}{2n}\sum_{i=1}^n \brac{\theta^{\top}x_i}^2$. The pseudo-Lipschitz Condition~\ref{cond_pseudolipschitz} is clearly satisfied by $\widehat{F}(\theta,X_n)$. For Condition~\ref{cond_allthelambdas}, we denote the eigenvalues of $\frac{1}{n}X^{\top}X$ by $\kappa_{i}$, $1\leq i\leq d$, then set $\lambda_2=\lambda_3=\lambda_4=\lambda_5=0$ and $\lambda_1$ to be any positive constant such that 
\begin{align*}
    \lambda_1< \min \left\{\frac{\gamma^2}{16};\kappa_i, 1\leq i\leq d \right\}. 
\end{align*}
Then \eqref{cond_frombaowang} and \eqref{cond_lambda45} are satisfied since $\frac{\gamma^2}{4}+\sqrt{\lambda_1}\gamma>0$. Meanwhile, let us assume that $\frac{1}{n}X^{\top}X$ has the decomposition $VDV^{\top}$
where $D$ is diagonal consisting of eigenvalues $\kappa_{i}$, $1\leq i\leq d$ of $\frac{1}{n}X^{\top}X$ 
and $V$ is an orthogonal matrix. Then the fact that $\lambda_1< \min\{\kappa_i:1\leq i\leq d \}$ ensures that the $d\times d$ matrix 
\begin{align*}
    \frac{1}{n}X^{\top}X-\lambda_1 I=V(D-\lambda_1 I)V^{\top}
\end{align*}
is positive definite and $\inner{\theta, \brac{\frac{1}{n}X^{\top}X-\lambda_1 I}\theta}>0$, for every $\theta\in \R^d$. Thus, \eqref{cond_lambda123} is satisfied. It follows that Conditions  \ref{cond_pseudolipschitz} and \ref{cond_allthelambdas} are satisfied, so that the processes $Y_t$  and $\hat{Y}_t$ admit unique stationary distributions per Lemma~\ref{lemma_wassersteindecaybyBaoWang}. 

Next, let us derive the estimate \eqref{wassersteinbound_squareloss_with_momentum}. It is possible to solve \eqref{sde_squaredloss} explicitly to obtain: 
\begin{align*}
    Y_t=e^{-At}Y_0+\int_0^t e^{-A(t-s)}\Sigma dL_s,\qquad  \hat{Y}_t=e^{-\hat{A}t}Y_0+\int_0^t e^{-\hat{A}(t-s)} \Sigma dL_s.
\end{align*}
Then for any $p\in [1,\alpha)$, we have the inequality
\begin{align}
\label{estimate_pwasserstein}
    &\mathcal{W}_p\brac{\operatorname{Law}(Y_t),\operatorname{Law}(\hat{Y}_t)}\nonumber
    \\
    &\leq \E{\norm{Y_t-\hat{Y}_t}^p}^{1/p}\nonumber\\
    &\leq \E{\norm{e^{-At}Y_0-e^{-\hat{A}t}Y_0 }^p}^{1/p}+\E{\norm{\int_0^t \left(e^{-A(t-s)}-e^{-\hat{A}(t-s)}\right) \Sigma dL_s }^p}^{1/p}. 
\end{align}

Similar to \citep[Proof of Lemma 13]{squareloss2023algorithmic}, we have
\begin{align}
\label{estimate_fromleastsquarepaper}
    \norm{e^{-At}Y_0-e^{-\hat{A}t}Y_0}\leq \frac{ 2\abs{\sigma_1+\sigma_2}\norm{Y_0}}{n}te^{-t\sigma_{\min}}. 
\end{align}
Therefore, what remains is to bound the second term on the right hand side of \eqref{estimate_pwasserstein}. This term can be decomposed into a sum of two Poisson stochastic integrals associated with respectively small jumps and big jumps. Specifically, let $N$ be the Poisson random measure on $\R^d\times [0,\infty)$ with intensity measure $\norm{z}^{-d-\alpha}dzds$ and $\widetilde{N}$ be the compensated Poisson measure, that is $\widetilde{N}(dz,ds):=N(dz,ds)-\norm{z}^{-d-\alpha}dzds$. Then
\begin{align}
\label{equation_sumofsmallandbigjumps}
   \int_0^t \left(e^{-A(t-s)}-e^{-\hat{A}(t-s)}\right) \Sigma dL_s\nonumber
   &=\int_0^t \int_{\norm{z}<1}\brac{e^{-A(t-s)}-e^{-\hat{A}(t-s)}}\Sigma z \widetilde{N}(dz,ds)\\
   &\qquad+\int_0^t \int_{\norm{z}\geq 1}\brac{e^{-A(t-s)}-e^{-\hat{A}(t-s)} }\Sigma z {N}(dz,ds). 
\end{align}

By \eqref{estimate_fromleastsquarepaper} and Kunita's inequality (see \citep[Lemma D.1]{dangzhu2024euler} where explicit constants are obtained),  
\begin{align*}
  &\E{\norm{\int_0^t \int_{\norm{z}<1}\left(e^{-A(t-s)}-e^{-\hat{A}(t-s)}\right)\Sigma z \widetilde{N}(dz,ds)}^p}^{1/p}\\
  &\leq \E{\norm{\int_0^t \int_{\norm{z}<1}\left(e^{-A(t-s)}-e^{-\hat{A}(t-s)}\right)\Sigma z \widetilde{N}(dz,ds)}^2}^{1/2}\\
  &\leq 2\zeta\brac{\int_0^t\int_{\norm{z}<1}\brac{\frac{2\abs{\sigma_1+\sigma_2}\norm{Y_0}}{n}(t-s)e^{-(t-s)\sigma_{\min}} \norm{z}}^2 \norm{z}^{-d-\alpha}dzds  }^{1/2}\\
  &\leq \frac{8\zeta\abs{\sigma_1+\sigma_2}\norm{Y_0}}{n} \brac{\int_0^t (t-s)^2 e^{-2(t-s)\sigma_{\min}}ds}^{1/2} \brac{\int_{\norm{z}<1} \norm{z}^{2-d-\alpha}dz }^{1/2}. 
\end{align*}
We have
\begin{align*}
    \int_0^t (t-s)^2 e^{-2(t-s)\sigma_{\min}}ds= \frac{1}{4 \sigma_{\min}^3}e^{-2\sigma_{\min}t }\brac{e^{2\sigma_{\min}t}-2\sigma_{\min}t\brac{\sigma_{\min} t+1}-1 }.
\end{align*}

Moreover, denote $V_d=\frac{\pi^{d/2}}{\Gamma(\frac{d}{2}+1)}$ the volume of a $d$-dimensional unit ball. Then by a change of variable, 
we have
\begin{align*}
    \int_{\norm{z}<1} \norm{z}^{2-d-\alpha}dz =V_d \int_0^1 r^{2-d-\alpha}r^{d-1}dr=\frac{V_d}{2-\alpha}. 
\end{align*}
Hence,
\begin{align}
\label{estimate_firstpoissonintegral}
    &\E{\norm{\int_0^t \int_{\norm{z}<1}\left(e^{-A(t-s)}-e^{-\hat{A}(t-s)}\right)\Sigma z \widetilde{N}(dz,ds)}^p}^{1/p}\nonumber\\
    &\leq \frac{8\zeta\abs{\sigma_1+\sigma_2}\norm{Y_0}V_d^{1/2} }{2\sigma^{3/2}_{\min}(2-\alpha)^{1/2}n}e^{-\sigma_{\min}t }\brac{e^{2\sigma_{\min}t}-2\sigma_{\min}t\brac{\sigma_{\min} t+1}-1 }^{1/2}. 
 \end{align}

To bound the second Poisson integral on the right hand side of \eqref{equation_sumofsmallandbigjumps}, we apply \eqref{estimate_fromleastsquarepaper} and \citep[Proposition 2.2, Part 3]{zhu2019maximal}. The latter states that there is a constant $C(p)$ such that 
\begin{align*}
  & \E{ \norm{\int_0^t \int_{\norm{z}\geq 1}\brac{e^{-A(t-s)}-e^{-\hat{A}(t-s)}}\Sigma z {N}(dz,ds)}^p}^{1/p}\\
   &\leq C(p)\zeta\brac{\int_0^t\int_{\norm{z}\geq 1} \brac{\frac{\abs{\sigma_1+\sigma_2}\norm{Y_0}}{n}(t-s)e^{-(t-s)\sigma_{\min}}\norm{\Sigma} \norm{z} }^p \norm{z}^{-d-\alpha} dzds}^{1/p}\\
   &\leq C(p)\frac{\zeta\abs{\sigma_1+\sigma_2}\norm{Y_0}}{n}\brac{\int_0^t (t-s)^pe^{-(t-s) p\sigma_{\min} }ds }^{1/p}\brac{\int_{\norm{z}\geq 1}\norm{z}^{p-d-\alpha}dz  }^{1/p}. 
\end{align*}

By a change of variable,
\begin{align*}
    \int_{\norm{z}\geq 1}\norm{z}^{p-d-\alpha}dz &= V_d \int_1^\infty r^{p-d-\alpha}r^{d-1}dr=\frac{V_d}{\alpha-p}.
\end{align*}
Furthermore, $p\in [1,\alpha)$ implies
\begin{align*}
    &\int_0^t (t-s)^pe^{-(t-s) p\sigma_{\min} }ds\\
    &\leq \int_{t-1}^t e^{-(t-s)\sigma_{\min}}ds+\int_0^{t-1}(t-s)^2e^{-(t-s)\sigma_{\min}} ds\\
    &=\frac{e^{\sigma_{\min}(s-t) }}{\sigma}\bigg{|}_{t-1}^t
    \\&+
    \frac{1}{\sigma_{\min}^3}e^{\sigma_{\min}(s-t)} \brac{\sigma_{\min}\brac{\sigma_{\min}^2+2\sigma_{\min}+2}-e^{-\sigma_{\min} t}\brac{\sigma_{\min}^2(s-t)^2-2\sigma_{\min}(s-t)+2 }}\bigg{|}_{0}^{t-1}\\
    &=\frac{1}{\sigma_{\min}}(1-e^{-\sigma_{\min}})+\frac{1}{\sigma_{\min}^3}\brac{e^{-\sigma_{\min}}\brac{\sigma_{\min}^2+2\sigma_{\min}+2 }-e^{-\sigma_{\min} t}\brac{\sigma_{\min}^2t^2+2\sigma_{\min}t+2}  }.
\end{align*}
Combining the previous calculations, we get
\begin{align}
\label{estimate_secondpoissonintegral}
     & \E{ \norm{\int_0^t \int_{\norm{z}\geq 1}\brac{e^{-A(t-s)}-e^{-\hat{A}(t-s)}}\Sigma z {N}(dz,ds)}^p}^{1/p}\nonumber\\
     &\leq  C(p)\frac{\zeta\abs{\sigma_1+\sigma_2}\norm{Y_0}}{n} \brac{\frac{V_d}{\alpha-p} }^{1/p}\nonumber\\
     &\qquad\cdot\bigg(\frac{1}{\sigma_{\min}}(1-e^{-\sigma_{\min}})\nonumber\\
     &\qquad\qquad+\frac{1}{\sigma_{\min}^3}\brac{e^{-\sigma_{\min}}\brac{\sigma_{\min}^2+2\sigma_{\min}+2 }-e^{-\sigma_{\min} t}\brac{\sigma_{\min}^2t^2+2\sigma_{\min}t+2}  }\bigg)^{1/p}.
\end{align}

Finally, we are able to deduce the desired estimate on $\mathcal{W}_p\brac{\operatorname{Law}(Y_t),\operatorname{Law}(\hat{Y}_t)}$ from \eqref{estimate_pwasserstein}, \eqref{estimate_firstpoissonintegral} and \eqref{estimate_secondpoissonintegral}. Now, the same argument as the one at the end of the proof of Theorem~\ref{theorem_wassersteinbound_appendix} (in particular letting $t\to \infty$) will lead to the bound at \eqref{wassersteinbound_squareloss_with_momentum} for SGDs with momentum.

The proof of the bound at \eqref{wassersteinbound_squareloss_without_momentum} (for SGDs without momentum) is along the same line with $\sigma_{\min}$ being replaced with $\theta_{\min}$. 
This completes the proof.

%%%%%%%%%%%%%%%%%%%%%%%%%%%%%%%%%%%%
\subsection{Proof of Proposition~\ref{lemma_comparesingularvalues}}
\label{section_proofcomparesingularvalues}

First of all, we notice that 
\begin{align*}
   AA^{\top}= \begin{bmatrix}
        I_d & -\gamma I_d\\
        -\gamma I_d & \gamma^2I_d+\brac{\frac{1}{n}X^{\top}X}^2
    \end{bmatrix}.
\end{align*}

Let us assume that $\frac{1}{n}X^{\top}X$ has the decomposition
\begin{equation*}
\frac{1}{n}X^{\top}X=VDV^{\top},
\end{equation*}
where $D$ is diagonal consisting of eigenvalues $\kappa_{i}$, $1\leq i\leq d$ of $\frac{1}{n}X^{\top}X$, 
and $V$ is an orthogonal matrix. Then
\begin{equation*}
\gamma^2I_d+\brac{\frac{1}{n}X^{\top}X}^2=V\tilde{D}V^{\top},
\end{equation*}
where $\tilde{D}=\gamma^2I_{d}+D^2$ is diagonal matrix with entries 
$\gamma^2+\kappa_i^2$, $1\leq i\leq d$. 
Therefore, the matrix $AA^{\top}$
has the same eigenvalues as the matrix
\begin{align*}
\begin{bmatrix}
    I_d &-\gamma I_d\\
    -\gamma I_d & \tilde{D}
\end{bmatrix},
\end{align*}
which has the same eigenvalues as the matrix:
\begin{equation*}
\left[
\begin{array}{cccc}
T_{1} & \cdots & 0 & 0
\\
0 & T_{2} & \cdots & 0
\\
\vdots & \cdots & \ddots & \vdots 
\\ 
0 & 0 & \cdots & T_{d}
\end{array}
\right],
\end{equation*}
where 
\begin{equation*}
T_{i}=\left[
\begin{array}{cc}
1 & -\gamma
\\
-\gamma & \gamma^2+\kappa_i^2
\end{array}
\right],
\qquad 1\leq i\leq d,
\end{equation*}
are $2\times 2$ matrices with eigenvalues:
\begin{align}
\label{eigenvalues}
\mu_{i,\pm}&=\frac{\gamma^2+\kappa_i^2+1\pm \sqrt{\brac{\gamma^2+\kappa_i^2+1}^2-4\kappa_i^2} }{2}\nonumber\\
&=\frac{\gamma^2+\kappa_i^2+1\pm \sqrt{\brac{\gamma^2+(\kappa_i-1)^2 }\brac{\gamma^2+(\kappa_i+1)^2 }} }{2}. 
\end{align}
with $1\leq i\leq d$. 

Notice that $\theta_{\min}=\min_{1\leq i\leq d}\{\kappa_i\}$ and $\sigma_{\min}=\min_{1\leq i\leq d}\{\sqrt{\mu_{i,\pm}}\}$. 
Moreover, it is easy to see that
\begin{equation}
\min\left\{\mu_{i,+},\mu_{i,-}\right\}=
\mu_{i,-},\qquad\text{for any $i=1,2,\ldots,d$}.
\end{equation}
Therefore, $\sigma_{\min}=\min_{1\leq i\leq d}\{\sqrt{\mu_{i,-}}\}$.
Moreover, one can verify that 
\begin{equation}\label{to:verify}
\mu_{i,-}\leq \kappa_i^2,
\end{equation} 
for any $i=1,2,\ldots,d$ and any $\gamma>0$, 
which implies the desired conclusion that 
\begin{align*}
    \sigma_{\min}\leq \theta_{\min}. 
\end{align*}
Finally, let us prove \eqref{to:verify}.
Note that \eqref{to:verify} is equivalent to:
\begin{equation}
\frac{\gamma^2+\kappa_i^2+1-\sqrt{\brac{\gamma^2+\kappa_i^2+1}^2-4\kappa_i^2} }{2}
\leq \kappa_i^2,
\end{equation}
which can be re-written as
\begin{equation}\label{to:verify:2}
\gamma^{2}-\kappa_{i}^{2}+1
\leq\sqrt{\brac{\gamma^2+\kappa_i^2+1}^2-4\kappa_i^2}.
\end{equation}
To show that \eqref{to:verify:2} holds, it suffices to show that
\begin{equation}\label{to:verify:3}
\left(\gamma^{2}-\kappa_{i}^{2}+1\right)^{2}
\leq \left(\gamma^2+\kappa_i^2+1\right)^{2}-4\kappa_i^2.
\end{equation}
It is easy to compute that
\begin{align}
&\left(\gamma^2+\kappa_i^2+1\right)^{2}-\left(\gamma^{2}-\kappa_{i}^{2}+1\right)^{2}
\nonumber
\\
&=\left(\left(\gamma^2+\kappa_i^2+1\right)-\left(\gamma^2-\kappa_i^2+1\right)\right)
\left(\left(\gamma^2+\kappa_i^2+1\right)+\left(\gamma^2-\kappa_i^2+1\right)\right)
\nonumber
\\
&=4\kappa_{i}^{2}(\gamma^{2}+1)\geq 4\kappa_{i}^{2}.
\end{align}
Hence, \eqref{to:verify:3} holds.
This completes the proof.

%%%%%%%%%%%%%%%%%%%%%%%%%%%%%%%%%%%%%%%%%%%%%%%%%%%%%%%%%%%%%%%%%%%%%%%%%%%%%%%%%%%%%%%%%%%%%%%%%%%%%%%%%%%%%%%%%%%%%%%%%%%%%%%%%%%%%%%%%%%%%%%

\section{Proof of the Results in Section~\ref{sec:discrete}}\label{appendix:discrete}

\subsection{Notations}
Let us recall from \eqref{discrete_equation:0}-\eqref{discrete_equation:1} 
the discrete-time dynamics:
\begin{align}
    \label{discrete_equation}
    V_{k+1}&=V_k-\eta\gamma V_k-\eta\nabla\widehat{F} (\Theta_k,X_n)+\zeta\xi_{k+1},\nonumber\\
    \Theta_{k+1}&=\Theta_k+\eta V_{k+1}, 
\end{align}
and
\begin{align}\label{discrete_equation_hat}
    \hat{V}_{k+1}&=\hat{V}_k-\eta\gamma\hat{V}_k-\eta\nabla\widehat{F} (\hat{\Theta}_k,\widehat{X}_n)+\zeta\xi_{k+1},\nonumber\\
    \hat{\Theta}_{k+1}&=\hat{\Theta}_k+\eta\hat{V}_{k+1}, 
\end{align}
with $\xi_{k+1}:=L_{k+1}-L_k$ and $(\Theta_{0},V_{0})=(\hat{\Theta}_{0},\hat{V}_{0})=(w,y)$. 

For the process $\left\{\brac{\Theta_k,V_k}:k\geq 0\right\}$ given in \eqref{discrete_equation} and the process $\left\{\brac{\hat{\Theta}_k,\hat{V}_k}:k\geq 0\right\}$ given in \eqref{discrete_equation_hat}, we write
\begin{equation*}
\brac{\Theta_k,V_k}=\brac{\Theta_k^{w,y},V_k^{w,y}},
\qquad
\brac{\hat{\Theta}_k,\hat{V}_k}=\brac{\hat{\Theta}_k^{w,y},\hat{V}_k^{w,y}},
\end{equation*}
for any $k=0,1,2,\ldots$ to emphasize the dependence on the initialization $\brac{\Theta_0,V_0}=\brac{\hat{\Theta}_0,\hat{V}_0}=(w,y)$.

In addition to the notations in Section~\ref{section_notationsinappendix}, we also introduce the following notations:
\begin{itemize}
\item 
$\{Q_k:k\in\N\}$ are the semi-groups corresponding to \eqref{discrete_equation}.
\item
$\{\hat{Q}_k:k\in\N\}$ are the semi-groups corresponding to \eqref{discrete_equation_hat}. 
\item 
Define
\begin{align*}
&\nabla \mathbb{W}(x,v,X_n):=\begin{pmatrix}
    \nabla_x \mathbb{W}(x,v,X_n)\\\nabla_v \mathbb{W}(x,v,X_n)
\end{pmatrix} ;
\\
&\nabla^2 \mathbb{W}(x,v,X_n):=\begin{pmatrix}
    \nabla_x\nabla_x \mathbb{W}(x,v,X_n) &     \nabla_v\nabla_x \mathbb{W}(x,v,X_n)\\
    \nabla_x\nabla_v \mathbb{W}(x,v,X_n) & \nabla_v\nabla_v \mathbb{W}(x,v,X_n)
\end{pmatrix}.
\end{align*}
\end{itemize}

%%%%%%%%%%%%%%%%%%%%%%%%%%%%%%%%%%%%%%%%%%%%%%%%%%%%%%%%%%%%%%%%%%%%%%%%%%%%%%%%%%%%%%%%%%%%%%%%%%%%%%%%%%%%%%%%%%%%%%%%%%%%%%%%%%%%%%%%%%%%%%%

\subsection{Proof of Theorem~\ref{thm:ergodicity:discretedynamics}}

\begin{theorem}[Restatement of Theorem~\ref{thm:ergodicity:discretedynamics}]
\label{theorem_ergodicity_discretedynamics}
   Assume Conditions~\ref{cond_gammaandbeta},~\ref{cond_pseudolipschitz}, and~\ref{cond_allthelambdas} and also that $\sup_{x,y\in \mathcal{X}}\norm{x-y}\leq D$ for some $D<\infty$. The Markov chains $\{(\Theta_n,V_n):n\in\N\}$ and $\{(\hat{\Theta}_n,\hat{V}_n):n\in\N\}$ respectively admit unique invariant measures $\mu_\eta$ and $\hat{\mu}_\eta$, provided that
   \begin{align}
    \eta&<\bar{\eta}:=\min\Bigg\{\frac{1}{4}\left(\max\left\{1,\gamma, \beta\left(K_1+2K_2D  \right) \right\}\right)^{-1};
    \nonumber
    \\
    &\qquad\frac{1}{2}\min \left\{\frac{1}{2}(\gamma-r_0),\frac{1}{2}\beta r_0(\lambda_1-\lambda_2\lambda_4),r_0\lambda_2  \right\}\frac{1}{1+r^2}\cdot\min  \left\{1,\sqrt{ \frac{r^2-r_0^2}{4}}, \sqrt{\frac{r^2-r_0^2}{4r^2}} \right\}\nonumber\\
    &\qquad\qquad\cdot\Bigg(  \left\{ \frac{1}{\sqrt{2}}, \frac{1}{2}\frac{r^2-r_0^2}{4}, \frac{1}{2}\frac{r^2-r_0^2}{4r^2}\right\}^{-1}+\frac{2}{3}C_4\max \left\{1-\gamma,2\beta \left(K_1+2K_2D \right) \right\}  \Bigg)^{-1}\Bigg\},\label{bar:eta:defn} 
\end{align}
where $r,r_0$ are constants given in Lemma~\ref{lemma_lyapunov}. 
\end{theorem}

%%%%%%%%%%%%%%%%%%%%%%%%%%%%
\begin{proof} 
We will show the ergodicity for the Markov chain $\{(\Theta_{n},V_{n}):n\in\mathbb{N}\}$.
The argument for the Markov chain $\left\{\left(\hat{\Theta}_{n},\hat{V}_{n}\right):n\in\mathbb{N}\right\}$ is similar
and hence omitted.
Our strategy to verify the ergodicity is to use \citep[Theorem 6.3]{meyntweediestability_i}. We start with
\begin{align}
\label{ergodicty_decomposition}
   & \E{\mathbb{W}(\Theta_1,V_1,X_n)|\Theta_0=x, V_0=v}- \mathbb{W}(x,v,X_n)=\mathcal{D}_1(x,v,X_n)+\mathcal{D}_2(x,v,X_n),
\end{align}
where 
\begin{align*}
    &\mathcal{D}_1(x,v,X_n)\\
    &:=\E{\mathbb{W}\brac{x+\eta(1-\eta\gamma)v-\eta^2\beta\nabla\widehat{F}(x,X_n)+\eta\zeta L_\eta, (1-\eta\gamma)v-\eta\beta\nabla\widehat{F}(x,X_n)+\zeta L_\eta,X_n } }\\
    &\hspace{5em}-\mathbb{W}\brac{x+\eta(1-\eta\gamma)v-\eta^2\beta\nabla\widehat{F}(x,X_n), (1-\eta\gamma)v-\eta\beta\nabla\widehat{F}(x,X_n),X_n },
\end{align*}
and 
\begin{align*}
    &\mathcal{D}_2(x,v,X_n)\\&:=\mathbb{W}\brac{x+\eta(1-\eta\gamma)v-\eta^2\beta\nabla\widehat{F}(x,X_n), (1-\eta\gamma)v-\eta\beta\nabla\widehat{F}(x,X_n),X_n }-\mathbb{W}(x,v,X_n). 
\end{align*}
By Dynkin's formula,
\begin{align*}
    \mathcal{D}_1(x,v,X_n)&=\int_0^\eta \mathbb{E}\bigg[\Delta^{\alpha/2}\mathbb{W}\bigg(x+\eta(1-\eta\gamma)v-\eta^2\beta\nabla\widehat{F}(x,X_n)+\eta\zeta L_s,\\
    &\qquad\qquad\qquad\qquad\qquad\qquad\qquad(1-\eta\gamma)v-\eta\beta\nabla\widehat{F}(x,X_n)+\zeta L_s, X_n \bigg)  \bigg]ds,
\end{align*}
where $\Delta^{\alpha/2}$ is the fractional Laplacian operator:
\begin{align*}
    \Delta^{\alpha/2} f(x,v)
    &:= C_{2d,\alpha}\int_{\R^d}\int_{\R^d}\big(f(x+z_1,v+z_2)-f(x,v)\\
    &\qquad-\brac{\inner{\nabla_x f(x,v),z_1}+\inner{\nabla_v f(x,v),z_2}}\mathds{1}_{\{\norm{z_1,z_2}\leq 1\}} \big)\frac{1}{\norm{(z_1,z_2)}^{2d+\alpha}}dz_1dz_2, 
\end{align*}
with $C_{2d,\alpha}:=\alpha 2^{\alpha-1}\pi^{-d}\Gamma\brac{\frac{2d+\alpha}{2}}/\Gamma(1-\frac{\alpha}{2})$. 
Then as shown in \citep[(A.2), Proof of Proposition 1.5]{chenxu2023euler}, the fact that $\norm{\nabla \mathbb{W}(x,v,X_n)}_{\operatorname{op},\infty}<C_3$ and also the fact that $\norm{\nabla^2 \mathbb{W}(x,v,X_n)}_{\operatorname{op},\infty}<C_4$ in Lemma~\ref{lemma_boundgradientlyapunov} implies 
\begin{align*}
    &\sup_{(x,v)\in\R^{2d}} \norm{\Delta^{\alpha/2}\mathbb{W}(x,v,X_n)}\\
    &\leq C_{2d,\alpha}\int_{\norm{y}<1}\int_0^1 \int_{0}^r C_4\norm{y}^{2-\alpha-2d}dsdrdy+C_{2d,\alpha} \int_{\norm{y}\geq 1}\int_0^1 C_3 \norm{y}^{1-\alpha-2d}drdy\\
    &=C_{2d,\alpha}\frac{1}{2}C_4 \frac{V_{2d}}{2(2-\alpha)}+C_{2d,\alpha}C_3\frac{2V_{2d}}{\alpha-1}\\
    &\leq C_{2d,\alpha}(C_3+C_4) 2V_{2d} \brac{\frac{1}{2-\alpha}+\frac{1}{\alpha-1} }. 
\end{align*}
In the above, $V_{2d}=\frac{\pi^{d}}{\Gamma(d+1)}$ is the volume of the unit ball in $\R^{2d}$. This implies 
\begin{align}
\label{step_boundD1}
   \norm{ \mathcal{D}_1(x,v,X_n)} \leq C_{2d,\alpha}(C_3+C_4) 2V_{2d} \brac{\frac{1}{2-\alpha}+\frac{1}{\alpha-1} }\eta . 
\end{align} 

Next, let us define 
\begin{align*}
    U_1(x,v,X_n):=(1- \eta\gamma)v-\eta\beta \nabla \widehat{F}(x,X_n).
\end{align*}
Then we can rewrite $\mathcal{D}_{2}(x,v,X_n)$ as:
\begin{align}
\label{expandD2}
    \mathcal{D}_2(x,v,X_n)&= \mathbb{W}\brac{x+\eta U_1(x,v,X_n),U_1(x,v,X_n), X_{n} }-\mathbb{W}(x,v,X_{n})\nonumber\\
    &=\inner{\nabla_x \mathbb{W}(x,v,X_n),\eta U_1(x,v,X_n)}+\inner{\nabla_v \mathbb{W}(x,v,X_n), U_1(x,v,X_n)-v}\nonumber
    \\
    &\qquad\qquad\qquad\qquad\qquad+S(x,v,X_n),
\end{align}
where 
\begin{align*}
   &S(x,v,X_n)\\
   &:=\int_0^1 \inner{\nabla \mathbb{W}\brac{x+s\eta U_1(x,v,X_n),v+s(U_1(x,v,X_n)-v),X_n }, \begin{pmatrix}
       \eta U_1(x,v,X_n)\\U_1(x,v,X_n)-v
   \end{pmatrix} }ds.
\end{align*}
Let us consider the terms on the right hand side of \eqref{expandD2}. Recall the definition of $N(x,v,X_n)$ in \eqref{definition_N}. Then, we have
\begin{align}
\label{contractioninequality}
    &\inner{\nabla_x \mathbb{W}(x,v,X_n),\eta U_1(x,v,X_n)}+\inner{\nabla_v \mathbb{W}(x,v,X_n),U_1(x,v,X_n)-v}\nonumber\\
    &=1/2 \left(N(x,v,X_n)\right)^{-1/2} \inner{\nabla V_0(x,X_n)+r^2x+r_0v,\eta v-\eta^2\gamma v-\eta\beta\nabla \widehat{F}(x,X_n) }\nonumber\\
    &\qquad\qquad\qquad+1/2 \left(N(x,v,X_n)\right)^{-1/2}  \inner{v+r_0x, -\eta\gamma v-\eta\beta\nabla \widehat{F}(x,X_n) }\nonumber\\
    &=\inner{\nabla V_0(x,X_n)+r^2x+r_0v,\eta v }+\inner{ v+r_0x,-\eta \gamma v-\eta\beta \nabla \widehat{F}(x,X_n) }+R(x,v,X_n),
\end{align} 
where 
\begin{align*}
R(x,v,X_n):=&1/2\left(N(x,v,X_n)\right)^{-1/2}\inner{\nabla V_0(x,X_n)+r^2x+r_0v, -\eta^2\gamma v-\eta^2\beta \nabla \widehat{F}(x,X_n)}.  
\end{align*}
 We follow \citep[Lemma 4.4]{jianwangbao2022coupling} (specifically the proof therein which contains explicit constants) and write 
 \begin{align*}
    &\inner{\nabla V_0(x,X_n)+r^2x+r_0v,\eta v }+\inner{ v+r_0x,-\eta \gamma v-\eta\beta \nabla \widehat{F}(x,X_n) }\\
    &\leq \eta \brac{-\frac{1}{2}(\gamma-r_0)\norm{v}^2-\frac{1}{2}\beta r_0(\lambda_1-\lambda_2\lambda_4)\norm{x}^2-r_0\lambda_2V_0(x)+\beta r_0(\lambda_3+\lambda_2\lambda_5) }\\
    &\leq -\eta \cdot \min \left\{\frac{1}{2}(\gamma-r_0),\frac{1}{2}\beta r_0(\lambda_1-\lambda_2\lambda_4),r_0\lambda_2  \right\}(\norm{x}^2+\norm{v}^2+V_0(x,X_n))\\
    &\hspace{23em}+\eta\cdot\beta r_0(\lambda_3+\lambda_2\lambda_5). 
\end{align*}
 Furthermore, inequality \eqref{lyapunovfunction_upperandlowerestimate} says $N(x,v,X_n)\geq 1$ 
 and $N(x,v,X_n) \leq 1+V_0(x,X_n)+r^2\norm{x}^2+\norm{v}^2$, so that 
 \begin{align*}
    \norm{x}^2+\norm{v}^2+V_0(x,X_n)&\geq \frac{1}{\max\{1,r^2\}}\brac{N(x,v,X_n) -1}\\
    &\geq \frac{1}{1+r^2}(N(x,v,X_n)^{1/2}-1)=\frac{1}{1+r^2}(\mathbb{W}(x,v,X_n)-2). 
 \end{align*}
 Consequently, 
\begin{align}
\label{contractioninequality_mainterm}
    &\inner{\nabla V_0(x,X_n)+r^2x+r_0v,\eta v }+\inner{ v+r_0x,-\eta \gamma v-\eta\beta \nabla \widehat{F}(x,X_n) }\nonumber\\
    &\leq -\eta\cdot \min \left\{\frac{1}{2}(\gamma-r_0),\frac{1}{2}\beta r_0(\lambda_1-\lambda_2\lambda_4),r_0\lambda_2  \right\}\frac{1}{1+r^2}\mathbb{W}(x,v,X_n)\nonumber\\
    &\qquad\qquad+\eta\cdot\Bigg(\min \left\{\frac{1}{2}(\gamma-r_0),\frac{1}{2}\beta r_0(\lambda_1-\lambda_2\lambda_4),r_0\lambda_2  \right\}\frac{(-2)}{1+r^2} +\beta r_0(\lambda_3+\lambda_2\lambda_5) \Bigg). 
\end{align}
Next, we consider $R(x,v,X_n)$ on the right hand side of \eqref{contractioninequality}. Via \eqref{lowerboundN}, \eqref{bound_gradientF} and \eqref{boundnormgradientV_0}, we can compute that
\begin{align}
\label{contractioninequality_remainderterm}
  &\norm{R(x,v,X_n) }\nonumber\\
  &\leq \eta^2\frac{1}{2} \min  \left\{ \frac{1}{\sqrt{2}}, \frac{1}{2}\frac{r^2-r_0^2}{4}, \frac{1}{2}\frac{r^2-r_0^2}{4r^2}\right\}^{-1}\frac{1}{\norm{x}+\norm{v}+1}\nonumber\\
  &\qquad\qquad\cdot\brac{\norm{\nabla V_0(x,X_n)}+r^2\norm{x}+r_0\norm{v} }\brac{\gamma \norm{v}+\beta\norm{\nabla \hat{F}(x,X_n) }}\nonumber\\
  &\leq \eta^2\frac{1}{2} \min  \left\{ \frac{1}{\sqrt{2}}, \frac{1}{2}\frac{r^2-r_0^2}{4}, \frac{1}{2}\frac{r^2-r_0^2}{4r^2}\right\}^{-1}\frac{1}{\norm{x}+\norm{v}+1}\nonumber\\
  &\qquad\qquad\cdot\brac{\beta\brac{K_1+4K_2D+2\lambda_4+\norm{\nabla f(0,0)}_{\operatorname{op}} }+r^2+r_0}(\norm{x}+\norm{v}+1)\nonumber\\
  &\qquad\qquad\qquad\cdot\max \left\{\beta(K_1+2K_2D),\gamma, \norm{\nabla f(0,0)}+2K_2D \right\}(\norm{x}+\norm{v}+1)\nonumber\\
  &\leq \eta^2 \brac{\norm{x}+\norm{v}+1}  \left\{ \frac{1}{\sqrt{2}}, \frac{1}{2}\frac{r^2-r_0^2}{4}, \frac{1}{2}\frac{r^2-r_0^2}{4r^2}\right\}^{-1}\nonumber\\
  &\qquad\cdot\brac{\beta\brac{K_1+4K_2D+2\lambda_4+\norm{\nabla f(0,0)}_{\operatorname{op}} }+r^2+r_0}
  \nonumber
  \\
  &\qquad\qquad\cdot\max \left\{\beta(K_1+2K_2D),\gamma, \norm{\nabla f(0,0)}+2K_2D \right\}\nonumber\\
  &\leq \eta^2\cdot \mathbb{W}(x,v,X_n)\cdot\min  \left\{1, \sqrt{ \frac{r^2-r_0^2}{4}}, \sqrt{\frac{r^2-r_0^2}{4r^2}} \right\}^{-1}  \left\{ \frac{1}{\sqrt{2}}, \frac{1}{2}\frac{r^2-r_0^2}{4}, \frac{1}{2}\frac{r^2-r_0^2}{4r^2}\right\}^{-1}\nonumber\\
  &\qquad\cdot\brac{\beta\brac{K_1+4K_2D+2\lambda_4+\norm{\nabla f(0,0)}_{\operatorname{op}} }+r^2+r_0}\nonumber
  \\
  &\qquad\qquad\cdot\max \left\{\beta(K_1+2K_2D),\gamma, \norm{\nabla f(0,0)}+2K_2D \right\}.
\end{align}
To get the last line, we have used the inequality \eqref{lyapunovfunction_upperandlowerestimate_secondversion} which implies
\begin{align}
\label{lyapunovfunction_upperandlowerestimate_thirdversion}
    \min  \left\{1, \sqrt{ \frac{r^2-r_0^2}{4}}, \sqrt{\frac{r^2-r_0^2}{4r^2}} \right\}(\norm{x}+\norm{v}+1)\leq \mathbb{W}(x,v,X_n). 
\end{align}

Combining \eqref{contractioninequality}, \eqref{contractioninequality_mainterm} and \eqref{contractioninequality_remainderterm}, it leads to
\begin{align}
\label{maintermD2}
    &\inner{\nabla_x \mathbb{W}(x,v,X_n),\eta V_1}+\inner{\nabla_v \mathbb{W}(x,v,X_n),V_1-v}\nonumber\\
    &\leq  -\eta Q_0(\eta)\cdot \mathbb{W}(x,v,X_n) \nonumber
    \\
    &\qquad+\eta\cdot\Bigg(\min \left\{\frac{1}{2}(\gamma-r_0),\frac{1}{2}\beta r_0(\lambda_1-\lambda_2\lambda_4),r_0\lambda_2  \right\}\frac{(-2)}{1+r^2} +\beta r_0(\lambda_3+\lambda_2\lambda_5) \Bigg),
    \end{align} 
where
\begin{align*}
    Q_0(\eta)&:=\min \left\{\frac{1}{2}(\gamma-r_0),\frac{1}{2}\beta r_0(\lambda_1-\lambda_2\lambda_4),r_0\lambda_2  \right\}\frac{1}{1+r^2}\\
    &\qquad\qquad-\eta \cdot \min  \left\{1,\sqrt{ \frac{r^2-r_0^2}{4}}, \sqrt{\frac{r^2-r_0^2}{4r^2}} \right\}^{-1}  \left\{ \frac{1}{\sqrt{2}}, \frac{1}{2}\frac{r^2-r_0^2}{4}, \frac{1}{2}\frac{r^2-r_0^2}{4r^2}\right\}^{-1}\nonumber\\
  &\qquad\qquad\qquad\qquad\cdot\brac{\beta\brac{K_1+4K_2D+2\lambda_4+\norm{\nabla f(0,0)}_{\operatorname{op}} }+r^2+r_0}\nonumber
  \\
  &\qquad\qquad\qquad\qquad\qquad\qquad\cdot\max \left\{\beta(K_1+2K_2D),\gamma, \norm{\nabla f(0,0)}+2K_2D \right\}. 
\end{align*}

Next, we consider $S(x,v,X_n)$ on the right hand side of \eqref{expandD2}. We have $ \norm{\nabla^2 \mathbb{W}(x,v,X_n)}_{\operatorname{op}}<\frac{C_4}{1+\norm{x}+\norm{v}}$ from Lemma~\ref{lemma_boundgradientlyapunov} so that 
\begin{align}
\label{equationS(x,v,X_n)}
    &\norm{S(x,v,X_n)}\nonumber\\
    &=\Bigg| \int_0^1\int_0^s \begin{pmatrix}
        \eta U_1(x,v,X_n)\\U_1(x,v,X_n)-v
    \end{pmatrix}^{\top} \nabla^2\mathbb{W}\brac{x+ts\eta U_1(x,v,X_n),v+ts(U_1(x,v,X_n)-v) ,X_n} \nonumber\\&\hspace{22em}\cdot\begin{pmatrix}
        \eta U_1(x,v,X_n)\\U_1(x,v,X_n)-v
    \end{pmatrix}dtds\Bigg|\nonumber\\
    &\leq C_4 \int_0^1\int_0^s s\brac{\norm{\eta U_1(x,v,X_n)}+\norm{U_1(x,v,X_n)-v} }^2\nonumber\\
    &\hspace{12em}\cdot\frac{1}{1+\norm{x+\eta\cdot ts U_1(x,v,X_n)}+\norm{v+ts(U_1(x,v,X_n)-v)} }dtds\nonumber\\
     &\leq C_4\cdot \eta^2 \int_0^1\int_0^s s\brac{(1-\eta\gamma+\gamma)\norm{v}+\brac{\eta\beta+\beta}\norm{\nabla \widehat{F}(x,X_n)}  }^2\nonumber\\
    &\hspace{12em}\cdot\frac{1}{1+\norm{x+\eta\cdot ts U_1(x,v,X_n)}+\norm{v+ts(U_1(x,v,X_n)-v)} }dtds.
\end{align}
From \eqref{bound_gradientF}, we know that
\begin{align*}
    &(1-\eta\gamma+\gamma)\norm{v}+\brac{\eta\beta+\beta}\norm{\nabla \widehat{F}(x,X_n)} \\
    &\leq \max \left\{1-\eta\gamma+\gamma,(\eta\beta+\beta )\left(K_1+2K_2D \right) \right\}(1+\norm{x}+\norm{v})\\
    &\leq \max \left\{1-\gamma,2\beta \left(K_1+2K_2D \right) \right\}(1+\norm{x}+\norm{v}),
\end{align*}
for $\eta\leq 1$.

Moreover, one can write 
\begin{align*}
    \norm{x+\eta\cdot ts U_1(x,v,X_n)}&\geq \norm{x}-\eta\cdot ts\norm{U_1(x,v,X_n)}\\
    &\geq \norm{x}-\eta\cdot ts\max \left\{ 1,\beta\left(K_1+2K_2D  \right)\right\}(1+\norm{x}+\norm{v}),
\end{align*}
and 
\begin{align*}
    \norm{v+ts(U_1(x,v,X_n)-v)} &\geq \norm{v}-ts\norm{U_1(x,v,X_n)-v}\\
    &\geq \norm{v}-\eta\cdot ts \max\left\{\gamma, \beta\left(K_1+2K_2D  \right) \right\}(1+\norm{x}+\norm{v}),
\end{align*}
which leads to 
\begin{align*}
    &1+\norm{x+\eta\cdot ts U_1(x,v,X_n)}+\norm{v+ts(U_1(x,v,X_n)-v)}\\
    &\geq(1+\norm{x}+\norm{v})\brac{1- \eta\cdot ts\cdot 2\max\left\{1,\gamma, \beta\left(K_1+2K_2D  \right) \right\} }\\
    &\geq \frac{1}{2}(1+\norm{x}+\norm{v}), 
\end{align*}
with the last line being a consequence of choosing
\begin{align*}
    \eta<\frac{1}{4}\left(\max\left\{1,\gamma, \beta\left(K_1+2K_2D  \right) \right\}\right)^{-1}. 
\end{align*}
Hence, we deduce from \eqref{equationS(x,v,X_n)} and \eqref{lyapunovfunction_upperandlowerestimate_thirdversion} that for such values of $\eta$, 
\begin{align}
\label{estimateS}
   & \norm{S(x,v,X_n)}\nonumber\\&\leq \eta^2\cdot \frac{2}{3}C_4\max \left\{1-\gamma,2\beta \left(K_1+2K_2D \right) \right\}(1+\norm{x}+\norm{v})\nonumber\\
    &\leq \eta^2\cdot  \mathbb{W}(x,v,X_n)\cdot \frac{2}{3}C_4\max \left\{1-\gamma,2\beta \left(K_1+2K_2D \right) \right\} \min \left\{1,\sqrt{ \frac{r^2-r_0^2}{4}}, \sqrt{\frac{r^2-r_0^2}{4r^2}} \right\}^{-1} . 
\end{align}
Combining \eqref{expandD2}, \eqref{maintermD2} and \eqref{estimateS} gives us 
\begin{align*}
    \mathcal{D}_2(x,v,X_n)&\leq  -\eta Q_1(\eta)\cdot \mathbb{W}(x,v,X_n)
    \\
    &\quad+\eta\cdot\Bigg(\min \left\{\frac{1}{2}(\gamma-r_0),\frac{1}{2}\beta r_0(\lambda_1-\lambda_2\lambda_4),r_0\lambda_2  \right\}\frac{(-2)}{1+r^2} +\beta r_0(\lambda_3+\lambda_2\lambda_5) \Bigg),
\end{align*}
where 
\begin{align*}
    Q_1(\eta)&:=\min \left\{\frac{1}{2}(\gamma-r_0),\frac{1}{2}\beta r_0(\lambda_1-\lambda_2\lambda_4),r_0\lambda_2  \right\}\frac{1}{1+r^2}\\
    &\qquad-\eta \cdot \min  \left\{1,\sqrt{ \frac{r^2-r_0^2}{4}}, \sqrt{\frac{r^2-r_0^2}{4r^2}} \right\}^{-1}  \left\{ \frac{1}{\sqrt{2}}, \frac{1}{2}\frac{r^2-r_0^2}{4}, \frac{1}{2}\frac{r^2-r_0^2}{4r^2}\right\}^{-1}\nonumber\\
  &\qquad\qquad\cdot\brac{\beta\brac{K_1+4K_2D+2\lambda_4+\norm{\nabla f(0,0)}_{\operatorname{op}} }+r^2+r_0}
  \nonumber
  \\
  &\qquad\qquad\qquad\cdot\max \left\{\beta(K_1+2K_2D),\gamma, \norm{\nabla f(0,0)}+2K_2D \right\}\nonumber\\
  &\qquad\qquad- \eta\cdot \min \left\{1,\sqrt{ \frac{r^2-r_0^2}{4}}, \sqrt{\frac{r^2-r_0^2}{4r^2}} \right\}^{-1}\frac{2}{3}C_4\max \left\{1-\gamma,2\beta \left(K_1+2K_2D \right) \right\}  . 
\end{align*}
By letting 
\begin{align}
\label{def_C6}
   & C_6:= \min \left\{\frac{1}{2}(\gamma-r_0),\frac{1}{2}\beta r_0(\lambda_1-\lambda_2\lambda_4),r_0\lambda_2  \right\}\frac{1}{1+r^2}\nonumber ,\\
    &\widetilde{C}_6:=\min \left\{\frac{1}{2}(\gamma-r_0),\frac{1}{2}\beta r_0(\lambda_1-\lambda_2\lambda_4),r_0\lambda_2  \right\}\frac{(-2)}{1+r^2} +\beta r_0(\lambda_3+\lambda_2\lambda_5),
\end{align}
and choosing 
\begin{align*}
    \eta&\leq \frac{1}{2}C_6 \Bigg\{\min  \left\{1,\sqrt{ \frac{r^2-r_0^2}{4}}, \sqrt{\frac{r^2-r_0^2}{4r^2}} \right\}^{-1}  \left\{ \frac{1}{\sqrt{2}}, \frac{1}{2}\frac{r^2-r_0^2}{4}, \frac{1}{2}\frac{r^2-r_0^2}{4r^2}\right\}^{-1}\\
    &\qquad\qquad+\min \left\{1,\sqrt{ \frac{r^2-r_0^2}{4}}, \sqrt{\frac{r^2-r_0^2}{4r^2}} \right\}^{-1}\frac{2}{3}C_4\max \left\{1-\gamma,2\beta \left(K_1+2K_2D \right) \right\}  \Bigg\}^{-1}, 
\end{align*}
we get $Q_1(\eta)\geq C_6/2$. Hence, with such choice of $\eta$, we arrive at
\begin{align*}
    \mathcal{D}_2(x,v,X_n)&\leq  \eta\cdot \frac{C_6}{2}\mathbb{W}(x,v,X_n)+\widetilde{C}_6\eta. 
\end{align*}

The above estimate of $ \mathcal{D}_2(x,v,X_n)$, the estimate on $\mathcal{D}_1(x,v,X_n)$ at \eqref{step_boundD1} and the decomposition at \eqref{ergodicty_decomposition} lead to 
\begin{align}
\label{inequality_contractionlyapunov}
     &\E{\mathbb{W}(\Theta_1,V_1,X_n)|\Theta_0=x, V_0=v} \nonumber\\
     &\leq \brac{1-\frac{C_6}{2}\eta} \mathbb{W}(x,v,X_n)+\brac{\widetilde{C}_6+C_{2d,\alpha}(C_3+C_4) 2V_{2d} \brac{\frac{1}{2-\alpha}+\frac{1}{\alpha-1} } }\eta\nonumber\\
     &=\brac{1-\frac{C_6}{2}\eta} \mathbb{W}(x,v,X_n)+C_8\eta,
\end{align}
for 
\begin{align}
\label{def_C8}
    C_8:=\widetilde{C}_6+C_{2d,\alpha}(C_3+C_4) 2V_{2d} \brac{\frac{1}{2-\alpha}+\frac{1}{\alpha-1} },
\end{align}
where $V_{2d}=\frac{\pi^{d}}{\Gamma(d+1)}$ is the volume of the unit ball in $\R^{2d}$. 
Observe that whenever we have $A(x)\leq C\norm{x} +C'$ for some positive constants $C,C'$, then we can write 
\begin{align*}
    A(x)\leq C\norm{x}+C'\mathds{1}_{\{C\norm{x}\leq 2C'\}}(x). 
\end{align*}

Consequently, we arrive at the estimate
\begin{align*}
    \E{\mathbb{W}(\Theta_1,V_1,X_n)|\Theta_0=x, V_0=v} \leq \brac{1-\frac{C_6}{2}\eta} \mathbb{W}(x,v,X_n) +\mathds{1}_{A}(x,v),
\end{align*}
where $A$ is the compact set
\begin{align*}
    A:=\left\{(x,v):\brac{1-\frac{C_6}{2}\eta}\cdot\norm{(x,v)}\leq 2C_8\eta \right\}. 
\end{align*}
Now one can follow \citep[Appendix A]{lihuxu2022central} to show $\{(\Theta_n,V_n):n\in\N\}$ is an irreducible Markov chain. Then via \citep[Theorem 6.3]{meyntweediestability_i}, our Markov chain is indeed ergodic and admits a unique invariant probability measure. 
The proof is complete.
\end{proof}

%%%%%%%%%%%%%%%%%%%%%%%%%%%%%%%%%%%%%%%%%%%%%%%%%%%%%%%%%%%%%%%%%%%%%%%%%%%%%%%%%%%%%%%%%%%%%%%%%%%%%%%%%%%%%%%%%%%%%%%%%%%%%%%%%%%%%%%%%%%%%%%
\subsection{Proof of Theorem~\ref{thm:discrete} }

\begin{theorem}[restatement of Theorem~\ref{thm:discrete}] Assume Conditions~\ref{cond_gammaandbeta},~\ref{cond_pseudolipschitz}, and~\ref{cond_allthelambdas}, and also that $\sup_{x,y\in \mathcal{X}}\norm{x-y}\leq D$ for some $D<\infty$. 
Also assume that the stepsize $\eta<\bar{\eta}$ where $\bar{\eta}$ is defined in \eqref{bar:eta:defn}. The following statements hold:
    \begin{enumerate}%[label={\roman*})]
        \item For every positive integer $N$, 
    \begin{align*}
&\mathcal{W}_1\brac{\mathrm{Law}\brac{\theta^{w,y}_{N\eta},v^{w,y}_{N\eta}}, \mathrm{Law}\brac{\Theta_N^{w,y},V_N^{w,y} } }\\
 &\leq \frac{C_*C_9}{\lambda_*}\cdot\eta^{1/\alpha}\cdot \Bigg(1+\min \left\{\sqrt{ \frac{r^2-r_0^2}{8}}, \sqrt{\frac{r^2-r_0^2}{8r^2}} \right\}^{-1}
     \\&\hspace{3em}\cdot\Bigg( 2\frac{C_8}{C_6}+1+{\sqrt{1+\beta\lambda_5}+\sqrt{\beta}\sqrt{{\max_{1\leq j\leq n}\abs{f(w,x_j)} }}+\sqrt{\beta\lambda_4+r^2}\norm{w}+\norm{y}  }\Bigg)\\
    &+\sqrt{\max_{1\leq j\leq n}\abs{f(0,x_j)}}+\brac{\sqrt{K_2\max_{1\leq j\leq n}\norm{x_j}+\norm{\nabla f(0,0)} }+\sqrt{\frac{K_2\max_{1\leq j\leq n}\norm{x_j}+1}{2}}}\\&
    \hspace{15em}\cdot\min \left\{\sqrt{ \frac{r^2-r_0^2}{8}}, \sqrt{\frac{r^2-r_0^2}{8r^2}} \right\}^{-1}\\
    &\hspace{3em}\cdot\bigg( 2\frac{C_8}{C_6}+1+{\sqrt{1+\beta\lambda_5}+\sqrt{\beta}\sqrt{{\max_{1\leq j\leq n}\abs{f(w,x_j)} }}+\sqrt{\beta\lambda_4+r^2}\norm{w}+\norm{y}  }\bigg)\Bigg);
\end{align*}
and furthermore
    \begin{align*}
&\mathcal{W}_1\brac{\mathrm{Law}\brac{\hat{\theta}^{w,y}_{N\eta},\hat{v}^{w,y}_{N\eta}}, \mathrm{Law}\brac{\hat{\Theta}_N^{w,y},\hat{V}_N^{w,y} } }\\
 &\leq \frac{C_*C_9}{\lambda_*}\cdot\eta^{1/\alpha}\cdot \Bigg(1+\min \left\{\sqrt{ \frac{r^2-r_0^2}{8}}, \sqrt{\frac{r^2-r_0^2}{8r^2}} \right\}^{-1}
     \\&\hspace{3em}\cdot\Bigg( 2\frac{C_8}{C_6}+1+{\sqrt{1+\beta\lambda_5}+\sqrt{\beta}\sqrt{{\max_{1\leq j\leq n}\abs{f(w,\hat{x}_j)} }}+\sqrt{\beta\lambda_4+r^2}\norm{w}+\norm{y}  }\Bigg)\\
    &+\sqrt{\max_{1\leq j\leq n}\abs{f(0,\hat{x}_j)}}+\brac{\sqrt{K_2\max_{1\leq j\leq n}\norm{\hat{x}_j}+\norm{\nabla f(0,0)} }+\sqrt{\frac{K_2\max_{1\leq j\leq n}\norm{\hat{x}_j}+1}{2}}}\\&
   \hspace{15em} \cdot\min \left\{\sqrt{ \frac{r^2-r_0^2}{8}}, \sqrt{\frac{r^2-r_0^2}{8r^2}} \right\}^{-1}\\
    &\hspace{3em}\cdot\bigg( 2\frac{C_8}{C_6}+1+{\sqrt{1+\beta\lambda_5}+\sqrt{\beta}\sqrt{{\max_{1\leq i\leq n}\abs{f(w,\hat{x}_i)} }}+\sqrt{\beta\lambda_4+r^2}\norm{w}+\norm{y}  }\bigg)\Bigg);
\end{align*}
The constant $C_9$ is provided in Lemma \ref{lemma_onestepestimate_discretedynamics}; $r$ and $r_0$ are from Lemma \ref{lemma_lyapunov}; $C_6,C_8$ are defined in respectively \eqref{def_C6} and \eqref{def_C8}; and finally $C_*,\lambda_*$ are defined in Lemma~\ref{lemma_wassersteindecaybyBaoWang}.

    \item Let $\mu$ and $\hat{\mu}$ be respectively the invariant measure of the process $\left\{\brac{\theta^{w,y}_{t},v^{w,y}_{t}}:t\geq 0\right\}$ and the process $\left\{\brac{\hat{\theta}^{w,y}_{t},\hat{v}^{w,y}_{t}}:t\geq 0\right\}$; while $\mu_\eta$ and $\hat{\mu}_\eta$ are respectively the invariant measure of the Markov chain $\left\{\brac{\Theta_N^{w,y},V_N^{w,y} } :N\in\N\right\}$ and the Markov chain $\left\{\brac{\hat{\Theta}_N^{w,y},\hat{V}_N^{w,y} } :N\in\N\right\}$. Then it holds that  
    \begin{align*}
    &\mathcal{W}_1\brac{\mu_\eta,\mu}\leq C\eta^{1/\alpha},
    \end{align*}
and 
    \begin{align*}
    &\mathcal{W}_1\brac{\hat{\mu}_\eta,\hat{\mu}}\leq \widehat{C}\eta^{1/\alpha}. 
    \end{align*}
The constants $C$ and $\widehat{C}$ are respectively defined as 
\begin{align*}
    C&:=\frac{C_*C_9}{\lambda_*}\cdot \Bigg(1+\min \left\{\sqrt{ \frac{r^2-r_0^2}{8}}, \sqrt{\frac{r^2-r_0^2}{8r^2}} \right\}^{-1}\\
    &\hspace{4em}\cdot\Bigg( 2\frac{C_8}{C_6}+1+{\sqrt{1+\beta\lambda_5}+\sqrt{\beta}\sqrt{{\max_{1\leq j\leq n}\abs{f(0,x_j)} }} }\Bigg)+\sqrt{\max_{1\leq j\leq n}\abs{f(0,x_j)}}\\
    &\hspace{4em}+\brac{\sqrt{K_2\max_{1\leq j\leq n}\norm{x_j}+\norm{\nabla f(0,0)} }+\sqrt{\frac{K_2\max_{1\leq j\leq n}\norm{x_j}+1}{2}}}\\&
    \hspace{5em}\cdot\min \left\{\sqrt{ \frac{r^2-r_0^2}{8}}, \sqrt{\frac{r^2-r_0^2}{8r^2}} \right\}^{-1}\\
    &\hspace{7em}\cdot\bigg( 2\frac{C_8}{C_6}+1+{\sqrt{1+\beta\lambda_5}+\sqrt{\beta}\sqrt{{\max_{1\leq j\leq n}\abs{f(0,x_j)} }}  }\bigg)\Bigg),
    \end{align*}
    and
    \begin{align*}
    \widehat{C}&:=\frac{C_*C_9}{\lambda_*}\cdot \Bigg(1+\min \left\{\sqrt{ \frac{r^2-r_0^2}{8}}, \sqrt{\frac{r^2-r_0^2}{8r^2}} \right\}^{-1}\\
    &\hspace{4em}\cdot\Bigg( 2\frac{C_8}{C_6}+1+{\sqrt{1+\beta\lambda_5}+\sqrt{\beta}\sqrt{{\max_{1\leq j\leq n}\abs{f(0,\hat{x}_j)} }} }\Bigg)+\sqrt{\max_{1\leq j\leq n}\abs{f(0,\hat{x}_j)}}
    \\
    &\hspace{4em}+\brac{\sqrt{K_2\max_{1\leq j\leq n}\norm{\hat{x}_j}+\norm{\nabla f(0,0)} }+\sqrt{\frac{K_2\max_{1\leq j\leq n}\norm{\hat{x}_j}+1}{2}}}\\&
   \hspace{5em}\cdot\min \left\{\sqrt{ \frac{r^2-r_0^2}{8}}, \sqrt{\frac{r^2-r_0^2}{8r^2}} \right\}^{-1}\\
    &\hspace{7em}\bigg( 2\frac{C_8}{C_6}+1+{\sqrt{1+\beta\lambda_5}+\sqrt{\beta}\sqrt{{\max_{1\leq j\leq n}\abs{f(0,\hat{x}_j)} }}  }\bigg)\Bigg). 
\end{align*}    
    \end{enumerate}
\end{theorem}

\begin{proof}
    The proof follows the same line as the proof of Theorem~\ref{theorem_wassersteinbound_continuousdynamics_appendix}. To prove Part i), we start with a decomposition of the semigroups that is in the spirit of the classical Lindeberg's principle:
\begin{align*}
\mathcal{W}_1\brac{\mathrm{Law}\brac{\theta^{w,y}_{N\eta},v^{w,y}_{N\eta}}, \mathrm{Law}\brac{\Theta_N^{w,y},V_N^{w,y} } }&=P_{N\eta}h(w,y)-Q_{N}h(w,y)\\&=\sum_{i=1}^{N}Q_{i-1}\brac{P_\eta-Q_1 }P_{(N-i)\eta}h(w,y),
\end{align*}
which leads to
\begin{align*}
   & \sup_{h\in \operatorname{Lip}(1)} \abs{P_{N\eta}h(w,y)-Q_{N}h(w,y)}\leq \sup_{h\in \operatorname{Lip}(1)}\sum_{i=1}^{N}\abs{Q_{i-1}\brac{P_\eta-Q_1 }P_{(N-i)\eta}h(w,y)}.
\end{align*}
Lemma~\ref{lemma_semigroupgradiateestimate}  says $ \norm{\nabla P_{(N-i)\eta} h}_{\operatorname{op},\infty}\leq \norm{\nabla h}_{\operatorname{op},\infty}   C_* \exp\brac{-\lambda_*(N-i)\eta}$. This fact combined with Lemma~\ref{lemma_onestepestimate_discretedynamics} implies that for any $h\in \operatorname{Lip}(1)$, 
\begin{align*}
   & \abs{ \brac{P_\eta-Q_1 }P_{(N-i)\eta}h(w,y)}\\
   &\leq C_9\norm{\nabla P_{(N-i)\eta}h(w,y)}_{\infty,\operatorname{op}}\brac{1+\norm{w}+\norm{y}+\max_{1\leq j\leq n}\sqrt{\abs{f\brac{w,x_j}}}}\eta^{1+1/\alpha}\\
   &\leq C_9 C_* \exp\brac{-\lambda_*(N-i)\eta}\brac{1+\norm{w}+\norm{y}+\max_{1\leq j\leq n}\sqrt{\abs{f\brac{w,x_j}}}}\eta^{1+1/\alpha}. 
\end{align*}
It follows from the above calculation and the estimates in Lemma~\ref{lemma_uniformmomentbound_discretedynamics}, Lemma~\ref{lemma_momentbound_squarerootf_discretedynamics} that
\begin{align*}
    &\sup_{h\in \operatorname{Lip}(1)}\sum_{i=1}^{N}\abs{Q_{i-1}\brac{P_\eta-Q_1 }P_{(N-i)\eta}h(w,y)}\\
    &\leq \eta^{1+1/\alpha}\sum_{i=1}^{N} C_9 C_* \exp\brac{-\lambda_*(N-i)\eta}
    \\
    &\qquad\qquad\qquad\cdot\brac{1+\E{\norm{\Theta_{i-1}^{w,y} }}+\E{\norm{V_{i-1}^{w,y} }}+\max_{1\leq j\leq n}\E{\sqrt{\abs{f\brac{\Theta_{i-1}^{w,y},x_j}}}}}\\
     &\leq \eta^{1+1/\alpha}\sum_{i=1}^{N} C_9 C_* \exp\brac{-\lambda_*(N-i)\eta}\brac{1+C_5(w,y,X_n)+\max_{1\leq j\leq n}C_7(w,y,x_j)}.
\end{align*}
Finally, by using
\begin{align*}
    \sum_{i=1}^{N} \exp\brac{-\lambda_*(N-i)\eta}\leq \exp\brac{-\lambda_*(N+1)}\int_1^{N+1} \exp\brac{\lambda_* \eta s}ds\leq \frac{1}{\lambda_* \eta}, 
\end{align*}
and the definition of $C_5(w,y,X_n)$ from Lemma~\ref{lemma_uniformmomentbound_discretedynamics}, the definition of $C_7(w,y,x)$ from Lemma~\ref{lemma_momentbound_squarerootf_discretedynamics},  we can deduce the desired estimate on $\mathcal{W}_1\brac{\mathrm{Law}\brac{\theta^{w,y}_{N\eta},v^{w,y}_{N\eta}}, \mathrm{Law}\brac{\Theta_N^{w,y},V_N^{w,y} } }$. The calculation for $\mathcal{W}_1\brac{\mathrm{Law}\brac{\hat{\theta}^{w,y}_{N\eta},\hat{v}^{w,y}_{N\eta}}, \mathrm{Law}\brac{\hat{\Theta}_N^{w,y},\hat{V}_N^{w,y} } }$ is the same, and hence we omit the details. 

Part ii) is a simple consequence of Part i). Existence of the unique invariant measure of the process $\{\brac{\theta^{w,y}_{t},v^{w,y}_{t}}:t\geq 0\}$ is guaranteed by Lemma~\ref{lemma_wassersteindecaybyBaoWang}, while existence  of the unique invariant measure of the Markov chain $\{\brac{\Theta_N^{w,y},V_N^{w,y} }:N\in\N\}$ is verified in Theorem~\ref{theorem_ergodicity_discretedynamics}. Therefore, 
\begin{align*}
    \mathcal{W}_1\brac{\mu_\eta,\mu}&\leq \mathcal{W}_1\brac{\mu_\eta, \mathrm{Law}\brac{\Theta_N^{w,y},V_N^{w,y} } }\\
&\qquad+\mathcal{W}_1\brac{\mathrm{Law}\brac{\theta^{w,y}_{N\eta},v^{w,y}_{N\eta}}, \mathrm{Law}\brac{\Theta_N^{w,y},V_N^{w,y} } }+\mathcal{W}_1\brac{\mathrm{Law}\brac{\theta^{w,y}_{N\eta},v^{w,y}_{N\eta}}, \mu }. 
\end{align*}
We have 
\begin{align*}
    \lim_{N\to\infty }\mathcal{W}_1\brac{\mu_\eta, \mathrm{Law}\brac{\Theta_N^{w,y},V_N^{w,y} } }= \lim_{N\to\infty }\mathcal{W}_1\brac{\mathrm{Law}\brac{\theta^{w,y}_{N\eta},v^{w,y}_{N\eta}}, \mu }=0,
\end{align*}
and by applying $\lim_{N\to\infty}$ on both sides of the previous inequality and letting $w=y=0$, we arrive at the estimate on $\mathcal{W}_1\brac{\mu_\eta,\mu}$. The calculation for  $\mathcal{W}_1\brac{\hat{\mu}_\eta,\hat{\mu}}$ is the same. 
This completes the proof.
\end{proof}

%%%%%%%%%%%%%%%%%%%%%%%%%%%%%%%%%%%%%%%%%%%%%%%%%%%%%%%%%%%%%%%%%%%%%%
%%%%%%%%%%%%%%%%%%%%%%%%%%%%%%%%%%%%%%%%%%%%%%%%%%%%%%%%%%%%%%%%%%%%%%
\subsection{Proof of Corollary~\ref{cor:discrete:generalizationerrorbound}}

\begin{proof}
The proof is along the same line as the proof of Corollary~\ref{coro_generalizationbound} in Section~\ref{section_proofcorollarygeneralization_continuous}. Under the assumption that $\sup_{x,y\in\mathcal{X}}\|x-y\|\leq D$
for some $D<\infty$ and $X_{n}$ and $\widehat{X}_{n}$ differ by at most one data point, 
we get
\begin{align}
    \rho(X_n,\widehat{X}_n)= \frac{1}{n}\sum_{i=1}^n \norm{x_i-\hat{x}_i}
    \leq\frac{D}{n}.
\end{align}
Since for any $w,y\in\mathbb{R}^{d}$, $\brac{\Theta^{w,y}_N,V^{w,y}_N}$ converges
to the unique invariant measure as $N\rightarrow\infty$ (per Theorem~\ref{theorem_ergodicity_discretedynamics}), we can write
$\brac{\Theta^{w,y}_\infty,V^{w,y}_\infty}=\brac{\Theta_\infty,V_\infty}$, 
omitting the superscript on $w,y$.
Then it follows from Corollary~\ref{cor:discrete:algostab} and \eqref{eqn:wass_stab} that 
\begin{align}
\label{generalizationbound_unsimplified_discrete}
    &\left|\mathbb{E}_{\Theta_\infty,X_n}~\left[ \hat{R}(\Theta_\infty,X_n) \right] -  R(\Theta_\infty)  \right| \nonumber  \\
    &\leq L\widetilde{C}\rho(X_n,\widehat{X}_{n}) +L C \eta^{1/\alpha}+L\widehat{C}\eta^{1/\alpha}. 
    \end{align}
Analysis of the first term on the right hand side has been done in the proof of Corollary~\ref{coro_generalizationbound}, yielding 
\begin{align*}
   L \widetilde{C}\rho(X_n,\widehat{X}_{n})   &\leq \frac{1}{n}  \Big(d_1 D+d_2 D^{5/4}+d_3 D^{3/2}+d_4 D^{7/4}+d_5D^2+d_6 D^{5/2}\Big),
\end{align*}
where the constants $d_i,1\leq i\leq 6$ independent of $D$ are provided in Corollary~\ref{coro_generalizationbound}. Therefore, what remains is to study the factors $C$ and $\widehat{C}$. In fact, due to their similarities, it is sufficient to just study $C$. 

Recall at \eqref{estimatesquarerootf(0,xi)} and \eqref{estimate_maxxi}, we have
\begin{align}
 \max_{1\leq i\leq n}\sqrt{\abs{f\brac{0,{x}_i}}}\vee  \max_{1\leq i\leq n}\sqrt{\abs{f\brac{0,\hat{x}_i}}}&\leq \sqrt{\abs{f(0,0)}}+\sqrt{\norm{\nabla f(0,0)}}\sqrt{D}+\sqrt{\frac{K_2}{2}}D
\end{align}
and
\begin{align*}
    \max_{1\leq i\leq n} \norm{x_i}\vee\max_{1\leq i\leq n} \norm{\hat{x}_i}\leq D. 
\end{align*}
Combining the above estimates with the inequality $\sqrt{a+b}\leq \sqrt{a}+\sqrt{b}$ for any $a,b\geq 0$, we obtain 
\begin{align*}
    C&\leq \frac{C_*C_9}{\lambda_*}\cdot \Bigg(1+\min \left\{\sqrt{ \frac{r^2-r_0^2}{8}}, \sqrt{\frac{r^2-r_0^2}{8r^2}} \right\}^{-1}
    \\
    &\qquad\qquad\qquad\cdot\Bigg( 2\frac{C_8}{C_6}+1+{\sqrt{1+\beta\lambda_5}+\sqrt{\beta}\sqrt{{\max_{1\leq j\leq n}\abs{f(0,x_j)} }} }\Bigg)\\
    &\qquad+\sqrt{\max_{1\leq j\leq n}\abs{f(0,x_j)}}+\brac{\sqrt{K_2\max_{1\leq j\leq n}\norm{x_j}+\norm{\nabla f(0,0)} }+\sqrt{\frac{K_2\max_{1\leq j\leq n}\norm{x_j}+1}{2}}}\\&
    \qquad\cdot\min \left\{\sqrt{ \frac{r^2-r_0^2}{8}}, \sqrt{\frac{r^2-r_0^2}{8r^2}} \right\}^{-1}\bigg( 2\frac{C_8}{C_6}+1+{\sqrt{1+\beta\lambda_5}+\sqrt{\beta}\sqrt{{\max_{1\leq j\leq n}\abs{f(0,x_j)} }}  }\bigg)\Bigg),
    \end{align*}
and furthermore, we can compute that
\begin{align*}
    C&\leq \frac{C_*C_9}{\lambda_*}\cdot \Bigg(1+\min \left\{\sqrt{ \frac{r^2-r_0^2}{8}}, \sqrt{\frac{r^2-r_0^2}{8r^2}} \right\}^{-1}\\
    &\hspace{3em}\cdot\Bigg( 2\frac{C_8}{C_6}+1+{\sqrt{1+\beta\lambda_5}+\sqrt{\beta}\brac{\sqrt{\abs{f(0,0)}}+\sqrt{\norm{\nabla f(0,0)}}\sqrt{D}+\sqrt{\frac{K_2}{2}}D} }\Bigg)\\
    &+\sqrt{\abs{f(0,0)}}+\sqrt{\norm{\nabla f(0,0)}}\sqrt{D}+\sqrt{\frac{K_2}{2}}D+\brac{\sqrt{K_2D}+\sqrt{\norm{\nabla f(0,0)} }+\sqrt{\frac{K_2D}{2}}+1}\\&
    \cdot\min \left\{\sqrt{ \frac{r^2-r_0^2}{8}}, \sqrt{\frac{r^2-r_0^2}{8r^2}} \right\}^{-1}\\
    &\hspace{3em}\cdot\Bigg( 2\frac{C_8}{C_6}+1+{\sqrt{1+\beta\lambda_5}+\sqrt{\beta}\brac{\sqrt{\abs{f(0,0)}}+\sqrt{\norm{\nabla f(0,0)}}\sqrt{D}+\sqrt{\frac{K_2}{2}}D} }\Bigg)\Bigg).
    \end{align*}
By rearranging terms, we get
\begin{align*}
    C&\leq\frac{C_*C_9}{\lambda_*}\Bigg(2+ \min \left\{\sqrt{ \frac{r^2-r_0^2}{8}}, \sqrt{\frac{r^2-r_0^2}{8r^2}} \right\}^{-1}\\
    &\hspace{5em}\cdot\brac{2\frac{C_8}{C_6}+1+\sqrt{1+\beta\lambda_5}+\sqrt{\beta}+\sqrt{\beta} \sqrt{\abs{f(0,0)}} }+\sqrt{\abs{f(0,0)}}\\
    &\quad+\brac{\sqrt{\norm{\nabla f(0,0)}}+1 }\min \left\{\sqrt{ \frac{r^2-r_0^2}{8}}, \sqrt{\frac{r^2-r_0^2}{8r^2}} \right\}^{-1}\\
    &\hspace{5em}\cdot\brac{2\frac{C_8}{C_6}+1+\sqrt{1+\beta\lambda_5}+\sqrt{\beta}+\sqrt{\beta} \sqrt{\abs{f(0,0)}} }\Bigg)\\
    &\quad\quad+\sqrt{D} \cdot\frac{C_*C_9}{\lambda_*}\min \left\{\sqrt{ \frac{r^2-r_0^2}{8}}, \sqrt{\frac{r^2-r_0^2}{8r^2}} \right\}^{-1}\\
    &\quad\quad\quad\cdot\Bigg(\sqrt{\beta} \sqrt{\norm{\nabla f(0,0)}}+\sqrt{\norm{\nabla f(0,0)}}+\brac{\sqrt{K_2}+\sqrt{\frac{K_2}{2}}}\\
    &\quad\quad\cdot\brac{2\frac{C_8}{C_6}+1+\sqrt{1+\beta\lambda_5}+\sqrt{\beta}\sqrt{\abs{f(0,0)}} }+\sqrt{\beta} \sqrt{\norm{\nabla f(0,0)}}\brac{\sqrt{\norm{\nabla f(0,0)}}+1 }\Bigg)\\
    &\quad\quad+D\cdot \frac{C_*C_9}{\lambda_*}\min \left\{\sqrt{ \frac{r^2-r_0^2}{8}}, \sqrt{\frac{r^2-r_0^2}{8r^2}} \right\}^{-1}\Bigg(\sqrt{\beta}\sqrt{\frac{K_2}{2}}+\sqrt{\frac{K_2}{2}}+\brac{\sqrt{K_2}+\sqrt{\frac{K_2}{2}}}\\
    &\quad\quad\quad\quad\quad\cdot\sqrt{\beta}\sqrt{\norm{\nabla f(0,0)}}+\sqrt{\beta}\sqrt{\frac{K_2}{2}}\brac{\sqrt{\norm{\nabla f(0,0)}}+1 } \Bigg). 
\end{align*}
Thus, we arrive at 
  \begin{align*}
      &\left|\mathbb{E}_{\Theta_\infty,X_n}~\left[ \hat{R}(\Theta_\infty,X_n) \right] -  R(\Theta_\infty)  \right| \nonumber  \\
     &\leq \frac{1}{n}  \Big(d_1 D+d_2 D^{5/4}+d_3 D^{3/2}+d_4 D^{7/4}+d_5D^2+d_6 D^{5/2}\Big)+2L\eta^{1/\alpha}\brac{d_7+d_8\sqrt{D}+d_9D}, 
\end{align*}
where the constants $d_i,1\leq i\leq 6$ independent of $D$ are provided in Corollary~\ref{coro_generalizationbound}, and 
\begin{align}
\label{def_d7throughd9}
&d_7:=\frac{C_*C_9}{\lambda_*}\Bigg(2+ \min \left\{\sqrt{ \frac{r^2-r_0^2}{8}}, \sqrt{\frac{r^2-r_0^2}{8r^2}} \right\}^{-1}
\nonumber
\\
&\qquad\qquad\qquad\cdot\brac{2\frac{C_8}{C_6}+1+\sqrt{1+\beta\lambda_5}+\sqrt{\beta}+\sqrt{\beta} \sqrt{\abs{f(0,0)}} }\nonumber
\\
&\qquad\qquad\qquad\qquad+\sqrt{\abs{f(0,0)}}
+\brac{\sqrt{\norm{\nabla f(0,0)}}+1 }\min \left\{\sqrt{ \frac{r^2-r_0^2}{8}}, \sqrt{\frac{r^2-r_0^2}{8r^2}} \right\}^{-1}\nonumber
\\
&\qquad\qquad\qquad\qquad\qquad\qquad\cdot\brac{2\frac{C_8}{C_6}+1+\sqrt{1+\beta\lambda_5}+\sqrt{\beta}+\sqrt{\beta} \sqrt{\abs{f(0,0)}} }\Bigg);\nonumber\\
&d_8:=    \frac{C_*C_9}{\lambda_*}\min \left\{\sqrt{ \frac{r^2-r_0^2}{8}}, \sqrt{\frac{r^2-r_0^2}{8r^2}} \right\}^{-1}\nonumber\\
&\qquad\qquad\qquad\cdot\Bigg(\sqrt{\beta} \sqrt{\norm{\nabla f(0,0)}}+\sqrt{\norm{\nabla f(0,0)}}+\brac{\sqrt{K_2}+\sqrt{\frac{K_2}{2}}}\nonumber\\
    &\qquad\cdot\brac{2\frac{C_8}{C_6}+1+\sqrt{1+\beta\lambda_5}+\sqrt{\beta}\sqrt{\abs{f(0,0)}} }+\sqrt{\beta} \sqrt{\norm{\nabla f(0,0)}}\brac{\sqrt{\norm{\nabla f(0,0)}}+1 }\Bigg);\nonumber\\
&d_9:=\frac{C_*C_9}{\lambda_*}\min \left\{\sqrt{ \frac{r^2-r_0^2}{8}}, \sqrt{\frac{r^2-r_0^2}{8r^2}} \right\}^{-1}\Bigg(\sqrt{\beta}\sqrt{\frac{K_2}{2}}+\sqrt{\frac{K_2}{2}}+\brac{\sqrt{K_2}+\sqrt{\frac{K_2}{2}}}\nonumber\\
    &\qquad\qquad\qquad\qquad\qquad\qquad\cdot\sqrt{\beta}\sqrt{\norm{\nabla f(0,0)}}+\sqrt{\beta}\sqrt{\frac{K_2}{2}}\brac{\sqrt{\norm{\nabla f(0,0)}}+1 } \Bigg). 
\end{align}
This completes the proof.
\end{proof}

%%%%%%%%%%%%%%%%%%%%%%%%%%%%%%%%%%%%%%%%%%%%%%%%%%%%%%%%%%%%%%%%%%%%%%
%%%%%%%%%%%%%%%%%%%%%%%%%%%%%%%%%%%%%%%%%%%%%%%%%%%%%%%%%%%%%%%%%%%%%%
\subsection{Technical Lemmas}
\begin{lemma}
\label{lemma_boundgradientlyapunov}
Assume Conditions~\ref{cond_gammaandbeta},~\ref{cond_pseudolipschitz}, and~\ref{cond_allthelambdas}, and also that $\sup_{x,y\in \mathcal{X}}\norm{x-y}\leq D$ for some $D<\infty$. Then we have the estimates:
\begin{align*}
    &\norm{\nabla \mathbb{W}(x,v,X_n)}_{\operatorname{op},\infty}
    \leq C_{3};\\
    &\norm{\nabla^2 \mathbb{W}(x,v,X_n)}_{\operatorname{op}}
    \leq \frac{C_4}{1+\norm{x}+\norm{v}},
\end{align*} 
where the constants $C_3=C_3(D)$ and $C_4=C_4(D)$ have the forms:
\begin{align*}
    C_3(D)&:=\min  \left\{ \frac{1}{\sqrt{2}}, \frac{1}{2}\frac{r^2-r_0^2}{4}, \frac{1}{2}\frac{r^2-r_0^2}{4r^2}\right\}^{-1}
    \\
    &\qquad\qquad\qquad\qquad\cdot\brac{K_1+2K_2D+2\beta\lambda_4+r^2+r_0+\norm{\nabla f(0,0)}_{\operatorname{op}}}\\
    &\qquad\qquad\qquad+\min  \left\{ \frac{1}{\sqrt{2}}, \frac{1}{2}\frac{r^2-r_0^2}{4}, \frac{1}{2}\frac{r^2-r_0^2}{4r^2}\right\}^{-1}\brac{1+r_0},
\end{align*}
and 
\begin{align*}
    C_4(D)&:=\Bigg(\min  \left\{ \frac{1}{\sqrt{2}}, \frac{1}{2}\frac{r^2-r_0^2}{4}, \frac{1}{2}\frac{r^2-r_0^2}{4r^2}\right\}^{-3} \\
    &\qquad\qquad\qquad\cdot\brac{\beta^2\brac{K_1+4K_2D+2\lambda_4+\norm{\nabla f(0,0)}_{\operatorname{op}} }^2+(r^2+r_0)^2}\\
   &\qquad\qquad+\min  \left\{ \frac{1}{\sqrt{2}}, \frac{1}{2}\frac{r^2-r_0^2}{4}, \frac{1}{2}\frac{r^2-r_0^2}{4r^2}\right\}^{-1}\brac{  \beta K_1+2\beta\lambda_4+r^2}\Bigg)\\
   &+2\Bigg(\min  \left\{ \frac{1}{\sqrt{2}}, \frac{1}{2}\frac{r^2-r_0^2}{4}, \frac{1}{2}\frac{r^2-r_0^2}{4r^2}\right\}^{-1}r_0+ \min  \left\{ \frac{1}{\sqrt{2}}, \frac{1}{2}\frac{r^2-r_0^2}{4}, \frac{1}{2}\frac{r^2-r_0^2}{4r^2}\right\}^{-3}\\
     &\qquad\cdot\brac{\beta\brac{K_1+4K_2D+2\lambda_4+\norm{\nabla f(0,0)}_{\operatorname{op}} }+r^2+r_0}\brac{\beta K_1+2\beta\lambda_4+r} \Bigg)\\
     &+\bigg(\min  \left\{ \frac{1}{\sqrt{2}}, \frac{1}{2}\frac{r^2-r_0^2}{4}, \frac{1}{2}\frac{r^2-r_0^2}{4r^2}\right\}^{-1}\\
     &\hspace{10em}+\min  \left\{ \frac{1}{\sqrt{2}}, \frac{1}{2}\frac{r^2-r_0^2}{4}, \frac{1}{2}\frac{r^2-r_0^2}{4r^2}\right\}^{-2}(1+r_0) \bigg). 
\end{align*}
\end{lemma}

\begin{proof}
First of all, we have
\begin{equation*}
\norm{\nabla \mathbb{W}(x,v,X_n)}_{\operatorname{op},\infty}\leq \norm{\nabla_x \mathbb{W}(x,v,X_n)}_{\operatorname{op},\infty} +\norm{\nabla_v \mathbb{W}(x,v,X_n)}_{\operatorname{op},\infty}.
\end{equation*}
We first consider the term $\nabla_x \mathbb{W}(x,v,X_n)$. Recall the definition of the function $N(x,v,X_n)$ in Lemma~\ref{lemma_lyapunov}.  Via \eqref{lyapunovfunction_upperandlowerestimate} and the fact that $V_0(x,X_n)\geq 0$, we have 
%and $\sqrt{a}+\sqrt{b}\leq \sqrt{2(a+b)}$
\begin{align}
\label{lowerboundN}
    \left(N(x,v,X_n)\right)^{1/2}&\geq \brac{1+V_0(x,X_n) +\frac{r^2-r_0^2}{4}\norm{x}^2 +\frac{r^2-r_0^2}{4r^2}\norm{v}^2}^{1/2}\nonumber\\
    &\geq \brac{1+\frac{r^2-r_0^2}{4}\norm{x}^2 +\frac{r^2-r_0^2}{4r^2}\norm{v}^2}^{1/2}\nonumber\\
    &\geq \brac{1+\frac{1}{2}\min  \left\{ \frac{r^2-r_0^2}{4}, \frac{r^2-r_0^2}{4r^2}\right\}\brac{\norm{x}+\norm{v}}^2}^{1/2}\nonumber\\
    &\geq \frac{1}{\sqrt{2}}\brac{1+\frac{1}{\sqrt{2}}\min  \left\{\sqrt{ \frac{r^2-r_0^2}{4}}, \sqrt{\frac{r^2-r_0^2}{4r^2}} \right\}\brac{\norm{x}+\norm{v}}  }\nonumber\\
    &\geq \min  \left\{ \frac{1}{\sqrt{2}}, \frac{1}{2}\frac{r^2-r_0^2}{4}, \frac{1}{2}\frac{r^2-r_0^2}{4r^2}\right\}\brac{1+\norm{x}+\norm{v}}. 
\end{align}
 Condition~\ref{cond_pseudolipschitz} and $0\in \mathcal{X}$ and the fact that $\sup_{x,y\in \mathcal{X}}\norm{x-y}\leq D$ for some $D<\infty$ lead to
\begin{align}
\label{bound_gradientF}
    \norm{\nabla \widehat{F}(x,X_n)}_{\operatorname{op}}\leq K_1\norm{x}+K_2D\brac{2\norm{x}+1}+ \norm{\nabla f(0,0)}_{\operatorname{op}}. 
\end{align}

Thus, for any $x,v\in\mathbb{R}^{d}$,
\begin{align*}
    &\norm{\nabla_x \mathbb{W}(x,v,X_n)}_{\operatorname{op}}\\
    &\leq 1/2\left(N(x,v,X_n)\right)^{-1/2}\brac{\norm{\nabla V_0(x,X_n)}_{\operatorname{op}}+r^2\norm{x}+r_0\norm{v} }\\
    &\leq  1/2\left(N(x,v,X_n)\right)^{-1/2}\brac{ \beta\norm{\nabla \widehat{F}(x,X_n)}_{\operatorname{op}}+\beta 2\lambda_4\norm{x} +r^2\norm{x}+r_0\norm{v} }\\
    &\leq \min  \left\{ \frac{1}{\sqrt{2}}, \frac{1}{2}\frac{r^2-r_0^2}{4}, \frac{1}{2}\frac{r^2-r_0^2}{4r^2}\right\}\brac{1+\norm{x}+\norm{v}}\\
    &\qquad\qquad\cdot\brac{\brac{K_1+K_2D+2\beta\lambda_4+r^2 } \norm{x}+r_0\norm{v}+K_2D+\norm{\nabla f(0,0)}_{\operatorname{op}} }\\
    &\leq \min  \left\{ \frac{1}{\sqrt{2}}, \frac{1}{2}\frac{r^2-r_0^2}{4}, \frac{1}{2}\frac{r^2-r_0^2}{4r^2}\right\}^{-1}\brac{1+\norm{x}+\norm{v}}^{-1}\\
    &\hspace{3em}\cdot\max  \left\{K_1+K_2D+2\beta\lambda_4+r^2,r_0,K_2D+\norm{\nabla f(0,0)}_{\operatorname{op}}  \right\}\brac{1+\norm{x}+\norm{v}}\\
    &= \min  \left\{ \frac{1}{\sqrt{2}}, \frac{1}{2}\frac{r^2-r_0^2}{4}, \frac{1}{2}\frac{r^2-r_0^2}{4r^2}\right\}^{-1}\brac{K_1+2K_2D+2\beta\lambda_4+r^2+r_0+\norm{\nabla f(0,0)}_{\operatorname{op}}},
\end{align*}
which implies
\begin{align*}
    &\norm{\nabla_x \mathbb{W}(x,v,X_n)}_{\operatorname{op},\infty}\\
    &\leq \min  \left\{ \frac{1}{\sqrt{2}}, \frac{1}{2}\frac{r^2-r_0^2}{4}, \frac{1}{2}\frac{r^2-r_0^2}{4r^2}\right\}^{-1}\brac{K_1+2K_2D+2\beta\lambda_4+r^2+r_0+\norm{\nabla f(0,0)}_{\operatorname{op}}}. 
\end{align*}

Next, we deal with the term $\nabla_v \mathbb{W}(x,v,X_n)$. We can compute that for any $x,v\in\mathbb{R}^{d}$,
\begin{align*}
      &\norm{\nabla_v \mathbb{W}(x,v,X_n)}_{\operatorname{op}}\\
    &\leq 1/2\left(N(x,v,X_n)\right)^{-1/2}\brac{\norm{v}+r_0\norm{x} }\\
    &\leq \min  \left\{ \frac{1}{\sqrt{2}}, \frac{1}{2}\frac{r^2-r_0^2}{4}, \frac{1}{2}\frac{r^2-r_0^2}{4r^2}\right\}^{-1}\brac{1+\norm{x}+\norm{v}}^{-1}\cdot\max \{1,r_0 \}
    (1+\norm{x}+\norm{v})\\
    &=\min  \left\{ \frac{1}{\sqrt{2}}, \frac{1}{2}\frac{r^2-r_0^2}{4}, \frac{1}{2}\frac{r^2-r_0^2}{4r^2}\right\}^{-1}\brac{1+r_0},
\end{align*}
and therefore
\begin{align*}
    &\norm{\nabla_v \mathbb{W}(x,v,X_n)}_{\operatorname{op},\infty}\leq \min  \left\{ \frac{1}{\sqrt{2}}, \frac{1}{2}\frac{r^2-r_0^2}{4}, \frac{1}{2}\frac{r^2-r_0^2}{4r^2}\right\}^{-1}\brac{1+r_0}. 
\end{align*}

Now, we consider the second gradient of $\mathbb{W}$. First, we notice that
\begin{align*}
&\norm{\nabla^2 \mathbb{W}(x,v,X_n)}_{\operatorname{op}}
\\
&\leq \norm{\nabla_x\nabla_x \mathbb{W}(x,v,X_n)}_{\operatorname{op}}+2\norm{\nabla_v\nabla_x \mathbb{W}(x,v,X_n)}_{\operatorname{op}}+\norm{\nabla_v\nabla_v \mathbb{W}(x,v,X_n)}_{\operatorname{op}}.
\end{align*}
Let us start with 
\begin{align*}
   \norm{ \nabla_x\nabla_x \mathbb{W}(x,v,X_n)}_{\operatorname{op}}&\leq \frac{1}{4}\left(N(x,v,X_n)\right)^{-3/2} \brac{\norm{\nabla V_0(x,X_n)}_{\operatorname{op}}+r^2\norm{x}+r_0\norm{v}}^2\\
    &\qquad\qquad+\frac{1}{2}\left(N(x,v,X_n)\right)^{-1/2} \brac{\norm{\nabla^2V_0(x,X_n)}_{\operatorname{op}}+r^2}. 
\end{align*}

This, together with \eqref{lowerboundN}, and 
\begin{align}
\label{boundnormgradientV_0}
    &\norm{\nabla V_0(x,X_n)}_{\operatorname{op}}\leq \beta\brac{K_1+4K_2D+2\lambda_4+\norm{\nabla f(0,0)}_{\operatorname{op}} }(1+\norm{x}),\nonumber\\
    &\norm{\nabla^2 V_0(x,X_n)}_{\operatorname{op}}\leq \beta K_1+2\beta\lambda_4,
\end{align}
implies that for every $x,v\in\mathbb{R}^{d}$,
\begin{align*}
   & \norm{ \nabla_x\nabla_x \mathbb{W}(x,v,X_n)}_{\operatorname{op}}\\
   &\leq \frac{1}{4}\min  \left\{ \frac{1}{\sqrt{2}}, \frac{1}{2}\frac{r^2-r_0^2}{4}, \frac{1}{2}\frac{r^2-r_0^2}{4r^2}\right\}^{-3}\frac{1}{(1+\norm{x}+\norm{v})^3}\\
   &\qquad\cdot\brac{\beta^2\brac{K_1+4K_2D+2\lambda_4+\norm{\nabla f(0,0)}_{\operatorname{op}} }^2+(r^2+r_0)^2}(1+\norm{x})^2\\
   &\qquad\qquad+\frac{1}{2}\min  \left\{ \frac{1}{\sqrt{2}}, \frac{1}{2}\frac{r^2-r_0^2}{4}, \frac{1}{2}\frac{r^2-r_0^2}{4r^2}\right\}^{-1}\frac{1}{\norm{x}+\norm{v}+1}\brac{  \beta K_1+2\beta\lambda_4+r^2}\\
   &\leq \Bigg(\min  \left\{ \frac{1}{\sqrt{2}}, \frac{1}{2}\frac{r^2-r_0^2}{4}, \frac{1}{2}\frac{r^2-r_0^2}{4r^2}\right\}^{-3}
   \\
   &\qquad\qquad\qquad\cdot\brac{\beta^2\brac{K_1+4K_2D+2\lambda_4+\norm{\nabla f(0,0)}_{\operatorname{op}} }^2+(r^2+r_0)^2}\\
   &\qquad\qquad+\min  \left\{ \frac{1}{\sqrt{2}}, \frac{1}{2}\frac{r^2-r_0^2}{4}, \frac{1}{2}\frac{r^2-r_0^2}{4r^2}\right\}^{-1}\brac{  \beta K_1+2\beta\lambda_4+r^2}\Bigg)\frac{1}{\norm{x}+\norm{v}+1}. 
\end{align*}

Similarly, for every $x,v\in\mathbb{R}^{d}$, we have
%there exist positive constants $c_3,c_4$ such that 
\begin{align*}
     &\norm{\nabla_v\nabla_x \mathbb{W}(x,v,X_n)}_{\operatorname{op}}\\
     &\leq \frac{1}{2}\left(N(x,v,X_n)\right)^{-1/2}r_0+\frac{1}{4}\left(N(x,v,X_n)\right)^{-3/2}\brac{\norm{\nabla^2 V_0(x,X_n)}_{\operatorname{op}}+r^2}
     \\
     &\qquad\qquad\qquad\qquad\qquad\qquad\qquad\cdot\brac{\norm{\nabla V_0(x,X_n)}_{\operatorname{op}}+r^2\norm{x}+r_0\norm{v} } \\
     &\leq \frac{1}{2}\min  \left\{ \frac{1}{\sqrt{2}}, \frac{1}{2}\frac{r^2-r_0^2}{4}, \frac{1}{2}\frac{r^2-r_0^2}{4r^2}\right\}^{-1}r_0\frac{1}{1+\norm{x}+\norm{v}}\\
     &\qquad+\frac{1}{4}\min  \left\{ \frac{1}{\sqrt{2}}, \frac{1}{2}\frac{r^2-r_0^2}{4}, \frac{1}{2}\frac{r^2-r_0^2}{4r^2}\right\}^{-3} \frac{1}{\brac{1+\norm{x}+\norm{v}}^3} \\&\qquad\qquad\cdot\brac{\beta\brac{K_1+4K_2D+2\lambda_4+\norm{\nabla f(0,0)}_{\operatorname{op}} }+r^2+r_0}
     \\
     &\qquad\qquad\qquad\qquad\qquad\cdot\brac{1+\norm{x}+\norm{v}}\brac{\beta K_1+2\beta\lambda_4+r}\\
     &\leq \Bigg(\min  \left\{ \frac{1}{\sqrt{2}}, \frac{1}{2}\frac{r^2-r_0^2}{4}, \frac{1}{2}\frac{r^2-r_0^2}{4r^2}\right\}^{-1}r_0+ \min  \left\{ \frac{1}{\sqrt{2}}, \frac{1}{2}\frac{r^2-r_0^2}{4}, \frac{1}{2}\frac{r^2-r_0^2}{4r^2}\right\}^{-3}\\    &\qquad\cdot\brac{\beta\brac{K_1+4K_2D+2\lambda_4+\norm{\nabla f(0,0)}_{\operatorname{op}} }+r^2+r_0}\brac{\beta K_1+2\beta\lambda_4+r} \Bigg)
     \\
     &\qquad\qquad\qquad\qquad\qquad\cdot\frac{1}{1+\norm{x}+\norm{v}}. 
\end{align*}
Finally, 
\begin{align*}
       &\norm{\nabla_v\nabla_v \mathbb{W}(x,v,X_n)}_{\operatorname{op}}\\  
       &\leq \frac{1}{2}\left(N(x,v,X_n)\right)^{-1/2}+\frac{1}{4}\left(N(x,v,X_n)\right)^{-3/2}\norm{v+r_0x}^2 \\
       &\leq\brac{\min  \left\{ \frac{1}{\sqrt{2}}, \frac{1}{2}\frac{r^2-r_0^2}{4}, \frac{1}{2}\frac{r^2-r_0^2}{4r^2}\right\}^{-1}+\min  \left\{ \frac{1}{\sqrt{2}}, \frac{1}{2}\frac{r^2-r_0^2}{4}, \frac{1}{2}\frac{r^2-r_0^2}{4r^2}\right\}^{-2}(1+r_0) } 
       \\
       &\qquad\qquad\qquad\cdot\frac{1}{1+\norm{x}+\norm{v}}, 
\end{align*}
for every $x,v\in\mathbb{R}^{d}$.
This completes the proof.
\end{proof}

\begin{lemma}
\label{lemma_uniformmomentbound_discretedynamics}
 Assume Conditions~\ref{cond_gammaandbeta},~\ref{cond_pseudolipschitz}, and~\ref{cond_allthelambdas} and also that $\sup_{x,y\in \mathcal{X}}\norm{x-y}\leq D$ for some $D<\infty$. Then for any $\eta$ satisfying \eqref{bar:eta:defn}, it holds that for every $m=0,1,2,\ldots$
 \begin{align*}
     &\E{\norm{\Theta_{m+1}^{x,v}}}+\E{\norm{V_{m+1}^{x,v} }}
     \leq C_5(x,v,X_n);\\
&\E{\norm{\hat{\Theta}_{m+1}^{x,v}}}+\E{\norm{\hat{V}_{m+1}^{x,v} }}
     \leq C_5(x,v,\widehat{X}_{n}), 
 \end{align*}
 where
 \begin{align*}
     &C_5(x,v,X_n):= \min \left\{\sqrt{ \frac{r^2-r_0^2}{8}}, \sqrt{\frac{r^2-r_0^2}{8r^2}} \right\}^{-1}
     \\&\qquad\qquad\qquad\cdot\Bigg( 2\frac{C_8}{C_6}+1+{\sqrt{1+\beta\lambda_5}+\sqrt{\beta}\sqrt{{\max_{1\leq i\leq n}\abs{f(x,x_i)} }}+\sqrt{\beta\lambda_4+r^2}\norm{x}+\norm{v}  }\Bigg);\\
     &C_5(x,v,\widehat{X}_{n}):= \min \left\{\sqrt{ \frac{r^2-r_0^2}{8}}, \sqrt{\frac{r^2-r_0^2}{8r^2}} \right\}^{-1}\\
     &\qquad\qquad\qquad\cdot\Bigg( 2\frac{C_8}{C_6}+1+{\sqrt{1+\beta\lambda_5}+\sqrt{\beta}\sqrt{{\max_{1\leq i\leq n}\abs{f(x,\hat{x}_i)} }}+\sqrt{\beta\lambda_4+r^2}\norm{x}+\norm{v}  }\Bigg). 
 \end{align*}
 The constants $r,r_0$ are provided in Lemma~\ref{lemma_lyapunov}, while the constants $C_6,C_8$ are respectively defined at \eqref{def_C6} and \eqref{def_C8}. 
\end{lemma}

\begin{proof}
  We will only provide the proof for $\left(\Theta_{m+1}^{x,v},V_{m+1}^{x,v}\right)$ and the proof for $\left(\hat{\Theta}_{m+1}^{x,v},\hat{V}_{m+1}^{x,v}\right)$ is similar. The same argument that leads to \eqref{inequality_contractionlyapunov} will also give us: for any $m\in \N$, 
    \begin{align*}
        \E{\mathbb{W}(\Theta_{m+1}^{x,v},V_{m+1}^{x,v},X_n )}\leq \brac{1-\frac{C_6}{2}\eta}\E{\mathbb{W}(\Theta_{m}^{x,v},V_{m}^{x,v},X_n ) }+ C_8\eta. 
    \end{align*}
Applying this inequality inductively to get
\begin{align*}
    \E{\mathbb{W}(\Theta_{m+1}^{x,v},V_{m+1}^{x,v},X_n )}&\leq \brac{1-\frac{C_6}{2}\eta}^{m+1}\mathbb{W}(x,v,X_n ) +C_8\eta\sum_{j=0}^m \brac{1-\frac{C_6}{2}\eta}^{j}\\
    &\leq \mathbb{W}(x,v,X_n ) +2\frac{C_8}{C_6}. 
\end{align*}
Finally, to complete the proof, we recall from the proof of Lemma~\ref{lemma_uniformmomentbound_continuousdynamics} that 
\begin{align*}
    & \min \left\{\sqrt{ \frac{r^2-r_0^2}{8}}, \sqrt{\frac{r^2-r_0^2}{8r^2}} \right\} \brac{\norm{x}+\norm{v}}
    \\
    &\leq \mathbb{W}(x,v,X_n)\\
     &\leq 1+\brac{\sqrt{1+\beta\lambda_5}+\sqrt{\beta}\sqrt{{\max_{1\leq i\leq n}\abs{f(x,x_i)} }}+\sqrt{\beta\lambda_4+r^2}\norm{x}+\norm{v} }. 
\end{align*}
The proof is complete.
\end{proof}

\begin{lemma}
\label{lemma_onestepestimate_discretedynamics}
 Assume Conditions~\ref{cond_gammaandbeta},~\ref{cond_pseudolipschitz}, and~\ref{cond_allthelambdas}. Then for any stepsize $\eta$ satisfying \eqref{bar:eta:defn} and any Lipschitz function $h:\mathbb{R}^{2d}\rightarrow\mathbb{R}$, it holds that 
    \begin{align*}
    \abs{P_\eta h(w,y)-Q_1h(w,y)}\leq C_9\norm{\nabla h}_{\infty,\operatorname{op}}\brac{1+\norm{w}+\norm{y}+\max_{1\leq i\leq n}\sqrt{\abs{f\brac{w,x_i}}}}\eta^{1+1/\alpha}. 
\end{align*}
where 
\begin{align*}
    C_9&:= \max \bigg\{C_2(\gamma^2+K_1),C_2\brac{2K_2D+\norm{\nabla f(0,0)}},
    \\
    &\qquad\qquad\qquad\gamma C_2+(K_1+2K_2D)C_2+4K_2D+2\norm{\nabla f(0,0)},\\
    &\qquad\qquad\gamma C_2+(K_1+2K_2D)C_2 +K_1+2K_2D,\gamma C_2+(K_1+2K_2D)C_2+\gamma,
    \\
    &\qquad\qquad\gamma C_2+(K_1+2K_2D)C_2+2,(\gamma+1)\zeta\brac{1+\frac{1}{1+1/\alpha}}\E{\norm{L_1}} 
    \bigg\}.
\end{align*}
\end{lemma}

\begin{proof}
First, we have
\begin{align*}
    &\abs{P_\eta h(w,y)-Q_1h(w,y)}\\
    &=\abs{\E{h(v_\eta^{w,y},\theta_\eta^{w,y})-h(V_1^{w,y},\Theta_1^{w,y})}}\leq \norm{\nabla h}_{\infty,\operatorname{op}}\brac{\E{\norm{v_\eta^{w,y}-V_1^{w,y}}}+\E{\norm{ \theta_\eta^{w,y}-\Theta_1^{w,y}}} }. 
\end{align*}

We can compute that
\begin{align*}
    &\E{\norm{v_\eta^{w,y}-V_1^{w,y} }}\\
    &= \E{\norm{ y+\int_0^\eta\left(-\gamma v_s^{w,y}-\nabla\widehat{F}\brac{\theta_s^{w,y},X_n}\right)ds+\zeta L_\eta-\brac{-\eta\gamma y-\eta\nabla\widehat{F}(w,X_n) +\zeta L_\eta} }}\\
    &=\E{\norm{\int_0^\eta v_s^{w,y}-yds-\beta\int_0^\eta \nabla\widehat{F}\brac{\theta_s^{w,y},X_n}-\nabla\widehat{F}\brac{w,X_n} ds }} \\
    &\leq \E{ \norm{\int_0^\eta\brac{\int_0^s-\gamma v_r^{w,y}-\beta\nabla\widehat{F}\brac{\theta_r^{w,y},X_n}dr+
    \zeta L_s}  ds }}+ \E{\int_0^\eta K_1\norm{\theta_s^{w,y}-w}ds}\\
    &\leq \gamma\int_0^\eta\int_0^s \left(\gamma\E{\norm{v_r^{w,y}}}+\E{\norm{\nabla\widehat{F}(\theta_r^{w,y},X_n) }}\right)drds\\
    &\qquad\qquad+\zeta \int_0^\eta s^{1/\alpha} \E{\norm{L_1}}ds +\beta K_1\int_0^\eta\int_0^s \E{\norm{v_r^{w,y}}}drds \\
    &\leq \brac{\eta^2\vee\eta^{1+1/\alpha}}\cdot\max\left\{\gamma^2+K_1,K_1+2K_2D,2K_2D+\norm{\nabla f(0,0)},\frac{\gamma\zeta}{1+1/\alpha} \right\}
    \\
    &\qquad\qquad\cdot \sup_{r\geq 0}\brac{\E{\norm{v_r^{w,y}}} +\E{\norm{\theta_r^{w,y}}} +1} \\ 
    &\leq \brac{\eta^2\vee\eta^{1+1/\alpha}}\cdot C_2\max\left\{\gamma^2+K_1,K_1+2K_2D,2K_2D+\norm{\nabla f(0,0)}, \frac{\gamma\zeta\E{\norm{L_1}}}{1+1/\alpha} \right\} \\
    &\qquad\qquad\cdot\brac{1+\norm{w}+\norm{y}+\max_{1\leq i\leq n}\sqrt{\abs{f\brac{w,x_i}}}}.
\end{align*}
To get the fourth line above, we use $\norm{\nabla\widehat{F}(\theta,x)-\nabla\widehat{F}(\hat{\theta},x)}\leq K_1\norm{\theta-\hat{\theta}}$. For the fifth line, we use the self-similarity property of $\alpha$-stable processes.  The second to last line is due to \eqref{bound_gradientF}, and the last line is a consequence of the uniform moment bound in Lemma~\ref{lemma_uniformmomentbound_continuousdynamics}.

Similarly, 
\begin{align*}
    &\E{\norm{\theta_\eta^{w,y}-\Theta_1^{w,y}}}\\
    &\leq \E{\norm{\int_0^\eta v_s^{w,y}ds-\int_0^\eta V_1^{w,y}ds }}\\
    &\leq \mathbb{E}\bigg[\bigg\|\int_0^\eta 
    \brac{y+\int_0^s \left(-\gamma v_r^{w,y}-\nabla\widehat{F}(\theta_r^{w,y},X_n)\right)dr+\zeta L_s}ds\\
    &\hspace{14em}-\int_0^\eta\brac{ y+\int_0^\eta\left(-\gamma y-\nabla\widehat{F}(w,X_n)\right)dr+\zeta L_\eta}ds \bigg\|\bigg]\\
    &\leq \E{\norm{\int_0^\eta \int_0^s \left(-\gamma v_r^{w,y}-\nabla\widehat{F}(\theta_r^{w,y},X_n)\right)drds-\int_0^\eta \int_0^\eta\left(-\gamma y-\nabla\widehat{F}(w,X_n)\right)drds }}\\
    &\hspace{22em}+\zeta\E{\int_{0}^\eta s^{1/\alpha}\norm{L_1}+\eta^{1/\alpha}\norm{L_1}  ds } \\
    &\leq \eta^2\cdot \brac{\gamma \sup_{r\geq 0}\E{\norm{v_r^{w,y}}}+\sup_{r\geq 0} \E{\norm{\nabla \hat{F}(\theta_r^{w,y},X_n)}}+\gamma\norm{y}+\E{\norm{\nabla\hat{F}(w,X_n)}} }\\
    &\hspace{21em}+\eta^{1+1/\alpha}\cdot \zeta\brac{1+\frac{1}{1+1/\alpha}}\E{\norm{L_1}}\\
    &\leq \brac{\eta^2\vee\eta^{1+1/\alpha}}\cdot\max \bigg\{\gamma C_2+(K_1+2K_2D)C_2+4K_2D+2\norm{\nabla f(0,0)},
    \\
    &\qquad\qquad\qquad\qquad\gamma C_2+(K_1+2K_2D)C_2 +K_1+2K_2D,\zeta\brac{1+\frac{1}{1+1/\alpha}}\E{\norm{L_1}},\\
    &\qquad\qquad\qquad\qquad\qquad\qquad\gamma C_2+(K_1+2K_2D)C_2+\gamma,\gamma C_2+(K_1+2K_2D)C_2+2  \bigg\}
    \\
    &\qquad\qquad\qquad\qquad\qquad\qquad\qquad\qquad\qquad\quad\cdot\brac{1+\norm{w}+\norm{y}+\max_{1\leq i\leq n}\sqrt{\abs{f\brac{w,x_i}}}}. 
\end{align*}
To get the fourth line above, we use the self-similarity property of $\alpha$-stable processes. The last line is a consequence of \eqref{bound_gradientF} and the uniform moment bound in Lemma~\ref{lemma_uniformmomentbound_continuousdynamics}.

Combining the previous calculations yields the desired estimate. Notice in particular that any stepsize $\eta$ satisfying \eqref{bar:eta:defn} is less than or equal to $1$, so that  $\eta^2\vee\eta^{1+1/\alpha}\leq \eta^{1+1/\alpha}$. This completes the proof. 
\end{proof}

Similar to Lemma~\ref{lemma_momentbound_squarerootf}, we can deduce from Lemma~\ref{lemma_uniformmomentbound_discretedynamics} the following result.
\begin{lemma}
\label{lemma_momentbound_squarerootf_discretedynamics}
  It holds for any $x\in\mathcal{X}$ and $m\in\mathbb{N}$,
\begin{align*}
  \E{\sqrt{f\brac{{\Theta}^{w,y}_{m},x }}}
   \leq C_7(w,y,x),
\end{align*}   
where
\begin{align*}
  &C_7(w,y,x):=\\
  &=\sqrt{\abs{f(0,x)}}+\brac{\sqrt{K_2\norm{x}+\norm{\nabla f(0,0)} }+\sqrt{\frac{K_2\norm{x}+1}{2}}}\\&
    \qquad\cdot\min \left\{\sqrt{ \frac{r^2-r_0^2}{8}}, \sqrt{\frac{r^2-r_0^2}{8r^2}} \right\}^{-1}
    \\
    &\qquad\qquad\cdot\Bigg( 2\frac{C_8}{C_6}+1+{\sqrt{1+\beta\lambda_5}+\sqrt{\beta}\sqrt{{\max_{1\leq i\leq n}\abs{f(w,x_i)} }}+\sqrt{\beta\lambda_4+r^2}\norm{w}+\norm{y}  }\Bigg). 
\end{align*}   
\end{lemma}

\begin{proof}
Similar to the proof Lemma~\ref{lemma_momentbound_squarerootf}, we can obtain 
    \begin{align*}
   \E{\sqrt{f\left({\Theta}^{w,y}_{m},x \right)}}&\leq \sqrt{\abs{f(0,x)}}+\sqrt{K_2\norm{x}+\norm{\nabla f(0,0)} }\E{\sqrt{\norm{{\Theta}^{w,y}_{m}}}}
   \\
   &\qquad\qquad\qquad+\sqrt{\frac{K_2\norm{x}+1}{2}}\E{\norm{{\Theta}^{w,y}_{m}}}. 
\end{align*}
By combining this with the uniform moment bound in Lemma~\ref{lemma_uniformmomentbound_discretedynamics} and the definition of $C_5(w,y,X_n)$ in Lemma \ref{lemma_uniformmomentbound_discretedynamics}, we complete the proof. 
\end{proof}

%%%%%%%%%%%%%%%%%%%%%%%%%%%%%%%%%%%%%%%%%%%%%%%%%%%%%%%%%
\section{Additional Experimental Details}
\label{appendix:experimental_details}
%%%%%%%%%%%%%%%%%%%%%%%%%%%%%%%%%%%%%%%%%%%%%%%%%%%%%%%%%%%%%%%%\
\subsection{Datasets}
In addition to the synthetic dataset described in the main paper, we used the well known MNIST \cite{lecunnMNIST1998} and CIFAR-10 \cite{krizhevskyImageNetClassification2017} datasets for our experiments with neural networks. The MNIST dataset contains $28 \times 28$ black and white images of handwritten digits. We used the default train-test split with 60,000 and 10,000 samples, respectively. CIFAR-10 dataset includes color images of 10 classes of objects or animals with each image having a dimensionality of $32 \times 32 \times 3$. Here, too, we utilized the default split with 50,000 training and 10,000 test instances.

\subsection{Models}
Our neural network experiments include results with fully connected networks (FCN) and convolutional neural networks (CNN). In both cases, we use ReLU as activation function, do not use advanced layer structures such as residual connections or layer/batch normalization, and do not use bias nodes. Due to their comparably low number of parameters, during training additive heavy-tailed noise was not applied to last layers, and first convolutional layers. In adding noise to CNNs, we reconfigured the four-dimensional convolutional layers to be two-dimensional with \textit{kernels} constituting the rows of the resulting matrix. The architecture of the CNN used in the experiments is a slightly simplified version of VGG11 \cite{simonyanVeryDeep2015}, with the structure
$$
128, M, 256, M, 512, 512, M, 1024, 1024, M, 1024, 1024, M,
$$
where $M$'s stand for $2 \times 2$ max pooling operations, and the numbers denote convolutional layer widths with $3 \times 3$ filters, each of which were followed by ReLU activation functions.

\subsection{Software and Hardware}
The experiments were implemented using Python programming language. While the computational frameworks \texttt{numpy}, \texttt{scipy}, and \texttt{scikit-learn} were used for synthetic linear regression experiments, the \texttt{PyTorch} deep learning framework was used for experiments with neural networks \cite{paszkePyTorchImperative2019}. The experiments were run on the server of an educational institution, and results published in the main paper required an estimated GPU time of 600 hours in total, with the linear regression and neural network experiments corresponding to 40 and 560 hours respectively. Our implementation can be seen in the the accompanying source code, which will be made publicly accessible upon the publication of the paper.
\end{document}